\documentclass[11pt]{article}

\usepackage[colorlinks = true,	
			bookmarks = false,
           linkcolor = red,
           urlcolor  = blue,
           citecolor = blue,
           anchorcolor = blue]{hyperref}

\usepackage[utf8]{inputenc} 
\usepackage[T1]{fontenc}    
\usepackage{url}            
\usepackage{booktabs}       
\usepackage{amsfonts}       
\usepackage{nicefrac}       
\usepackage{microtype}      

\usepackage[T1]{fontenc}
\usepackage[english]{babel}
\usepackage{amsmath,amssymb,amsthm}
\usepackage{amscd}
\usepackage{mathrsfs} 

\usepackage[shortlabels]{enumitem}
\usepackage{mathrsfs, eufrak,euscript}
\usepackage{bm}

\usepackage{amsmath, amsthm} 
\usepackage{amsfonts,amssymb}

\usepackage{nicefrac}       
\usepackage{microtype}      

\usepackage{thmtools, thm-restate}
\usepackage{mathrsfs} 
\usepackage[usenames,dvipsnames]{pstricks}

\usepackage{graphicx}

\usepackage{booktabs}
\usepackage{algorithm,algpseudocode}

\usepackage{enumitem}
\usepackage{xcolor}

\usepackage{caption}
\usepackage[capitalize]{cleveref}

\usepackage{etoolbox}

\usepackage{url}            

\newcommand{\eqals}[1]{\begin{align*}#1\end{align*}}
\newcommand{\eqal}[1]{\begin{align}#1\end{align}}
\newcommand{\eqalsplit}[1]{\eqal{\begin{split}#1\end{split}}}

\newcommand{\bpr}{\begin{proof}}
\newcommand{\epr}{\end{proof}}
\newcommand{\be}{\begin{equation}}
\newcommand{\ee}{\end{equation}}

\newcommand{\bd}{\begin{definition}}
\newcommand{\ed}{\end{definition}}

\newcommand{\bi}{\begin{itemize}}
\newcommand{\ei}{\end{itemize}}

\newcommand{\ba}{\begin{ass}}
\newcommand{\ea}{\end{ass}}

\newcommand{\br}{\begin{remark}}
\newcommand{\er}{\end{remark}}

\newcommand{\bp}{\begin{proposition}}
\newcommand{\ep}{\end{proposition}}

\newcommand{\blm}{\begin{lemma}}
\newcommand{\elm}{\end{lemma}}

\newcommand{\bt}{\begin{theorem}}
\newcommand{\et}{\end{theorem}}

\newcommand{\bcor}{\begin{corollary}}
\newcommand{\ecor}{\end{corollary}}

\newcommand{\bex}{\begin{example}}
\newcommand{\eex}{\end{example}}

 \newcommand{\X}{X}
\newcommand{\Y}{Y}
\newcommand{\Z}{Z}
\newcommand{\pset}[1]{\left[#1\right]}
\newcommand{\pY}{{\pset{\Y}}}
\newcommand{\pX}{{\pset{\X}}}
\newcommand{\pZ}{{\pset{\Z}}}

\newcommand{\pp}{\pi}
\newcommand{\mubar}{{\bar{\mu}}}

\newcommand{\fhat}{{\widehat f}}
\newcommand{\fstar}{{f^*}}

\newcommand{\ghat}{{\widehat g}}
\newcommand{\gstar}{{g^*}}
\newcommand{\gbs}{{\bar{g}^*}}

\newcommand{\hstar}{{h^*}}

\newcommand{\gtstar}{{\bar{g}^*}}

\newcommand{\E}{\mathcal{E}}

\newcommand{\Expect}{\mathbb{E}}

\newcommand{\A}{\mathcal{A}}

\newcommand{\R}{\mathbb{R}}

\renewcommand{\P}{P}
\newcommand{\N}{\mathbb{N}}

\renewcommand{\L}{L}
\newcommand{\F}{\mathcal{F}}

\newcommand{\hh}{ {\mathcal{H}} }

\newcommand{\ff}{{\mathcal{F}}}

\newcommand{\K}{{\mathbf{K}}}

\newcommand{\la}{\lambda}

\newcommand{\loss}{\bigtriangleup}

\newcommand{\closs}{{\msf{c}_\loss}}

\newcommand{\rhox}{{\rho_\X}}

\newcommand{\rhoxhat}{ {\widehat\rho_\X} }

\newcommand{\Ctilde}{ {\widetilde {C}} }
\newcommand{\Chat}{ {\widehat {C}} }

\newcommand{\Chatla}{ {\widehat {C}_\la} }
\newcommand{\Cla}{ {{C_\la}} }

\newcommand{\Btilde}{ {\widetilde {B}} }
\newcommand{\Bhat}{ {\widehat {B}} }

\newcommand{\Ltwo}{{L^2}}

\newcommand{\LXP}{{L^2(\X\times\P,\pp\rhox)} }

\newcommand{\LXPH}{{L^2(\X\times\P,\pp\rhox,\hh)} }
\newcommand{\LXPHtilde}{{L^2(\X\times\P,\pp\rhoxhat,\hh)} }

\renewcommand{\gg}{{\cal G}}
\newcommand{\ggbar}{{\overline{\gg}}}

\newcommand{\kbar}{{\bar{k}}}

\newcommand{\scov}{\mathsf{C}}

\renewcommand{\eqals}[1]{\eqal{#1}}

\newcommand{\msf}[1]{\mathsf{#1}}

\newcommand{\mq}{{\msf{s}}}

\renewcommand{\i}{{\mathfrak{i}}}
\newcommand{\ix}{{\i_\X}}

\newcommand{\pt}[2]{{{#1}_{#2}}}

\newcommand{\auxx}{\chi}
\newcommand{\auxy}{\eta}

\usepackage{lipsum}
\usepackage{xcolor}

\declaretheorem[name=Theorem,refname=Thm.]{theorem}
\declaretheorem[name=Lemma,sibling=theorem]{lemma}
\declaretheorem[name=Proposition,refname=Prop.,sibling=theorem]{proposition}
\declaretheorem[name=Remark]{remark}
\declaretheorem[name=Corollary,refname=Cor.,sibling=theorem]{corollary}
\declaretheorem[name=Definition,refname=Def.]{definition}

\declaretheorem[name=Assumption,refname=Asm.]{assumption}

\declaretheorem[name=Example]{example}

\newcommand{\Nystrom}[1]{{Nystr\"om}}

\providecommand{\scal}[2]{\left\langle{#1},{#2}\right\rangle}
\providecommand{\nor}[1]{\bigl\|{#1}\bigr\|}

\newcommand{\proj}{\ensuremath{\text{\rm proj}}}
\newcommand{\tr}{\ensuremath{\text{\rm Tr}}}
\newcommand{\Span}{\ensuremath{\text{\rm span}}}

\newcommand{\esssup}{\ensuremath{\text{\rm ess~sup~ }}}

\newcommand{\argmin}{\operatornamewithlimits{argmin}}

\newcommand{\HS}{{\rm HS}}

\newcommand{\sign}{\ensuremath{\text{\rm sign}}}

\renewcommand{\paragraph}[1]{\ \newline\noindent{{\textbf{#1.}}}}

\renewcommand{\leq}{\leqslant}

\crefname{assumption}{Assumption}{Assumptions}
\crefname{equation}{}{}
\crefname{figure}{Fig.}{Fig.}
\crefname{table}{Table}{Tables}
\crefname{section}{Sec.}{Sec.}
\crefname{theorem}{Thm.}{Thm.}
\crefname{lemma}{Lemma}{Lemmas}
\crefname{corollary}{Cor.}{Cor.}
\crefname{example}{Example}{Examples}
\crefname{remark}{Remark}{Remarks}
\crefname{algorithm}{Alg.}{Algorightms}

\crefname{appendix}{Appendix}{Appendices}
\crefname{subappendix}{Appendix}{Appendices}
\crefname{subsubappendix}{Appendix}{Appendices}

\newcommand{\bl}{BL}
\newcommand{\wl}{WL}

\AfterEndEnvironment{restatable}{\noindent\ignorespacesafterend}
\AfterEndEnvironment{theorem}{\noindent\ignorespacesafterend}
\AfterEndEnvironment{remark}{\noindent\ignorespacesafterend}
\AfterEndEnvironment{example}{\noindent\ignorespacesafterend}
\AfterEndEnvironment{assumption}{\noindent\ignorespacesafterend}
\AfterEndEnvironment{lemma}{\noindent\ignorespacesafterend}
\AfterEndEnvironment{proof}{\noindent\ignorespacesafterend}
\AfterEndEnvironment{corollary}{\noindent\ignorespacesafterend}
\AfterEndEnvironment{proposition}{\noindent\ignorespacesafterend}
\AfterEndEnvironment{definition}{\noindent\ignorespacesafterend}

\title{\Huge\bf Localized Structured Prediction\vspace{1em}}

\author{
Carlo Ciliberto \thanks{Imperial College London - University College London, London, United Kingdom.}\\
{\em\small c.ciliberto@imperial.ac.uk}\\
\and
Francis Bach \thanks{INRIA - Département d’informatique de l’ENS, Ecole normale supérieure, CNRS, INRIA, PSL
	Research University, 75005 Paris, France}\\
{\em\small francis.bach@inria.fr}\\
\and
Alessandro Rudi $^\dag$\\
{\em\small alessandro.rudi@inria.fr}\\
}

\begin{document}

\maketitle

\begin{abstract}
\noindent Key to structured prediction is exploiting the problem's structure to simplify the learning process. A major challenge arises when data exhibit a local structure (i.e., are made ``by parts'') that can be leveraged to better approximate the relation between (parts of) the input and (parts of) the output. Recent literature on signal processing, and in particular computer vision, shows that capturing these aspects is indeed essential to achieve state-of-the-art performance. However, in this context algorithms are typically derived on a case-by-case basis. In this work we propose the first theoretical framework to deal with part-based data from a general perspective and study a novel method within the setting of statistical learning theory. Our analysis is novel in that it explicitly quantifies the benefits of leveraging the part-based structure of a problem on the learning rates of the proposed estimator.
\end{abstract}

\maketitle


\section{Introduction}
\label{sec:introduction}

%
%
%
%
Structured prediction deals with supervised learning problems where the output space is not endowed with a canonical linear metric but has a rich semantic or geometric structure \cite{bakir2007predicting,nowozin2011}. Typical examples are settings in which the outputs correspond to strings (e.g., captioning \cite{karpathy2015deep}), images (e.g., segmentation \cite{alahari2008reduce}), rankings \cite{duchi2010} or protein foldings \cite{joachims2009predicting}.
While the lack of linearity poses several modeling and computational challenges, this additional complexity comes with a potentially significant advantage: when suitably incorporated within the learning model, knowledge about the structure allows to capture key properties of the data. This could potentially lower the sample complexity of the problem, attaining better generalization performance with less training examples. A natural scenario in this sense is the case where both input and output data are organized into ``parts'' that can interact with one another according to a specific structure. Examples can be found in computer vision (e.g., segmentation \cite{alahari2008reduce}, localization \cite{blaschko2008learning,lampert2009efficient}, pixel-wise classification \cite{szummer2008learning}), speech recognition \cite{bahl1986maximum,sutton2012introduction}, natural language processing \cite{tsochantaridis2005}, trajectory planing \cite{ratliff2006maximum} or hierarchical classification \cite{tuia2011structured}. 

Recent literature on the topic has empirically shown that the local structure in the data can indeed lead to significantly better predictions than global approaches \cite{felzenszwalb2010object,vedaldi2009structured}. However in practice, these ideas are typically investigated on a case-by-case basis, leading to ad-hoc algorithms that cannot be easily adapted to new settings. On the theoretical side, few works have considered less specific part-based factorizations \cite{cortes2016structured} and a comprehensive theory analyzing the effect of local interactions between parts within the context of learning theory is still missing.

In this paper, we propose: $1)$ a novel theoretical framework that can be applied to a wide family of structured prediction settings able to capture potential local structure in the data, and $2)$ a structured prediction algorithm, based on this framework for which we prove universal consistency and generalization rates. 
A key contribution of our analysis is to quantify the impact of the part-based structure of the problem on the learning rates of the proposed estimator. In particular, we prove that under natural assumptions on the local behavior of the data, our algorithm benefits {\em adaptively} from this underlying structure. We support our theoretical findings with experiments on the task of detecting local orientation of ridges in images depicting human fingerprints.

\section{Learning with Between- \& Within-locality} \label{sec:locality}

To formalize the concept of locality within a learning problem, in this work we assume that the data is structured in terms of ``parts''. Practical examples of this setting often arise in image/audio or language processing, where the signal has a natural factorization into patches or sub-sequences. Following these guiding examples, we assume every input $x\in\X$ and output $y \in \Y$ to be interpretable as a collection of (possibly overlapping) parts, and denote $\pt x p$ (respectively $\pt y p$) its $p$-th part, with $p\in\P$ a set of part identifiers (e.g., the position and size of a patch in an image). We assume input and output to share same part structure with respect to $\P$. To formalize the intuition that the learning problem should interact well with this structure of parts, we introduce two key assumptions: {\em between-locality} and {\em within-locality}. They characterize respectively the interplay {\em between} corresponding input-output parts and the correlation of parts {\em within} the same input.

  \begin{figure}
    \centering
    \includegraphics[height=12em]{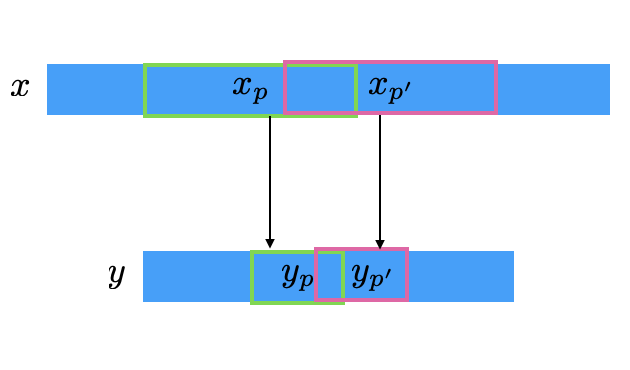}\qquad\qquad
    \includegraphics[height=16em]{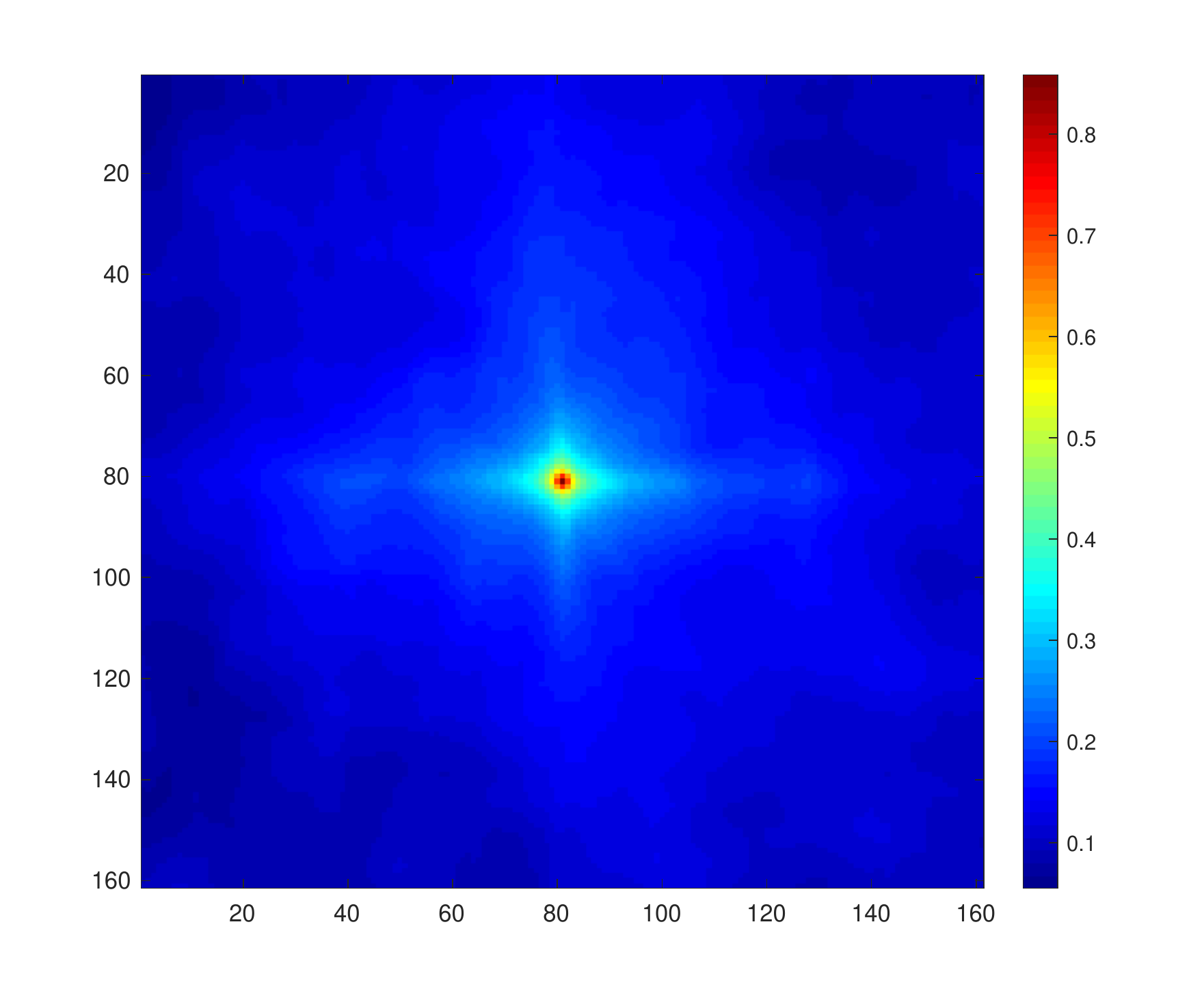}

   \vspace*{-.3cm}
   
    \caption{(Left) Between-locality in a sequence-to-sequence setting: each window (part) $y_p$ of the output sequence $y$ is fully determined by the part $x_p$
     of the input sequence $x$, for every $p\in\P$. (Right) Empirical within-locality $\scov_{p,q}$ of $100$ images sampled from ImageNet between a $20\times20$ patch $q$ and the central patch $p$.}
    \label{fig:locality}
  \end{figure}
  %


\begin{assumption}[Between-locality]\label{asm:between-locality}
$y_p$ is conditionally independent from $x$, given $x_p$, moreover the probability of $y_p$ given $x_p$ is the same as $y_{q}$ given $x_{q}$, for any $p,q \in P$. 
\end{assumption}
\noindent{\em Between-locality (\bl{})} assumes that the $p$-th part of the output $y\in\Y$ depends only on the $p$-th part of the input $x\in\X$, see \cref{fig:locality} (Left) for an intuition in the case of sequence-to-sequence prediction. This is often verified in pixel-wise classification settings, where the class $y_p$ of a pixel $p$ is determined only by the sub-image in the corresponding patch $x_p$. \bl{} essentially corresponds to assuming a joint graphical model on the parts of $x$ and $y$, where each $y_p$ is only connected to $x_p$ but not to other parts.


\bl{} motivates us to focus on a local level by directly learning the relation between input-output parts. This is often an effective strategy in computer vision \cite{lampert2009efficient,vedaldi2009structured,felzenszwalb2010object} but intuitively, one that provides significant advantages only when the input parts are not highly correlated with each other: in the extreme case where all parts are identical, there is no advantage in solving the learning problem locally. 
In this sense it can be useful to measure the amount of ``covariance''
\eqal{\label{eq:within-locality-qp}
\scov_{p,q} = \mathbb{E}_x ~{S(x_p,x_q)} - \mathbb{E}_{x,x'} ~S(x_p, x'_q)
}
between two parts $p$ and $q$ of an input $x$, for $S(x_p,x_q)$ a suitable measure of similarity between parts (if $S(x_p,x_q) = x_p x_q$, with $x_p$ and $x_q$ scalars random variables, then $\scov_{p,q}$ is the $p,q$-th entry of the covariance matrix of the vector $(x_1,\dots,x_{|\P|})$~).
Here $\Expect_x S(x_p,x_q)$ and $\Expect_{x,x'} S(x_p,x'_q)$ measure the similarity between the $p$-th and the $q$-th part of, respectively, the {\em same} input, and two {\em independent} ones (in particular $\scov_{p,q}=0$ when the $p$-th and $q$-th part of $x$ are independent).
In many applications, it is reasonable to assume that $\scov_{p,q}$ decays according to the distance between $p$ and $q$. 

\begin{assumption}[Within-locality]\label{asm:within-locality}
There exists a distance $d:\P\times\P\to\R$ and $\gamma \geq 0$, such that
\eqal{\label{eq:asm-within-locality}
  \left|\scov_{p,q}\right| ~\leq ~ {\msf r}^2 ~e^{-\gamma d(p,q)} \qquad\text{with}\qquad \msf r^2 = \sup_{x,x'} |S(x,x')|.
}
\end{assumption}
\noindent{\em Within-locality (\wl{})} is always satisfied for $\gamma = 0$. However, when $x_p$ is independent of $x_q$, it holds with $\gamma = \infty$ and $d(p,q) = \delta_{p,q}$ the Dirac's delta. Exponential decays of correlation are typically observed when the distribution of the parts of $x$ factorizes in a graphical model that connects parts which are close in terms of the distance $d$: although all parts depend on each other, the long-range dependence typically goes to zero exponentially fast in the distance (see, e.g.,~\cite{meyn2012markov} for mixing properties of Markov chains). \cref{fig:locality} (Right) reports the empirical \wl{} measured on $100$ images randomly sampled from ImageNet \cite{deng2009imagenet}: each pixel $(i,j)$ reports the value of $\scov_{p,q}$ of the central patch~$p$ with respect to a $20\times20$ patch $q$ centered in $(i,j)$. Here $S(x_p,x_q) = x_p^\top x_q$. We note that $\scov_{p,q}$ decreases extremely fast as a function of the distance $\nor{p-q}$, suggesting that \cref{asm:within-locality} holds with a large value of $\gamma$. 

\paragraph{Contributions} In this work  we present a novel structured prediction algorithm that adaptively leverages locality in the learning problem, when present (\cref{sec:algorithm}). 
%
%
We study the generalization properties of the proposed estimator (\cref{sec:theory}), showing that it is equivalent to the state of the art in the worst case scenario. More importantly, if the locality \cref{asm:between-locality,asm:within-locality} are satisfied, we prove that our learning rates improve proportionally to the number $|\P|$ of parts in the problem. Here we give an informal version of this main result, reported in more detail in \cref{thm:rates-improved-with-parts} (\cref{sec:theory}). Below we denote by $\fhat$ the proposed estimator, by $\E(f)$ the expected risk of a function $f:\X\to\Y$ and $\fstar = \argmin_{f} \E(f)$.
\begin{theorem}[Informal - Learning Rates \& Locality]\label{thm:motivation}
Under mild assumptions on the loss and the data distribution, if the learning problem is local (\cref{asm:between-locality,asm:within-locality}), there exists $c_0>0$ such that
\eqal{\label{eq:informal-bound}
  \Expect~\left[ \E(\fhat~) - \E(\fstar)\right]  ~\leq~ c_0 \left( \frac{\mq}{ n |P|}\right)^{1/4}, \qquad \mq = \frac{\mathsf{r}^2}{|P|}\sum_{p,q=1}^{|P|}~ e^{-\gamma d(p,q)},
}
\noindent where the expectation is taken with respect to the sample of $n$ input-output points used to train $\fhat$.
\end{theorem}
\noindent In the worst-case scenario $\gamma=0$ (no exponential decay of the covariance between parts), the bound in \cref{eq:informal-bound} scales as $1/n^{1/4}$ (since $\mq = \msf r^2 |P|$) recovering \cite{ciliberto2016}, where no structure is assumed on the parts. However, as soon as $\gamma>0$, $\mq$ can be upper bounded by a constant independent of $|P|$ and thus the rate scales as $1/(|P|n)^{1/4}$, accelerating proportionally to the number of parts. In this sense, \cref{thm:motivation} shows the significant benefit of making use of locality. The following example focuses on the special case of sequence-to-sequence prediction.

\begin{example}[Locality on Sequences]\label{ex:locality-on-sequences}
As depicted in \cref{fig:locality}, for discrete sequences we can consider parts (e.g., windows) indexed by $\P=\{1,\dots,|\P|\}$, with $d(p,q) = |p-q|$ for $p,q\in\P$ (see \cref{sec:example-sequence-to-sequence} for more details). In this case, \cref{asm:within-locality} leads to
\eqals{\label{eq:locality-on-sequences-example}
\mq ~~\leq~~ 2\mathsf{r}^2 (1-e^{-\gamma})^{-1},
}
which for $\gamma > 0$ is bounded by a constant not depending on the number of parts. Hence, \cref{thm:motivation} guarantees a learning rate of order $1/(n|P|)^{1/4}$, which is significanlty faster than the rate $1/n^{1/4}$ of methods that do not leverage locality such as \cite{ciliberto2016}. See \cref{sec:exp} for empirical support to this observation.
\end{example}

\section{Problem Formulation}\label{sec:problem-formulation}
  
We denote by $\X, \Y$ and $\Z$ respectively the {\em input space}, {\em label space} and {\em output space} of a learning problem. Let $\rho$ be a probability measure on $\X \times \Y$ and $\loss:\Z \times \Y \times \X \to \R$ a loss measuring prediction errors between a label $y\in\Y$ and a output $z\in\Z$, possibly parametrized by an input $x\in\X$. To stress this interpretation we adopt the notation $\loss(z,y|x)$. Given a finite number of $(x_i, y_i)_{i=1}^n$ independently sampled from $\rho$, our goal is to approximate the minimizer $\fstar$ of the {\em expected risk}
\eqal{\label{eq:base-fstar}%
\min_{f:\X\to\Z} \E(f), \quad \textrm{with} \quad \E(f) = \int \loss(f(x), y|x) ~d\rho(x,y).%
}

\paragraph{Loss Made by Parts} We formalize the intuition introduced in \cref{sec:locality} that data are decomposable into parts: we denote the sets of {\em parts} of $\X,\Y$ and $\Z$ by respectively $\pX,\pY$ and $\pZ$. These are abstract sets that depend on the problem at hand (see examples below). We assume $\P$ to be a set of part ``indices'' equipped with a selection operator $\X\times\P\to\pX$ denoted $(x,p)\mapsto[x]_p$ (analogously for $\Y$ and $\Z$). When clear from context, we will use the shorthand $x_p = [x]_p$. For simplicity, in the following we will assume $\P$ be finite, however our analysis generalizes also to the infinite case (see supplementary material). Let $\pi(\cdot|x)$ be a probability distribution over the set of parts $\P$, conditioned with respect to an input $x\in\X$. We study loss functions $\loss$ that can be represented as 
\eqal{\label{eq:loss-base-parts}
\loss(z,y|x) = \sum_{p\in\P} ~ \pp(p|x) ~\L_p (\pt z p,\pt yp |~ \pt xp).
}
The collection of $(\L_p)_{p\in P}$ is a family of loss functions $\L_p:\pZ\times\pY\times\pX\to\R$, each comparing the $p$-th part of a label $y$ and output $z$.  For instance, in an image processing scenario, $\L_p$ could measure the similarity between the two images at different locations and scales, indexed by $p$. In this sense, the distribution $\pi(p|x)$ allows to weigh each $\L_p$ differently depending on the application (e.g., mistakes at large scales could be more relevant than at lower scales). Various examples of parts and concrete cases are illustrated in the supplementary material,
here we report an extract.

\begin{example}[Sequence to Sequence Prediction]\label{ex:intro-learning-sequences}
  Let $\X = A^k$, $\Y = \Z =  B^k$ for two sets $A, B$ and $k \in \N$ a fixed length. We consider in this example parts that are windows of length $l \leq k$. Then $P = \{1,\dots, k-l+1\}$ where $p \in P$ indexes the window $x_p = (x^{(p)},\dots, x^{(p+l-1)})$, with $x \in \X$, where we have denoted $x^{(s)}$ the $s$-th entry of the sequence $x\in X$, analogous definition for $y_p, z_p$. Finally, we choose the loss $L_p$ to be the 0-1 distance between two strings of same length $L_p(z_p,y_p|x) = \boldsymbol{1}(z_p \neq y_p)$. Finally, we can choose $\pi(p|x) = 1/|P|$, leading to a loss function
  $\loss(z,y|x) = \frac{1}{|P|} \sum_{p\in\P} ~ \boldsymbol{1}(\pt z p \neq \pt yp)$, which is common in the context of CRF \cite{lafferty2001conditional}.
\end{example}  
\begin{remark}[Examples of Loss Functions by Parts]
Several loss functions used in machine learning have a natural formulation in terms of \cref{eq:loss-base-parts}. Notable examples are the Hamming distance \cite{collins2004parameter,taskar2004max,cortes2014ensemble}, used in settings such as hierarchical classification \cite{tuia2011structured}, computer vision \cite{nowozin2011,vedaldi2009structured,szummer2008learning} or trajectory planning \cite{ratliff2006maximum} to name a few. Also, loss functions used in natural language processing, such as the precision/recall and F$1$ score can be written in this form. Finally, we point out that multi-task learning settings \cite{micchelli2004} can be seen as problem by parts, with the loss corresponding to the sum of standard regression/classification loss functions (least-squares, logistic, etc.) over the tasks/parts. 
\end{remark}

\begin{algorithm}[t]
\caption{~}
\caption{ -- {\sc Localized Structured Prediction} }
\label{alg:self-learning}
\begin{algorithmic}
  \State ~
  \State {\bfseries Input:} training set $(x_i,y_i)_{i=1}^n$, distributions $\pi(\cdot|x)$ a reproducing kernel $k$ on $\X\times\P$, hyperparameter $\la>0$, auxiliary dataset size $m\in\N$.
  \State~
  \State {\sc Generate} the auxiliary set $(\eta_j,\chi_j,p_j)_{j=1}^m$:
  \State \quad Sample $i_j\in U_n(\cdot)$. Set $\chi_j = x_{i_j}$.
  \State \quad Sample $p_j\sim\pp(\cdot|\chi_j)$. Set $\eta_j = [y_{i_j}]_{p_j}$.
  
  \State~
  
  \State {\sc Learn} the coefficients for the map $\alpha$:
  \State \quad Set $\K$ with $\K_{jj'}\,=\,k((\chi_j,p_j),(\chi_{j'},p_{j'}))$.
  \State \quad $\mathbf{A} = (\K + m\la I)^{-1}$.
  
  \State~
  \State {\bfseries Return} the map $\alpha:(x,p) \mapsto \mathbf{A} ~ v(x,p) \in\R^m$ 
  \State \quad with $v(x,p)_j = k\big((\chi_j,p_j),(x,p)\big)$. 
  
\end{algorithmic}
\end{algorithm}

\section{Algorithm}\label{sec:algorithm}
  
In this section we introduce our estimator for structured prediction problems with parts. Our approach starts with an auxiliary step for dataset generation that explicitly extracts the parts from the data.

\paragraph{Auxiliary Dataset Generation} The locality assumptions introduced in \cref{sec:locality} motivate us to learn the local relations between individual parts $p\in\P$ of each input-output pair. In this sense, given a training dataset ${\cal D} = (x_i,y_i)_{i=1}^n$ a first step would be to extract a new, part-based dataset $\{ (x_p, p, y_p)~|~ (x,y) \in {\cal D},~ p \in \P\}$. However in most applications the cardinality $|P|$ of the set of parts can be very large (possibly infinite as we discuss in the Appendix) making this process impractical. Instead, we generate an {\em auxiliary dataset} by randomly sub-sampling $m \in \N$ elements from the part-based dataset. Concretely, for $j \in \{1,\dots, m\}$, we first sample $i_j$ according to the uniform distribution $U_{n}$ on $\{1,\dots,n\}$, set $\auxx_j = x_{i_j}$, sample $p_j\sim\pi(\cdot~|~\auxx_j)$ and finally set $\auxy_j = [y_{i_j}]_{p_j}$. This leads to the auxiliary dataset ${\cal D}' = (\auxx_j,p_j,\auxy_j)_{j=1}^m$, as summarized in the {\sc Generate} routine of \cref{alg:self-learning}.

\paragraph{Estimator} Given the auxiliary dataset, we propose the estimator $\fhat:\X\to\Z$, such that $\forall x\in\X$
\eqal{\label{eq:estimator}
\fhat(x) = \argmin_{z\in\Z} \sum_{p \in P} \sum_{j=1}^m \alpha_j(x,p)~  \Big[\pp(p|x)~\L_p(\pt z p, \auxy_j | \pt xp)\Big].
}
%
The functions $\alpha_j:\X\times\P\to\R$ are {\em learned} from the auxiliary dataset and are the fundamental components allowing our estimator to capture the part-based structure of the learning problem. Indeed, for any test point $x\in\X$ and part $p\in\P$, the value $\alpha_j(x,p)$ can be interpreted as a measure of how similar $x_p$ is to the $p_j$-th part of the auxiliary training point $\auxx_j$. For instance, assume $\alpha_j(x,p)$ to be an approximation of the delta function that is $1$ when $\pt xp = \pt {[\auxx_j]}{p_j}$ and $0$ otherwise. Then, 
\eqals{
\alpha_j(x,p)~\L_p(\pt z p, \auxy_j | \pt xp) ~~\approx~~  \delta(\pt xp, [\auxx_j]_{p_j}) ~\L_p(\pt z p, \auxy_j | \pt xp), 
}
%
which implies essentially that
\eqals{
x_p \approx [\auxx_j]_{p_j} ~~\Longrightarrow~~ z_p \approx \auxy_j.
}
In other words, if the $p$-th part of test input $x$ and the $p_j$-th part of the auxiliary training input $\auxx_j$ (i.e., the $p_j$-th part of the training input $x_{i_j}$) are deemed similar, then the estimator will encourage the $p$-th part of the test output $z$ to be similar to the auxiliary part $\auxy_j$. This process is illustrated in \cref{fig:impliedsimilarity} for an ideal computer vision application: for a given test image $x$, the $\alpha$ scores detect a similarity between the $p$-th patch of $x$ and the $p_j$-th patch of the training input $x_{i_j}$. Hence, the estimator will enforce the $p$-th patch of the output $z$ to be similar to the $p_j$-th patch of the training label $y_{i_j}$.

\begin{figure}[t]
    \vspace{0pt}
  	\centering
  	\includegraphics[trim={1cm 3.5cm 1.5cm 2cm},clip,height=14em]{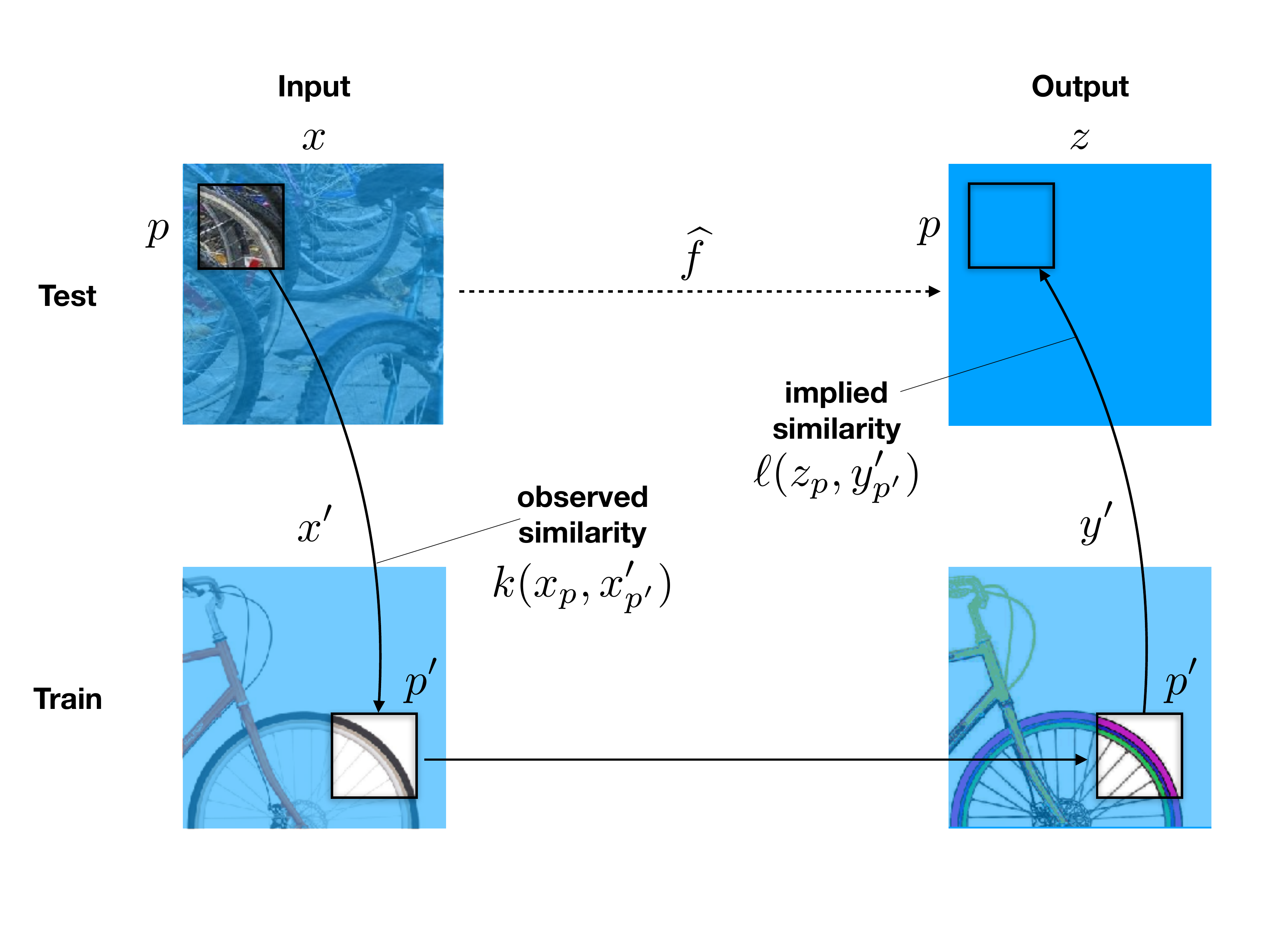}
  	\caption{Illustration of the prediction process for the Localized Structured Prediction Estimator \cref{eq:estimator} for a hypothetical computer vision application.}
  	\label{fig:impliedsimilarity}
\end{figure}

\paragraph{Learning $\alpha$} In line with previous work on structured prediction \cite{ciliberto2016}, we learn each $\alpha_j$ by solving a linear system for a problem akin to kernel ridge regression (see \cref{sec:theory} for the theoretical motivation). In particular, let $k:(\X\times\P)\times(\X\times\P)\to\R$ be a positive definite kernel, we define
\eqal{\label{eq:estimator-alpha}
(\alpha_1(x,p),\dots,\alpha_m(x,p))^\top = (\K + m\la I)^{-1}v(x,p) ,
}
where $\K\in\R^{m\times m}$ is the empircal kernel matrix with entries $\K_{jh} = k((\auxx_j,p_j),(\auxx_h,p_{h}))$ and $v(x,p)\in\R^m$ is the vector with entries $v(x,p)_j = k((\auxx_j,p_j),(x,p))$. Training the proposed algorithm, consists in precomputing $\mathbf{A} = (\K + m\la I)^{-1}$ to evaluate the coefficients $\alpha$ as detailed by the {\sc Learn} routine in \cref{alg:self-learning}. While computing $\mathbf{A}$ amounts to solving a linear system, which requires $O(m^3)$ operations, we note that it is possible to achieve the same statistical accuracy with reduced complexity $O(m\sqrt{m})$ by means of low rank approximations (see \cite{dieuleveut2017harder,rudi2017falkon}). 


\begin{remark}[\bf Evaluating $\fhat$]\label{rem:evaluating-fhat}
According to \eqref{eq:estimator}, evaluating $\fhat$ on a test point $x\in\X$ consists in solving an optimization problem over the output space $\Z$. This is a standard strategy in structured prediction, where an optimization protocol is derived on a case-by-case basis depending on both $\loss$ and $\Z$ (see, e.g., \cite{nowozin2011}). However, the specific form of our estimator suggests a general stochastic meta-algorithm to address this problem,. In particular, we can reformulate \cref{eq:estimator} as
\eqal{\label{eq:evaluating-as-sgd}
\fhat(x) = \argmin_{z\in\Z} ~ \mathbb{E}_{j,p}~ h_{j,p}(z|x)
,}
with $p$ sampled according to $\pp$, $j\in\{1,\dots,m\}$ sampled according to the weights $\alpha_j$ and $h_{j,p}$ suitably defined in terms of $\L_p$. When the $h_{j,p}$ are (sub)differentiable, \eqref{eq:evaluating-as-sgd} can be effectively addressed by stochastic gradient methods (SGM). In \cref{alg:sgm-hatf-evaluation} in \cref{sec:app-algorithm} we give an example of this strategy.
\end{remark}

  
  
  
  

\section{Generalization Properties of Structured Prediction with Parts}\label{sec:theory}

In this section we study the statistical properties for the proposed algorithm, with particular attention to the impact of locality on learning rates, see \cref{thm:rates-improved-with-parts} (for a complete analysis of univeral consistency and learning rates without locality assumptions, see \cref{sec:app-consistency,sec:app-learning-rates}). Our analysis leverages the assumption that the loss function $\loss$ is a {\em Structure Encoding Loss Function (SELF) by Parts}.
\begin{definition}[SELF by Parts]\label{def:self-simple}
A function $\loss:\Z\times\Y\times\X\to\R$ is a {\em Structure Encoding Loss Function (SELF) by Parts} if it admits a factorization in the form of \eqref{eq:loss-base-parts} with functions $\L_p:\pZ\times\pY\times\pX\to\R$, and there exists a separable Hilbert space $\hh$ and two bounded maps $\psi:\pZ\times\pX\times\P\to\hh$, $\varphi:\pY\to\hh$ such that for any $\zeta \in \pZ$, $\eta \in \pY$, $\xi \in \pX$, $p \in \P$
\eqal{\label{eq:self-simple}
  L_p(\zeta,\eta|\xi) ~~=~~ \scal{\psi(\zeta,\xi,p)}{\varphi(\eta)}_\hh.
}
\end{definition}
%
%
%
\noindent The definition of ``SELF by Parts'' specializes the definition of SELF in \cite{ciliberto2017consistent} and in the following we will always assume $\loss$ to satisfy it. Indeed, \cref{def:self-simple} is satisfied when the spaces of parts involved are discrete sets and it is rather mild in the general case (see \cite{ciliberto2016} for an exhaustive list of examples). 
Note that when $\loss$ is SELF, the solution of \cref{eq:base-fstar} is completely characterized in terms of the conditional expectation (related to the {\em conditional mean embedding} \cite{caponnetto2007,lever2012conditional,song2013kernel}) of $\varphi(y_p)$ given $x$, denoted by $\gstar:\X\times \P \to \hh$, as follows.
\begin{restatable}{lemma}{LCharacterizationFstar}\label{lm:char-fstar}
Let $\loss$ be SELF and $\Z$ compact. Then, the minimizer of \cref{eq:base-fstar} is $\rhox$-a.e. characterized by
\eqal{\label{eq:char-fstar}
\fstar(x) = \argmin_{z \in Z} \sum_{p \in P} \pi(p|x) \scal{\psi(z_p,x_p,p)}{\gstar(x,p)}_\hh, \qquad \gstar(x,p) = \int_\Y \varphi(y_p) d\rho(y|x).
}
\end{restatable}
\noindent\cref{lm:char-fstar} (proved in \cref{app:derivation}) shows that $\fstar$ is completely characterized in terms of the conditional expectation $\gstar$, which indeed plays a key role in controlling the learning rates of $\fhat$. 
%
%
%
%
In particular, we investigate the learning rates in light of the two assumptions of between- and within-locality introduced in \cref{sec:locality}. To this end, we first study the direct effects of these two assumptions on the learning framework introduced in this work.

\paragraph{The effect of Between-locality} We start by observing that the between-locality between parts of the inputs and parts of the output allows for a refined characterization of the conditional mean $\gstar$.

\begin{lemma}\label{prop:gstar-as-gbs}
  Let $\gstar$ be defined as in \cref{eq:char-fstar}. Under \cref{asm:between-locality}, there exists $\gbs:\pX\to\hh$ such that 
  \eqal{\label{eq:gbs-defintion}
    \gstar(x,p) = \gbs(x_p) \qquad\qquad \forall x\in\X,~ p\in\P.
  }
\end{lemma}
\noindent\cref{prop:gstar-as-gbs} above shows that we can learn $\gstar$ by focusing on a ``simpler'' problem, identified by the function $\gbs$ acting only the parts $\pX$ of $\X$ rather than on the whole input directly (for a proof see \cref{lem:general-gstar-as-gbs} in \cref{sec:app-learning-rates-parts}). This motivates the adoption of the restriction kernel \cite{blaschko2008learning}, namely a function $k:(\X\times\P)\times(\X\times\P)\to\R$ such that
\eqal{\label{eq:restriction-kernel}
  k((x,p),(x',q)) = \bar k(x_p,x_q),
}
which, for any pair of inputs $x,x'\in\X$ and parts $p,q\in\P$, measures the similarity between the $p$-part of $x$ and the $q$-th part of $q$ via a kernel $\bar k:\pX\times\pX\to\R$ on the parts of $\X$. The restriction kernel is a well-established tool in structured prediction settings \cite{blaschko2008learning} and it has been observed to be remarkably effective in computer vision applications \cite{lampert2009efficient,vedaldi2009structured,felzenszwalb2010object}.

\paragraph{The effect of Within-locality} We recall that within-locality characterizes the statistical correlation between two different parts of the input (see \cref{asm:within-locality}). 
To this end we consider the simplified scenario where the parts are sampled from the uniform distribution on $\P$, i.e., $\pi(p|x) = \frac{1}{|P|}$ for any $x\in\X$ and $p\in\P$. While more general situations can be considered, this setting is useful to illustrate the effect we are interested in this work. We now define some important quantities that characterize the learning rates under locality,
\eqal{\label{eq:within-locality-measures}
    \scov_{p,q} =  \Expect_{x,x'} \left[ ~\bar k(x_p,x_q)^2 - \bar k(x_p,x'_q)^2 ~ \right], \qquad \mathsf{r} = \sup_{x \in\X, p \in \P} \bar k(x_p,x_p).
  }
 It is clear that the terms $\scov_{p,q}$ and $\msf r$ above correspond respectively to the correlations introduced in \cref{eq:within-locality-qp} and the scale parameter introduced in \cref{eq:asm-within-locality}, with similarity function $S = \bar k^2$.
%
%
%
Let $\fhat$ be the structured prediction estimator in \cref{eq:estimator} learned using the restriction kernel in \cref{eq:restriction-kernel} based on $\bar{k}$ and denote by $\bar{\gg}$ the space of functions $\bar{\gg} = \hh \otimes \bar{\ff}$ with $\bar{\ff}$ the reproducing kernel Hilbert space \cite{aronszajn1950theory} associated to $\bar{k}$. In particular, in the following we will consider the standard assumption in the context of non-parametric estimation \cite{caponnetto2007} on the regularity of the target function, which in our context reads as $\gbs \in \bar{\gg}$. Finally we introduce $\msf{c}_\loss^2 = \sup_{z\in\Z,x\in\X} \frac{1}{|P|} \sum_{p \in P} \|\psi(z,x,p)\|_\hh^2$ to measure the ``complexity'' of the loss~$\loss$ w.r.t. the representation induced by SELF decomposition (\cref{def:self-simple}) analogously to Thm.~2 of \cite{ciliberto2016}.
\begin{restatable}[Learning Rates \& Locality]{theorem}{TRatesImprovedWithParts}\label{thm:rates-improved-with-parts}
   Under \cref{asm:between-locality,asm:within-locality} with $S = \bar{k}^2$, let $\gbs$ satisfying \cref{prop:gstar-as-gbs}, with  $\bar{\msf{g}} = \|\gbs\|_{\bar{\gg}} < \infty$. Let $\mq$ be as in \cref{eq:informal-bound}. When $\la = (\mathsf{r}^2/m + \mq/(|P|n))^{1/2}$, then
  \eqal{
    \mathbb{E}~{\cal E}(\fhat~)- {\cal E}(\fstar)  ~\leq~  12 ~\closs~ \mathsf{\bar{g}} ~ \left(\frac{\msf r^2}{m} + \frac{\msf r^2}{|P|n} + \frac{\mq}{|P|n} \right)^{1/4}.
  }
\end{restatable}
\noindent The proof of the result above can be found in \cref{sec:proof-rates-locality}. 
We can see that between- and within-locality allow to refine (and potentially improve) the bound of $n^{-1/4}$ from structured prediction without locality \cite{ciliberto2016} (see also \cref{thm:rates} in \cref{sec:app-learning-rates}). In particular, we observe that the adoption of the restriction kernel in \cref{thm:rates-improved-with-parts} allows the structured prediction estimator to leverage the within-locality, gaining a benefit proportional to the magnitude of the parameter $\gamma$. Indeed $\mathsf{r}^2 \leq \mq \leq \mathsf{r}^2 |P|$ by definition. More precisely, if $\gamma = 0$ (e.g., all parts are identical copies) then $\mq = \mathsf{r}^2 |P|$ and we recover the rate of $O(n^{-1/4})$ of \cite{ciliberto2016}, while if $\gamma$ is large (the parts are almost not correlated) then $\mq = \mathsf{r}^2$ and we can take $m\propto n|\P|$ achieving a rate of the order of $O\big( (n|\P|)^{-1/4}\big)$. We clearly see that depending on the amount of within-locality in the learning problem, the proposed estimator is able to gain significantly in terms of finite sample bounds.

\section{Empirical Evaluation}\label{sec:exp}

We evaluate the proposed estimator on simulated as well as real data. We highlight how locality leads to improved generalization performance, in particular when only few training examples are available.



\begin{figure}[t]
    \centering
    \includegraphics[height=12em]{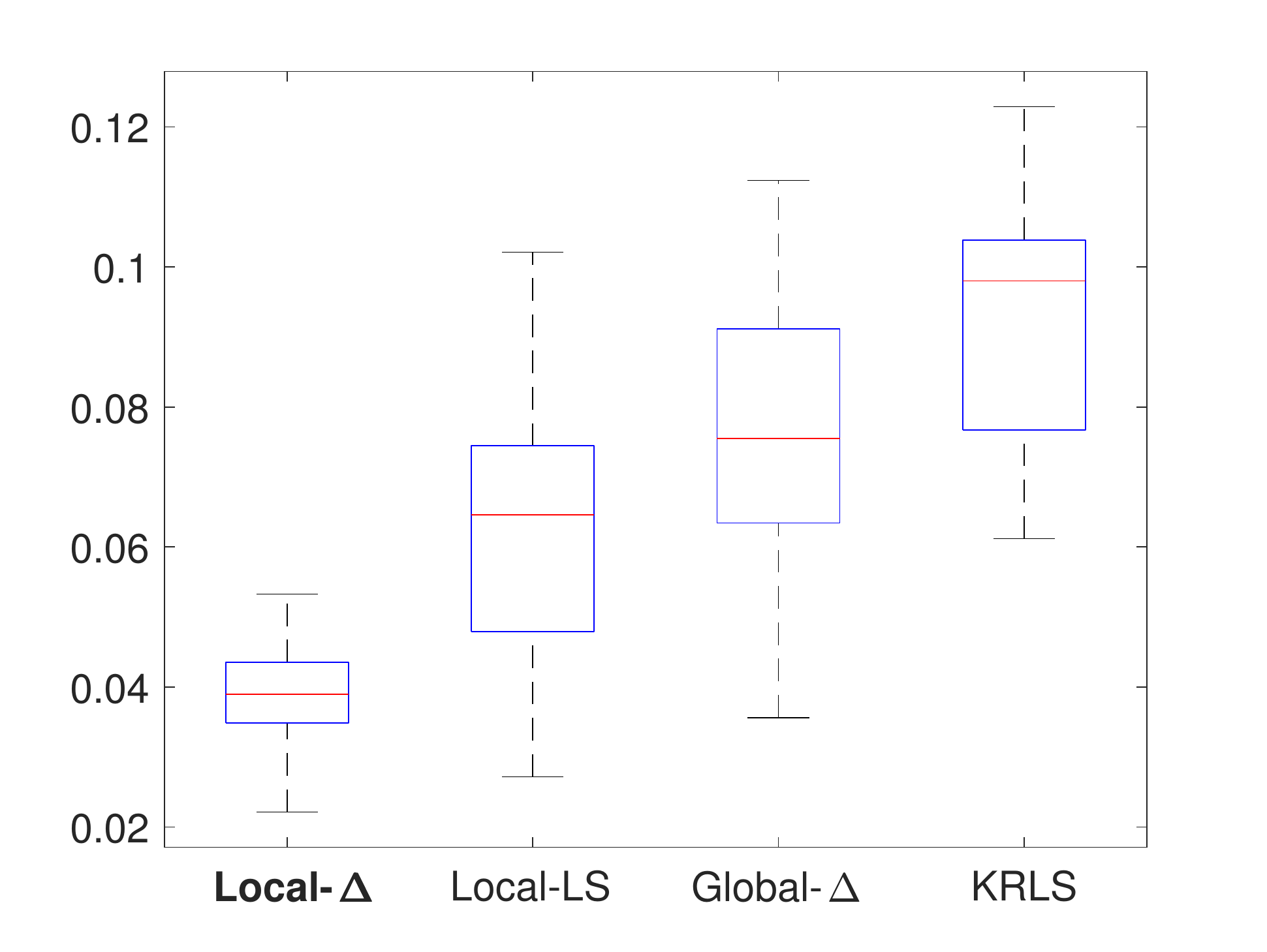}\quad
    \includegraphics[height=14em]{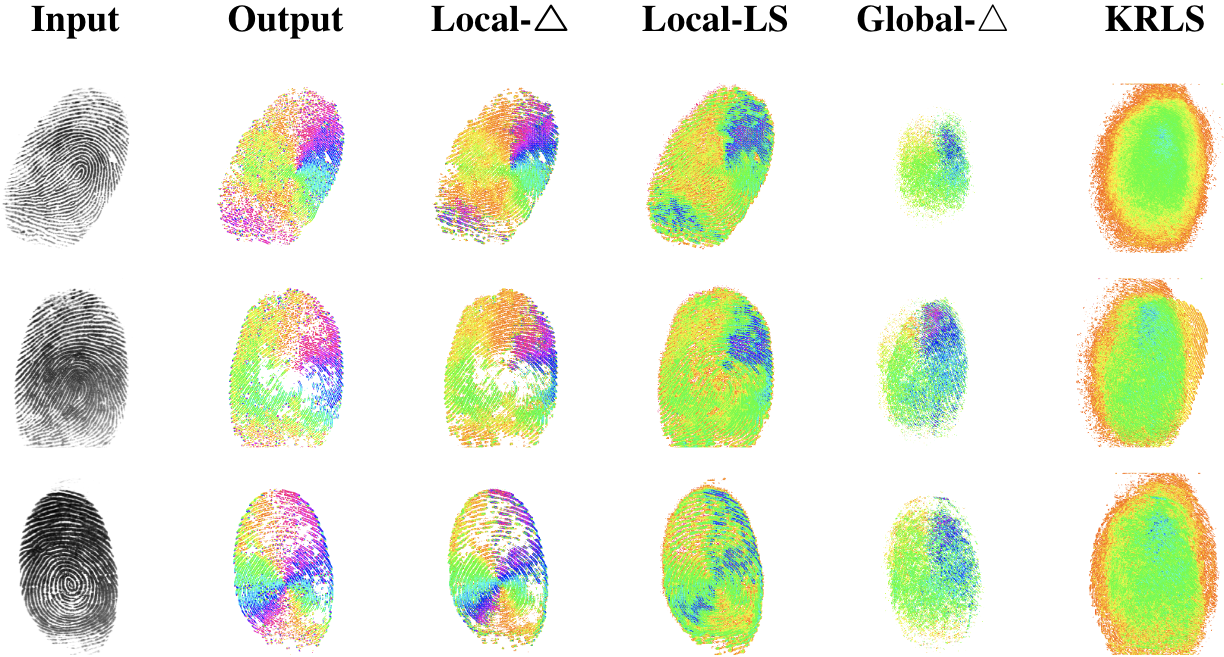}
    \caption{Learning the direction of ridges in fingerprint images. (Left) Examples of ground truths and predictions with pixels' color corresponding to the local direction of ridges. (Right) Test error according to $\loss$ in \cref{eq:fingerprint-loss-parts}. \label{fig:fingerprints}\label{fig:MSE}}
\end{figure}

\paragraph{Learning the Direction of Ridges for Fingerprint}
Similarly to \cite{steinke2010}, we considered the problem of detecting the pointwise direction of ridges in a fingerprint image on the FVC04 dataset\footnote{\url{http://bias.csr.unibo.it/fvc2004}, \texttt{DB1\_B}. The output is obtained by applying $7\times 7$ Sobel filtering.} comprising $80$ grayscale $640\times 480$ input images depicting fingerprints and corresponding output images encoding in each pixel the local direction of the ridges of the input fingerprint as an angle $\theta\in[-\pi,\pi]$. A natural loss function is the average pixel-wise error $\sin(\theta-\theta')^2$ between a ground-truth angle $\theta$ and the predicted $\theta'$ according to the geodesic distance on the sphere. To apply the proposed algorithm, we consider the following representation of the loss in term of parts: let $P$ be the collection of patches of dimension $20\times 20$ and equispaced each $5\times 5$ pixels\footnote{For simplicity we assume ``circular images'', namely $[x]_{i, j} =  [x]_{(i \! \mod 640), (j \! \mod 480)}$.} so that each pixel belongs exactly to $16$ patches. For all $z,y\in\R^{640\times 480}$, the average pixel-wise error is
\eqal{\label{eq:fingerprint-loss-parts}
\loss(z,y) = \frac{16}{|P|} \sum_{p \in P} \L(z_p, y_p),\qquad\text{with}\qquad 
\L(\zeta,\eta)=\frac{1}{20\times 20} \sum_{i,j=1}^{20} \sin([\zeta]_{ij}, [\eta]_{ij})^2,
} 
where $\zeta = z_p,\eta=y_p \in [-\pi,\pi]^{20\times 20}$ are the extracted patches and  $[\cdot]_{ij}$ their value at pixel $(i,j)$.




We compared our approach using $\loss$ ({\em Local-$\loss$}) or least-squares ({\em Local-LS}) with competitors that do not take into account the local structure of the problem, namely standard vector-valued kernel ridge regression ({\em KRLS}) \cite{caponnetto2007} and the structured prediction algorithm in \cite{ciliberto2016} with $\loss$ loss ({\em $\loss$-Global}). We used a Gaussian kernel on the input (for the local estimators the restriction kernel in \cref{eq:restriction-kernel} with $\bar k$ Gaussian). We randomly sampled $50/30$ images for training/testing, performing $5$-fold cross-validation on $\lambda$ in $[10^{-6},10]$ (log spaced) and the kernel bandwidth in $[10^{-3},1]$. For Local-$\loss$ and Local-LS we built an auxiliary set with $m=30 000$ random patches (see \cref{sec:algorithm}), sampled from the $50$ training images.

\paragraph{Results} \cref{fig:fingerprints} (Left) reports the average prediction error across $10$ random train-test splits. We make two observations: first, methods that leverage the locality in the data are consistently superior to their ``global'' counterparts, supporting our theoretical results in \cref{sec:theory} that the proposed estimator can lead to significantly better performance, in particular when few training points are available. Second, the experiment suggests that choosing the right loss is critical, since exploiting locality without the right loss (i.e., Local-LS in the figure) generally leads to worse performance. The three sample predictions in \cref{fig:fingerprints} (Right) provide more qualitative insights on the models tested. In particular while both locality-aware methods are able to recover the correct structure of the fingerprints, only combining this information with the loss $\loss$ leads to accurate recovery of the ridge orientation.

\begin{figure}[t]
    \begin{minipage}[t]{0.45\columnwidth}
        \centering
    	\includegraphics[height=6cm,trim={1cm 0 1em 0},clip]{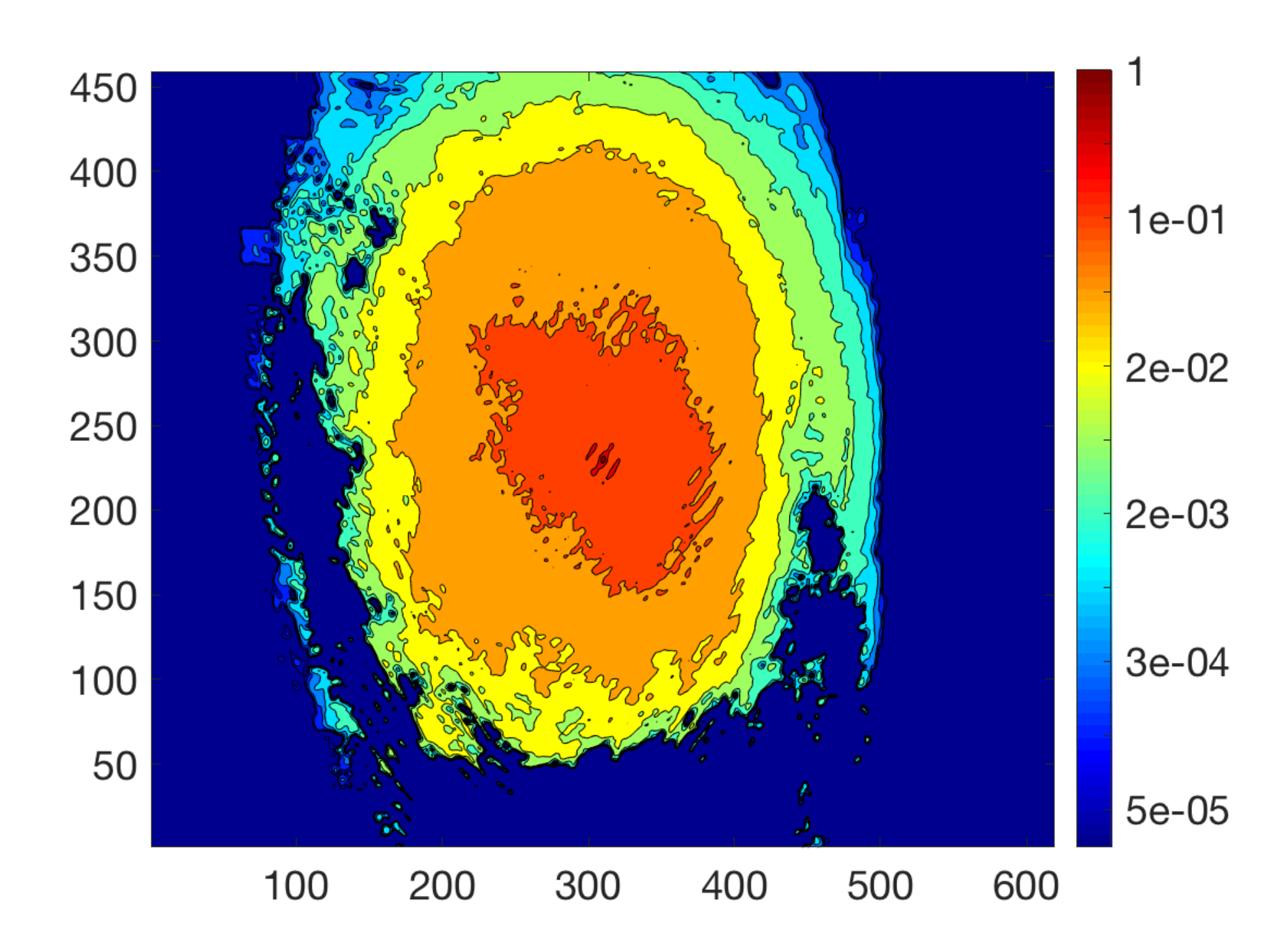}
    	\vspace*{-.2cm}
		\caption{ Empirical estimation of {\em within-locality} for the central patch of the fingerprints dataset.\label{fig:within-fingerprints}}
    \end{minipage}
    ~~~~
    \begin{minipage}[t]{0.5\columnwidth}
        \centering
        \includegraphics[height=5.5cm,trim={1cm 0cm 1.75cm 1cm},clip]{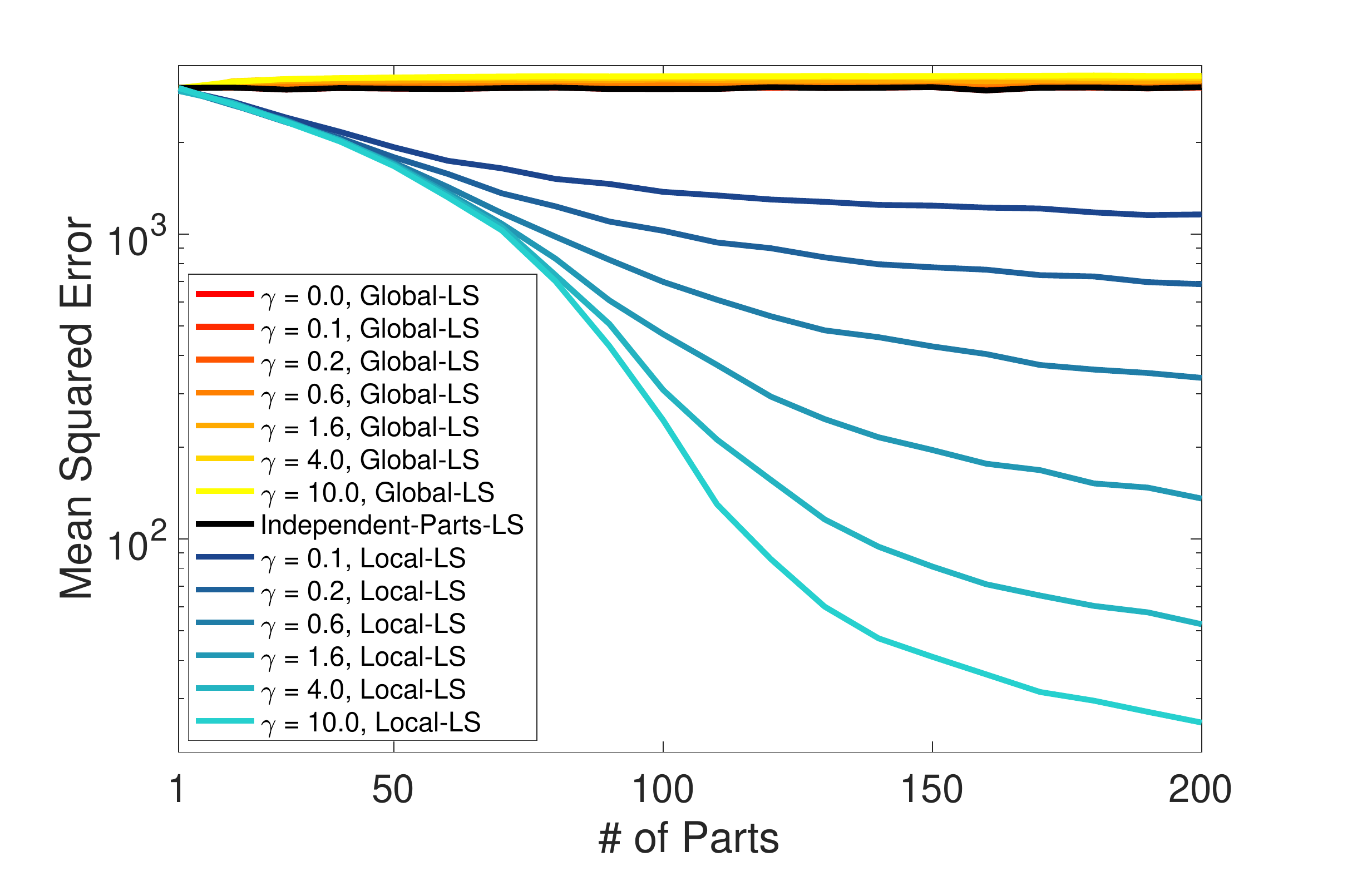}
        	\vspace*{-.2cm}
        \caption{Effect of within-locality w.r.t.~$\gamma$ and $|\P|$: {\em Global-LS} vs.~{\em IndependentParts-LS} vs.~{\em Local-LS} (ours). \label{fig:simulation-varying-parts} }
    \end{minipage}
    \vspace*{-.2cm}
\end{figure}

\paragraph{Within-locality} In \cref{fig:within-fingerprints} we visualize the (empirical) within-locality of the central patch $p$ for the fingerprint dataset. The figure depicts $\scov_{p,q}$ (defined in \cref{eq:within-locality-measures}) for $q \in P$, with the $(i,j)$-th pixel in the image corresponding to $\scov_{p,q}$ with $q$ the $20\times20$ patch centered in $(i,j)$. The fast decay of these values as the distance from the central patch $p$ increase, suggests that within-locality holds for a large value of $\gamma$, possibly justifying the good performance exhibited by ({\em Local-$\loss$}) in light of \cref{thm:rates-improved-with-parts}.

\paragraph{Simulation: Within-Locality}
We complement our analysis with synthetic experiments where we control the ``amount'' of within-locality $\gamma$. We considered a setting where input points are vectors $x\in\R^{k|\P|}$ comprising $|\P|$ parts of dimension $k=1000$. Inputs are sampled according to a normal distribution with zero mean and covariance $\Sigma(\gamma) = M(\gamma) \otimes I$, where $M(\gamma)\in\R^{|\P|\times|\P|}$ has entries $M(\gamma)_{pq} = \textnormal{e}^{-\gamma d(p, q)}$ and $d(p, q) = |p-q|/|P|$. By design, as $\gamma$ grows $\scov$ varies from being rank-one (all parts are identical copies) to diagonal (all parts are  independently sampled). 


To isolate the effect of within-locality on learning, we tested our estimator on a linear multitask (actually vector-valued) regression problem with least-squares loss $\loss$. We generated datasets $(x_i,y_i)_{i=1}^n$ of size $n=100$ for training and $n=1000$ for testing, with $x_i$ sampled as described above and $y_i = w^\top x_i + \epsilon$ with noise $\epsilon\in\R^{k|\P|}$ sampled from an isotropic Gaussian with standard deviation $0.5$. To guarantee {\em between}-locality to hold, we generated the target vector $w = [\bar w,\dots,\bar w]\in\R^{k|\P|}$ by concatenating copies of a $\bar w\in\R^k$ sampled uniformly on the radius-one ball. We performed regression with linear restriction kernel on the parts/subvectors ({\em Local-LS}) on the ``full'' auxiliary dataset $([x_i]_p, [y_i]_p)$ with $1\leq i\leq n$ and $1\leq p \leq|\P|$, and compared it with standard linear regression ({\em Global-LS}) on the original dataset $(x_i, y_i)_{i=1}^n$ and linear regression performed independently for each (local) subdataset $([x_i]_p, [y_i]_p)_{i=1}^n$ ({\em IndependentParts - LS}). The parameter $\lambda$ was chosen by hold-out cross-validation in $[10^{-6},10]$ (log spaced).

\cref{fig:simulation-varying-parts} reports the (log scale) mean square error (MSE) across $100$ runs of the two estimators for increasing values of $\gamma$ and $|\P|$. In line with \cref{thm:rates-improved-with-parts}, when $\gamma$ and $|\P|$ are large, {Local-LS} significantly outperforms both $i)$ Global-LS, which solves one single problem jointly and does not benefit within-locality, and $ii)$ IndependentParts-LS, which is insensitive to the between-locality across parts and solves each local prediction problem in isolation. For a smaller $\gamma$, such advantage becomes less prominent even when the number of parts is large. This is expected since for $\gamma=0$ the input parts are extremely correlated and there is no within locality that can be exploited.

\section{Conclusion}

We proposed a novel approach for structured prediction in presence of locality in the data. Our method builds on \cite{ciliberto2016} by incorporating knowledge of the parts directly within the learning model. We proved the benefits of locality by showing that, under a low-correlation assumption on the parts of the input (within locality), the learning rates of our estimator can improve proportionally to the number of parts in the data. To obtain this result we additionally introduced a natural assumption on the conditional independence between input-output parts (between locality), which provides also a formal justification for adoption of the so-called ``restriction kernel'', previously proposed in the literature, as a mean to lower the sample complexity of the problem. Empirical evaluation on synthetic as well as real data shows that our approach offers significant advantages when few training points are available and leveraging structural information such as locality is crucial to achieve good prediction performance. We identify two main directions for future work: $1)$ consider settings where the parts are unknown (or ``latent'') and need to be discovered/learned from data; $2)$ Consider more general locality assumptions. In particular, we argue that \cref{asm:within-locality} (WL) might be weakened to account for different (but related) local input-output relations across adjacent parts.


\section*{Acknowledgments}
We acknowledge support from the European Research Council (grant SEQUOIA 724063).






\bibliographystyle{plain}
\bibliography{biblio}
%

%

\newpage 








\crefalias{section}{appendix}
\crefalias{subsection}{subappendix}
\crefalias{subsubsection}{subsubappendix}


\appendix

\section*{\huge Appendix}

In this appendix we provide further background to the main discussion and results in the main sections of the current work. In particular: 
\begin{itemize}
    \item \cref{sec:app-generalized-model} introduces a generalization of the proposed framework to account for a larger family of structured prediction problems where locality can be exploited.
    \item \cref{sec:app-notation} introduces the notation and auxiliary results that will be useful to prove the results discussed in this work. 
    \item \cref{sec:app-derivation} discusses the derivation of the structured prediction estimator proposed and studied in this work.
    \item \cref{sec:app-comparison} extends the Comparison inequality for the SELF estimator in \cite{ciliberto2016} to the case where the locality of the problem can be exploited.
    \item \cref{sec:app-analytical} provides an analytical decomposition of a bound for the excess risk of the proposed estimator that is then used to prove the learning rates of the proposed estimator without and with parts (respectively \cref{sec:app-learning-rates,sec:app-learning-rates-parts}) and also the universal consistency (\cref{sec:app-consistency}).
    \item \cref{sec:app-equivalence-self} compares the proposed framework with structured prediction (without parts) in \cite{ciliberto2016}. 
    \item \cref{sec:app-algorithm} provides more details on the problem of learning and evaluating the estimator proposed in this work. 
    \item \cref{sec:example-loss} discusses in more detail loss functions considered in the literature that can be decomposed into ``parts''.
\end{itemize}

\paragraph{An overview of the main result in \cref{thm:rates-improved-with-parts}}
For the sake of clarity, before delving in the discussion below, we discuss here how the main result of this work, namely \cref{thm:rates-improved-with-parts}, is situated within the appendix. While the formal proof is given in \cref{sec:proof-rates-locality}, here we highlight and reference the key results used to this purpose. The main analysis in this sense can be found in \cref{sec:app-analytical,sec:app-learning-rates,sec:app-learning-rates-parts}. In particular, the proof hinges on three main components: 

\begin{enumerate}
    \item  We begin by studing the conditional expectation $\gstar$ introduced in \cref{lm:char-fstar} in terms of an estimator $\ghat$. In \cref{sec:app-derivation,sec:app-comparison} we prove that this estimator is tightly connected to our structured prediction estimator in \cref{eq:estimator} according to the comparison inequality 
    \eqal{\label{eq:comparison-inequality}
        \E(\fhat~) - \E(\fstar) \leq \closs \|\ghat - \gstar\|_{\LXPH}.
    }
    proved in \cref{thm:comparison} in \cref{sec:app-comparison}.

    \item The inequality above suggest to focus on $\|\ghat - \gstar\|_{\LXPH}$. We do this by providing an analytic decomposition for this quantity in \cref{thm:analytic-decomposition} in \cref{sec:app-analytical}. 

    \item Finally, in \cref{sec:app-learning-rates} we consider how each term in such analytical decomposition can be controlled in expectation with respect to a training dataset randomly sampled from the underlying distribution $\rho$. 
\end{enumerate}

Putting together all these results we are able to characterize the excess risk bounds for our estimator $\fhat$ in the {\em general setting where locality does not necessary hold}, which is reported below and proved at the end of \cref{sec:bound-in-expect}. 

\begin{restatable}{theorem}{TRates}\label{thm:rates}
  Let $\fhat$ as in \cref{eq:estimator} with i.i.d. training set and auxiliary dataset sampled according to \cref{alg:self-learning}. Let $\loss$ be SELF, $\Z$ compact, $g^* \in \gg$ and $\la \geq (\mathsf{r}^2/m + \mathsf{q}/n)^{1/2}$. Then
  \eqal{
    \Expect~\left[\E(\fhat~)- \E(\fstar)\right] ~\leq~ 12 ~\closs ~\mathsf{g} ~\left(\frac{\mathsf{r}^2}{\la m} + \frac{\mathsf{q}}{\la n} + \la\right)^{1/2}.
  }
\end{restatable}
\noindent Here we have introduced the quantity 
\eqalsplit{
  \mathsf{q} = \Expect_{x,x'} \Expect_{p,q|x, r|x'} ~~\scov_{p,q}(x,x') \qquad 
  \scov_{p,q}(x,x') = \left[k((x,p),(x,q))^2 - k((x,p),(x',r))^2\right],
}
where $\Expect_{p,q|x} [\cdot]$ is a shorthand for $\sum_{p,q \in P} \pi(p|x) \pi(q|x) [\cdot]$ (analogously for $\Expect_{r|x}$). It can be seen that this quantity allows to capture and leverage the within-locality assumption. In particular, it will allow us to quantify explicitly the advantages of using our locality-aware estimator.

The result above explicitly shows how the quantities measuring the within locality do affect the constants in the learning rates of the proposed estimator. By combining \cref{thm:rates} with \cref{asm:between-locality,asm:within-locality} and leveragin the locality properties of the restriction kernel introduced in \cref{eq:restriction-kernel}, we are then able to prove \cref{thm:rates-improved-with-parts} as desired. As mentioned, the details of this proof are reported in \cref{sec:proof-rates-locality}.

\section{Generalization of the Model by Parts}\label{sec:app-generalized-model}

In this section we introduce a slight generalization of the model considered in this work and that will be used in the rest of the appendixes. In particular we consider the case where $P$ is not necessarily finite and, possibly, the observed parts of $y$ are not necessarily deterministic. 

\subsection{When the Parts don't correspond exactly}\label{sec:parts-dont-correspond}

In general, $y_p$ (the $p$-th part of $y$) could not be univocally determined given $p\in\P$. For instance, consider a speech recognition problem where the goal is to predict the sentence pronounced by a speaker from an audio signal. In this setting the input space $\X$ is the set of all audio signals and $\Y = \Z$ is the set of all strings that can be produced in the speaker's language. In principle, for any part $x_p$ of an input signal $x\in\X$ it is possible to identify the corresponding part $y_p$ of the target string. In practice, such a procedure would require significant preprocessing (e.g. using hidden markov models) and would however not be guaranteed to be error-free.

In general, given an input $x\in\X$ a label $y\in\Y$ and a part $p\in\P$, observations for the $p$-th part of $y$ can be distributed according to some probability $\mu(w|y,x,p)$ over the set $\pY$ of parts of $\Y$. A possible way to model this situation is to consider a characterization of $\L$ in terms of a further function $\ell:\Z \times \pY \times \X \times \P \to \R$ such that
\eqal{\label{eq:loss-parts-with-mu}
\loss(z,y|x) &= \int_P L(z,y|x,p) d\pi(p|x),~~\textrm{where}~~\\
L(z,y|x,p) &= \int_\pY \ell(z,\eta|x,p) ~ d\mu(\eta|y,x,p). 
}
In this sense, the distribution $\mu$ can be interpreted as characterizing how likely it is for the part $p$ of an input $x$ with associated label $y$ to correspond to $\eta\in\pY$. It is possible to recover the standard characterization by selecting $\mu$ to be the Dirac de
$$\mu(\eta|y,x,p) = \delta(\eta, y_p).$$

\begin{remark}[Connection with standard Structured Prediction]
Note that the loss above generalizes the standard structured prediction framework as in \cite{tsochantaridis2005,nowozin2011,ciliberto2016}. Indeed, it is always possible to formulate a structured prediction loss $\loss$ in the proposed setting, by taking $\ell = \loss$ and $P = \{0\}$, $\pY = \Y$, $\pp(0|x) = 1$ and $\mu(w|y,x,0) = \delta_{y}$. However, if there exists a non-trivial characterization of $\loss$ in terms of these objects, then the algorithm proposed in this work is able to exploit this additional structure to achieve improved generalization performance. 
\end{remark}
Here we give the extended defintion of the SELF assumption, given the definition of loss in \cref{eq:loss-parts-with-mu}.

\begin{definition}[SELF by Parts (Extended)]\label{def:self}
  A function $\loss:\Z\times\Y\times\X\to\R$ is a {\em Structure Encoding Loss Function (SELF) by Parts} if it admits a factorization in the form of \eqref{eq:loss-parts-with-mu} with functions $\ell:Z\times\pY\times \X\times P \to\R$, and there exists a separable Hilbert space $\hh$ and two bounded continuous maps $\psi:\pZ\times\pX\times\P\to\hh$, $\varphi:\pY\to\hh$ such that for any $z \in Z$, $\eta \in \pY$, $x \in \X$, $p \in \P$
  \eqal{\label{eq:self}
    \ell(z,\eta|x,p) = \scal{\psi(z,x,p)}{\varphi(\eta)}_\hh.
  }
\end{definition}

\begin{remark}[\cref{def:self} is more general than \cref{def:self-simple}]\label{rem:defs-SELF}
Given a loss $\loss$ satisfying \cref{def:self-simple} for some $\psi', \phi, \hh'$, then it satisfy \cref{def:self}, with $\psi(z,x,p) = \psi'(z_p,z_p,p)$, with $\phi = \phi'$ with $\hh = \hh'$.
\end{remark}

\section{Notation and Main Definitions}\label{sec:app-notation}\label{sec:notation}
Let $\LXP$ be the Lebesgue function space with norm 
$$
\|\beta\|_{\LXP}^2 = \int_{\X\times\P} \beta(x,p)^2 ~ d\pp(p|x)d\rhox(x)
$$
with $\beta:\X\times\P\to\R$.
Analogously, $\LXPH$ be the Lebesgue function space with norm 
$$
\|\beta\|_{\LXPH}^2 = \int_{\X\times\P} \|\beta(x,p)\|_\hh^2 ~ d\pp(p|x)d\rhox(x)
$$
with $\beta:\X\times\P\to\hh$. Let $\big((x_i,y_i)\big)_{i=1}^n$ be the training set and let $\big((x_{i_j},y_{i_j},p_j,w_j)\big)_{j=1}^m$. Denote with $\rhoxhat$ the probability measure $\frac{1}{n}\sum_{i=1}^n \delta_{x_i}$. We define $\LXPHtilde$ the Lebesgue function space with norm
$$
\|\beta\|_{\LXPHtilde}^2 = \frac{1}{n} \sum_{i=1}^n \int_\P \|\beta(x_i,p)\|_\hh^2 ~ d\pp(p|x_i).
$$
with $\beta:\X\times\P\to\hh$. 

Let $k:(\X\times\P)\times(\X\times\P)\to\R$ be a reproducing kernel with associated reproducing kernel Hilbert space (RKHS) $\ff$. For any $(x,p)\in\X\times\P$ we denote $k_{x,p} = k\big((x,p),\cdot\big)\in\ff$.

We introduce the following objects: 

\begin{itemize}
  \item $S:\ff\to \LXP$~~ the operator such that, for any $f\in\ff$, 
  $$(Sf)(\cdot,\cdot) = \scal{f}{k_{(\cdot,\cdot)}}_\F.$$
  \item $S^*:\LXP\to\ff$~~ the operator such that, for any $\beta\in\LXP$, 
  $${S^*\beta = \int_{\X\times\P} k_{x,p}\beta(x,p)~d\pp(p|x)d\rhox(x)}.$$
  \item $C:\ff\to\ff$ the operator $\displaystyle{C = \int_{\X\times\P}k_{x,p}\otimes k_{x,p}~d\pp(p|x)d\rhox(x)}$.
  \item $\Ctilde:\ff\to\ff$ the operator $\displaystyle{\Ctilde = \frac{1}{n}\sum_{i=1}^n \int_{\P} k_{x_i,p} \otimes k_{x_i,p}~d\pp(p|x_i)}$.
  \item $\Chat:\ff\to\ff$ the operator $\displaystyle{\Chat = \frac{1}{m}\sum_{j=1}^m k_{x_{i_j},p_j} \otimes k_{x_{i_j},p_j}}$.
  
  \item $L:\LXP\to\LXP$ the operator such that for any $\beta\in\LXP$, we have that $(L\beta)(\cdot) = \int_{\X\times\P} k\big((x,p),\cdot\big) \beta(x,p)~d\pp(p|x)d\rhox(x)$. 
  
  \item $B:\hh\to\ff$ the operator $B = \int_{\P\times\X}k_{x,p}\otimes\varphi(w) ~ d\mu(w|y,x,p)d\pp(p|x)d\rho(y,x)$. Note that by definiton $B = \int k_{x,p}\otimes \gstar(x,p) ~d\pp(p|x)d\rho_\X(x)$ with $\gstar$ defined as in \cref{eq:char-fstar}.

  \item $\Bhat:\hh\to\ff$ the operator $\Bhat = \frac{1}{m}\sum_{j=1}^m k_{x_{i_j},p_j}\otimes\varphi(w_j)$.

  \item $G:\hh\to\LXP$ the operator such that, for any $h\in\hh$ is such that $(Gh)(\cdot) = \scal{\gstar(\cdot)}{h}_\hh$ for any $h\in\hh$, with $\gstar$ defined as in \cref{eq:char-fstar}. 
  
\end{itemize}

\paragraph{Further Notation} Let $\hh$ and $\ff$ be two Hilbert spaces and let $h\in\hh$ and $f\in\ff$, we denote with $h\otimes f$ the bounded linear operator from $\ff\to\hh$ such that, for any $g\in\ff$, we have $(h \otimes f) g = h \scal{f}{g}_\ff$. Note that $h \otimes f\in\hh\otimes\ff$, where $\hh \otimes \ff$ is the tensor product between the Hilbert spaces $\hh, \ff$ and is isometric to the  the space of Hilbert-Schmidt operators from $\ff$ to $\hh$, denoted by $\HS(\ff,\hh)$, namely the bounded linear operators $G:\ff\to\hh$ with finite Hilbert-Schmidt norm $\|G\|_\HS = \sqrt{\tr(G^*G)}$.

\subsection{Auxiliary Results}

\begin{lemma}\label{lem:characterization-BLC}
  With the notation introduced above, the following equations hold.
  \bi
  \item $L = SS^*$.
  \item $C = S^*S$.
  \item $SC_\la^{-1} S^* = LL_\la^{-1} = I - \la L_\la^{-1}$.
  \item $C_\la^{-1}S^* = S^*L_\la^{-1}$.
  \item $\|C_\la^{-1/2}S^*\| = \|S^* L_{\la}^{-1/2}\|\leq 1$ for any $\la>0$
  \ei
\end{lemma}
The proof of the result above are well known and we refer to Appendix B in \cite{ciliberto2016} for a proof with same notation as the one adopted in this paper. Below we show two further results that we will need 

\begin{lemma}
  with the notation introduced above we have
  \eqals{
    B = S^*G.
  }
\end{lemma}

\begin{proof} By applying the definition of the two operators $S$ and $G$ we have that for any $h\in\hh$, 
  \eqals{ 
    S^*G h & = S^*\big((Gh) (\cdot)\big)\\
    & = S^*(\scal{g^*(\cdot)}{h}_\hh) \\
    & = \int k_{x,p}\scal{\gstar(x,p)}{h}_\hh ~d\pp(p|x)d\rhox(x)\\
    & = \int (k_{x,p} \otimes \gstar(x,p)) h ~d\pp(p|x)d\rhox(x) = Bh
  }
  Hence $B = S^*G$ as required. 
\end{proof}

\section{Derivation of the algorithm}\label{sec:app-derivation}\label{app:derivation}

In this section we show how the algorithm naturally derives from the definition of the problem and in particular we prove \cref{lm:char-fstar}. Our analysis starts from the observation that when the loss function is SELF the solution of the learning problem in \cref{eq:base-fstar} is completely characterized in terms of the {\em conditional expectation} of $\varphi(y_p)$ given $x$, denoted by $\gstar:\X\times \P \to \hh$, with
\eqals{\label{eq:char-fstar-general}
  \gstar(x,p) = \int \varphi(\eta) d\mu(\eta|x,y,p) d\rho(y|x).
}
Note that since $\varphi(\cdot)$ is bounded and continuous, we have that $\gstar \in \Ltwo(\X, \pi\rhox, \hh)$. Below we prove \cref{lm:char-fstar}

\LCharacterizationFstar*

\begin{proof}
By Berge maximum theorem\cite{aliprantis2006} (see also \cite{ciliberto2016}), since $Z$ is compact,  we have that the solution of the learning problem in \cref{eq:base-fstar} is characterized by
$$\fstar(x) = \argmin_{z \in Z} \int \loss(z,y|x) d \rho(y|x).$$
The result is obtained by expanding the definition of $\loss$ with respect to SELF (\cref{def:self}) and the linearity of the inner product and the integral
\eqals{
\int \loss(z,y|x) d \rho(y|x) & = \int \ell(z,\eta|x,p)d\mu(\eta|y,x,p)d\pi(p|x)d\rho(y|x) \\
& = \int  \scal{\psi(z,x,p)}{\varphi(\eta)}_\hh d\mu(\eta|y,x,p)d\pi(p|x)d\rho(y|x) \\
& = \int  \scal{\psi(z,x,p)}{\int \varphi(\eta) d\mu(\eta|y,x,p) d\rho(y|x)}_\hh d\pi(p|x) \\
& = \int  \scal{\psi(z,x,p)}{g^*(x,p)}_\hh d\pi(p|x),
}
as desired.
\end{proof}
Since $\gstar$ depends on the unknown distribution $\rho$, we substitute it in \cref{eq:char-fstar} with an approximation $\ghat$. In particular, since $\gstar$ is the conditional expectation induced by $\rho(y|x)$, a viable choice for $\ghat$ is the {\em empirical risk minimizer} of the squared loss, which is a well known estimator for the conditional expectation \cite{caponnetto2007}, namely
\eqal{\label{eq:ghat-problem}
  \widehat{g} = \argmin_{g \in \gg} \frac{1}{m} \sum_{j=1}^m \|\psi(\auxy_j) - g(\auxx_j, p_j)\|_\hh^2 + \la \|g\|^2_{\gg},
}
where $\gg$ is a normed space of functions from $\X \times \P$ to $\hh$. In this work we will consider $\gg = \hh \otimes \F $ where $\F$ is the space of functions associated to a kernel $k$ on $\X\times \P$. In this case $\ghat$ can be obtained in closed form in terms of the auxiliary dataset and, when plugged in \cref{eq:char-fstar}, the resulting estimator corresponds exactly to the one in \cref{eq:estimator}, as shown in next Lemma.
\begin{lemma}\label{lm:fhat-plugin-ghat}
  Let $\loss$ be SELF, $\Z$ a compact set and $k$ be a positive definite kernel on $\X \times \P$ and $\fhat$ defined as in \cref{eq:estimator} with weights as in \cref{eq:estimator-alpha} computed using kernel $k$.
  Then $\fhat$ is characterized by
  \eqal{\label{eq:char-fhat}
    \fhat(x) = \argmin_{z \in Z} \sum_{p \in P} \pi(p|x) \scal{\psi(z_p,x_p,p)}{\ghat(x,p)}_\hh,
  }
  with $\ghat$ the solution of \cref{eq:ghat-problem} computed using kernel $k$.
\end{lemma}

\begin{proof}
We recall (see \cite{caponnetto2007}) that the least-squares solution of \cref{eq:ghat-problem} can be obtained in close form solution as 
$$
\ghat(x,p) = \sum_{j=1}^m \alpha_j(x,p) \varphi(y_{p_j})
$$
for any $x\in\X$ and $p\in\P$, where the weights $\alpha$ are defined as in \cref{eq:estimator-alpha}. By linearity of the inner product we have 
\eqal{
  \sum_{p \in P} \pi(p|x) \scal{\psi(z_p,x_p,p)}{\ghat(x,p)}_\hh & = \sum_{j=1}^m \sum_{p \in P} \pi(p|x) \alpha_j(x,p)\scal{\psi(z_p,x_p,p)}{\varphi(y_{p_j})}_\hh \\ 
  & = \sum_{j=1}^m \sum_{p \in P} \pi(p|x) \alpha_j(x,p) L_p(z_p,y_p|x_p)
}
where the last step follows from the assumption that the loss is SELF. 
\end{proof}
An interesting consequence of the lemma above is that $\psi, \varphi, \ghat, \gstar, \hh$ are only needed for theoretical purposes -- i.e. to establish the connection between the estimator $\fhat$ and the ideal solution $\fstar$ -- and are not needed for the evaluation of $\fhat$ which is done in terms of known objects, via \cref{eq:estimator}.
\section{Comparison Inequality}\label{sec:app-comparison}
 In this we derive a result, \cref{thm:comparison}, that is crucial to prove the statistical properties of the proposed algorithm. Note that it is analogous to the Comparison Inequality of \cite{ciliberto2016} and of independent interest for the proposed framework.  First we define the following estimator, that is a more general version of the one presented in the paper 
  \eqal{\label{eq:char-fhat-general}
  \fhat(x) = \argmin_{z \in Z} \int_P \scal{\psi(z,x,p)}{\ghat(x,p)}_\hh \pi(p|x).
}
Note that the estimator presented in the main paper which is characterized by \eqref{eq:char-fhat}, \cref{lm:fhat-plugin-ghat} can be written like \eqref{eq:char-fhat-general}, applying Remark~\ref{rem:defs-SELF} in \cref{sec:parts-dont-correspond}.

\begin{theorem}\label{thm:comparison}
  When $Z$ is a compact set and $\loss$ satisfies \cref{def:self}, for any measurable $\ghat:\X\times\P\to\hh$ and $\fhat:\X\to\Z$ defined in terms of $\ghat$ as in \eqref{eq:char-fhat-general}. Then 
  \eqal{\label{eq:comparison-inequality}
    \E(\fhat~) - \E(\fstar) \leq \closs \|\ghat - \gstar\|_{\LXPH}
  }
  and $\closs$ is a constant depending only on $\loss$ and defined at the end of the proof.
\end{theorem}
\begin{proof}
For any $x\in\X$ and $z\in\Z$, let 
\eqal{
  A(z|x) & = \int_\P \scal{\psi(z,x,p)}{\gstar(x,p)}_\hh ~d\pp(p|x), \\ 
  \widehat A(z|x) & = \int_\P \scal{\psi(z,x,p)}{\ghat(x,p)}_\hh ~d\pp(p|x).
}
By the SELF assumption $\ell(z,w|x,p) = \scal{\psi(z,x,p)}{\varphi(w)}_\hh$ and the definition of $\gstar$ as in \eqref{eq:char-fstar} we have the following alternative characterization for $A(z|x)$ as shown in \cref{lm:char-fstar}
\eqal{
  A(z|x) = \int_{\pY\times\Y\times\P} \ell(z,w|x,p) ~d\mu(w|y,x,p)d\rho(y|x)d\pp(p|x).
}
Then, $\E(f) = \int_\X A(f(x)|x) ~d\rho_\X(x)$ for any $f:\X\to\Z$ and we have the following decomposition of the excess risk
\eqal{
  \E(\fhat) - \E(\fstar) & = \int_\X A(\fhat(x)|x) - A(\fstar(x)|x) ~ d\rho_\X(x) \\
  & = \int_\X A(\fhat(x)|x) - \widehat A(\fhat(x)|x) + \underbrace{\widehat A(\fhat(x)|x) - \widehat A(\fstar(x)|x)}_{\leq 0} \\
  & \quad + \int_\X \widehat A(\fstar(x)|x) - A(\fstar(x)|x) ~ d\rho_\X(x)\\
  & \leq 2 \int_\X \sup_{z\in\Z} \Big|\widehat A(z|x) - A(z|x)\Big|~d\rho_\X(x)\label{eq:integral-sup-A-Ahat}
}
where we have used the fact that $\widehat A(\fhat(x)|x) - \widehat A(\fstar(x)|x) \leq 0$ since, by definition, $\fhat(x)$ is the minimizer of $\widehat A(\cdot|x)$ (see Eq.~\eqref{eq:char-fhat-general}).

Now, note that by the linearity of the inner product we have
\eqal{
  \Big|\widehat A(z|x) - A(z|x)\Big| & = \left|\int_\P \scal{\psi(z,x,p)}{\ghat(x,p)-\gstar(x,p)}_\hh~d\pp(p|x)\right| \\
  & \leq \int_\P \|\psi(z,x,p)\|_\hh~ \|\gstar(x,p)-\ghat(x,p)\|_\hh~d\pp(p|x) \\
  & \leq \sqrt{\int_\P \|\psi(z,x,p)\|_\hh^2 ~d\pp(p|x)}\sqrt{ \int_\P \|\gstar(x,p)-\ghat(x,p)\|_\hh^2~d\pp(p|x)}\\
  & = q(x,z) ~ \sqrt{ \int_\P \|\gstar(x,p)-\ghat(x,p)\|_\hh^2~d\pp(p|x)}
}
where we applied Cauchy-Schwartz for each of the two inequalities, with $q(x,z) =\sqrt{\int_\P \|\psi(z,x,p)\|_\hh^2 ~d\pp(p|x)}$.

Denote with $\|\cdot\|_{\LXPH}$ the norm such that 
\eqals{
\|g\|_{\LXPH}^2 = \int_{\X\times\P}\|g(x,p)\|_\hh^2 ~d\pp(p|x)d\rho_\X(x),
} for any $g:\X\times\P\to\hh$. Then, plugging the inequality above in \eqref{eq:integral-sup-A-Ahat}, we obtain
\eqal{
  2 \int_\X & \sup_{z\in\Z} \Big|\widehat A(z|x) - A(z|x)\Big|~d\rho_\X(x) \\
  & \leq 2 \int_\X \sup_{z\in\Z}~\left[q(x,z) ~ \sqrt{ \int_\P \|\gstar(x,p)-\ghat(x,p)\|_\hh^2~d\pp(p|x)}~~\right]~d\rho_\X(x) \\
  & = 2 \int_\X \sup_{z\in\Z}~\Big[q(x,z)\Big]~ \sqrt{ \int_\P \|\gstar(x,p)-\ghat(x,p)\|_\hh^2~d\pp(p|x)}~d\rho_\X(x) \\
  & \leq 2 \sqrt{\int_\X \left(\sup_{z\in\Z}~q(x,z)\right)^2~d\rhox(x)}\sqrt{ \int_{\X\times\P} \|\gstar(x,p)-\ghat(x,p)\|_\hh^2~d\pp(p|x)d\rho_\X(x)} \\
  & = \closs \|\ghat - \gstar\|_{\LXPH}
}
where the last inequality follows from Cauchy-Schwartz and
\eqal{
  \closs & = 2 ~\sqrt{\int_\X \left(\sup_{z\in\Z}q(x,z)\right)^2~d\rhox(x)} \\ 
  & = 2 ~\sqrt{\int_\X~ \sup_{z\in\Z}\left[ \int_\P \|\psi(z,x,p)\|_\hh^2 ~d\pp(p|x) \right] ~d\rhox(x) }
}
\end{proof}

\begin{remark}[Remove the dependency of $\closs$ from $\rhox$]\label{rem:dependency-closs-rhox}
Note that it is always possible to remove the dependency of $\closs$ from $\rhox$ by bounding it with 
\eqal{
  \closs \leq 2~\Bigg(\sup_{\substack{z\in\Z \\ x\in\X} }\int_\P \|\psi(z,x,p)\|_\hh^2 ~d\pp(p|x) \Bigg)^{1/2}
}
\end{remark}

\section{Analytical Decomposition}\label{sec:app-analytical}
According to the comparison inequality \cref{eq:comparison-inequality} it is sufficient to bound the quantity $\|\ghat - \gstar\|_\LXPH$ in order to control the excess risk of the estimator $\fhat$. Equipped with the notation introduced above, we can now focus on studying this quantity. In particular in \cref{thm:analytic-decomposition} we provide an analytical decomposition of $\|\ghat - \gstar\|_\LXPH$ in terms of basic quantities that can be controlled in expectation (or probability, for the universal consistency).

\begin{proposition}\label{prop:l2norm}
Let $\ghat, \gstar$ be defined as in \cref{eq:ghat-problem} and \cref{eq:char-fstar-general}, then the following holds
\eqal{
\|\ghat - \gstar\|_\LXPH = \|S\Chat_\la^{-1} \Bhat - G\|_{\HS(\hh,\LXP)}
}

\end{proposition}
\begin{proof}
First of all we recall that the space $\LXPH$ is isometric to $\hh \otimes \LXP$ which is isometric to the space of linear Hilbert-Schmidt operators from $\hh \to \LXP$, denoted by $\HS(\hh,\LXP)$. Now note that $G$ is the operator in $\HS(\hh,\LXP)$, that is isometric to $\gstar \in \LXPH$, indeed $G v = \scal{g^*(\cdot, \cdot)}{v}_\hh$, for any $v \in \hh$.

Now note that is the solution of the problem in \cref{eq:ghat-problem}. Indeed, first note that the functional $\widehat{R}_\la(W)$, defining the problem in \cref{eq:ghat-problem}, is smooth and strongly convex ($W \in \hh \otimes \F$, $\la > 0$). Then we find the solution by equating the derivative of $\widehat{R}_\la(W)$ to $0$. First note that for any $W \in \hh \otimes \F$, the functional $\widehat{R}_\la(W)$, is equivalent to
\eqals{
\widehat{R}_\la(W) & = \frac{1}{m}\sum_{j=1}^m \|\phi(w_j) - Wk_{(x_{i_j}, p_j)}\|^2_\hh + \la \|W\|_{\hh \otimes \F} \\
&=\tr\Big[W\left(\frac{1}{m} \sum_{j =1}^m  k_{(x_{i_j}, p_j)} \otimes k_{(x_{i_j}, p_j)} ~+~ \la I~\right) W^* \\
& \qquad - 2 \left(\frac{1}{m}\sum_{j=1}^m  k_{(x_{i_j}, p_j)} \otimes \phi(w_j) \right)W^* + \frac{1}{m} \sum_{j=1}^m \phi(w_j) \otimes \phi(w_j) \Big] \\
& = \tr\Big[W\left(\Chat + \la I\right) W^*  - 2\Bhat W + \frac{1}{m} \sum_{j=1}^m \phi(w_j) \otimes \phi(w_j) \Big],
}
where for the last step we applied the defintion of $\Chat$ and $\Bhat$. 
By taking the derivative of $\widehat{R}_\la(W)$ in $W$ and equating it to $0$ the following minimizer is obtained $\widehat{W} = \Bhat^*\Chat_\la^{-1}$.

Moreover note that, $S\Chat_\la^{-1} \Bhat$ is the operator in $\HS(\hh,\LXP)$, that is isometric to $\ghat \in \LXPH$, indeed by definition of $S$
$$S\Chat_\la^{-1} \Bhat v = \scal{k_{(\cdot, \cdot)}}{ \widehat{W}^*v}_\F = \scal{\widehat{W}k_{(\cdot, \cdot)}}{v}_\hh = \scal{\widehat{g}(\cdot, \cdot)}{v}_\hh, \quad \forall v \in \hh.$$
\end{proof}

\begin{theorem}\label{thm:analytic-decomposition}
Let $\la>0$. With the definitions in \cref{sec:notation}, we have
\eqal{
\|\ghat - \gstar\|_\LXPH \leq \left(\frac{1}{\sqrt{\la}} + \frac{\beta_1^{1/2}}{\la}\right)\Big(\beta_1 \A_{1/2}(\la) + \beta_2\Big) + \la \A_1(\la).
}
where $\beta_1 = \|C - \Chat\|$, $\beta_2 = \|\Bhat - B\|_\HS$ and $\A_{r}(\la) = \|L_\la^{-r}G\|_\HS$ ~for $r>0$.
\end{theorem}

\begin{proof}
By \cref{prop:l2norm} and by adding and subtracting $S\Chat_\la^{-1}B$ and $SC_\la^{-1}B$ we have
\eqal{
  \|\ghat - \gstar\|_\LXPH = \|S\Chat_\la^{-1} \Bhat - G\|_{\HS(\hh,\Ltwo)} \leq A_1 + A_2 + A_3
}
with
\eqal{
  A_1 & = \|S\Chat_\la^{-1} \Bhat - S\Chat_\la^{-1}B\|_{\HS(\hh,\Ltwo)}\\
  A_2 & =  \|S\Chat_\la^{-1}B - S C_\la^{-1}B\|_{\HS(\hh,\Ltwo)}\\
  A_3 & =  \|SC_\la^{-1}B - G\|_{\HS(\hh,\Ltwo)}.
}

\paragraph{Bounding $A_1$} Now, by dividing and multiplying by $C_\la^{1/2}$, we have
\eqal{
  A_1 = \|S\Chatla^{-1}(\Bhat - B)\|_{\HS(\hh,\Ltwo)} \leq \|S\Chatla^{-1}\|\|\Bhat - B\|_{\HS(\hh,\F)}
}

\paragraph{Bounding $A_2$} By using the identity $R^{-1} - T^{-1} = R^{-1}(T - R)T^{-1}$ holding for any invertible operators $R,T:\ff\to\ff$, we have

\eqals{
  A_2 & = \|S(\Chat_\la^{-1} - C_\la^{-1})B\|_{\HS(\hh,\Ltwo)} \\
  &  = \|S\Chat_\la^{-1}(C_\la - \Chat_\la)C_\la^{-1}B\|_{\HS(\hh,\Ltwo)}\\
  & = \|S\Chat_\la^{-1}(C - \Chat)C_\la^{-1}B\|_{\HS(\hh,\Ltwo)}\\
  & \leq \|S\Chatla^{-1}\|\|C - \Chat\|\|\Cla^{-1}B\|_{\HS(\hh,\F)}.\\
}
We further apply \cref{lem:characterization-BLC} to have $\|C_\la^{-1/2} S^*\|  = \|S^*L_\la^{-1/2}\| \leq 1$ and $C_\la^{-1}S = S^*L_\la^{-1}$. Then, 
\eqal{
  \|C_\la^{-1}B\|_{\HS(\hh,\F)} & = \|C_\la^{-1}S^*G\|_{\HS(\hh,\F)} = \|S^*L_{\la}^{-1} G\|_{\HS(\hh,\F)}\\
  & \leq \|S^*L_\la^{-1/2}\|\|L_\la^{-1/2} G\|_{\HS(\hh,\Ltwo)} \leq \|L_\la^{-1/2} G\|_{\HS(\hh,\Ltwo)}.
} 

\paragraph{Bounding $A_3$} From \cref{lem:characterization-BLC} we have $B = S^*G$ and $SC_\la^{-1} S^* = LL_\la^{-1} = I - \la L_\la^{-1}$. Then, 
\eqal{
  A_3 = \|SC_\la^{-1}S^*G - G\|_{\HS(\hh,\Ltwo)} = \|(I - \la L_\la^{-1})G - G\|_{\HS(\hh,\Ltwo)} = \la\|L_\la^{-1}G\|_{\HS(\hh,\Ltwo)}.
}

To conclude, we control the term $\|S\Chatla^{-1}\|$ by
\eqal{
  \|S\Chatla^{-1}\|^2 = \|\Chatla^{-1}C\Chatla^{-1}\| & \leq \|\Chatla^{-1}(C - \Chat)\Chatla^{-1}\| +  \|\Chatla^{-1}\Chat\Chatla^{-1}\| \\
  & \leq \|\Chatla^{-1}\|^2 \|C - \Chat\| + \frac{1}{\la}\\
  & \leq \frac{1}{\la^2} \|C - \Chat\| + \frac{1}{\la}
}
Therefore
\eqal{
  \|S\Chatla^{-1}\| \leq \sqrt{\frac{\|C - \Chat\|}{\la^2} + \frac{1}{\la}} \leq \frac{1}{\sqrt{\la}} + \frac{\sqrt{\|C - \Chat\|}}{\la}
}

Combining the bounds for $A_1,A_2$ and $A_3$ we obtain the desired result. 
\end{proof}

\section{Learning Rates}\label{sec:app-learning-rates}

Building on the analytic decomposition of \cref{thm:analytic-decomposition} we observe that the key quantities to study in this setting are the $\Expect \|\Chat - C\|^2$ and $\Expect \|\Bhat - B\|_\HS^2$ as discussed below. 
In particular the following theorem further decomposes the quantities from \cref{thm:analytic-decomposition}, and $\Expect \|\Chat - C\|^2$ and $\Expect \|\Bhat - B\|_\HS^2$, are bounded in \cref{sec:controlling-beta1,sec:controlling-beta2}. Finally \cref{thm:rates-general} is given in \cref{sec:bound-in-expect}.

\begin{theorem}\label{thm:bound-expectation-depends-betas}
Let $\la>0$. With the definitions in \cref{sec:notation} and \cref{thm:analytic-decomposition}, we have
\eqal{
  \Expect \|\ghat - \gstar\|_\LXPH \leq 2 \left(1 + \frac{\sqrt{\Expect\beta_1^2}}{\la}\right)^{1/2}\left(\frac{\A_{1/2}(\la)^2\Expect\beta_1^2 }{\la}  + \frac{\Expect\beta_2^2}{\la}\right)^{1/2} + \la\A_1(\la).
}
\end{theorem}

\begin{proof}

Let $a = \frac{1}{\sqrt{\la}}$, $b = \frac{1}{\la}$, $c = \|L_\la^{-1/2}G\|_\HS$ and $d = \la\|L_\la^{-1}G\|_\HS$. Then,
\eqal{
  \Expect \|\ghat - \gstar\|_\LXPH & \leq \Expect (a + b\beta_1^{1/2})(c\beta_1 + \beta_2) + d \\
  & \leq \sqrt{\Expect (a + b\beta_1^{1/2})^2 \Expect(c\beta_1 + \beta_2)^2} + d\\
  & \leq \sqrt{4 (a^2 + b^2\Expect\beta_1)(c^2\Expect\beta_1^2 + \Expect\beta_2^2)} + d \\
  & \leq 2 \sqrt{(a^2 + b^2\sqrt{\Expect\beta_1^2})(c^2\Expect\beta_1^2 + \Expect\beta_2^2)} + d,
}
as desired.
\end{proof}
The rest of this section will be devoted to characterizing the behavior of $\Expect \beta_1^2$ and $\Expect \beta_2^2$ in order to obtain a more interpretable learning rates for the estimator proposed in this work.

\subsection[alt-text]{Bounding $\Expect \beta_1^2$}\label{sec:controlling-beta1}

Denote $\zeta_{x_{i_j},p_j} = k_{x_{i_j},p_j}\otimes k_{x_{i_j},p_j} - C$. First, we show that $\Expect \zeta_{x_{i_j},p_j} = 0$. 

\begin{lemma}\label{lem:zeta-zero-mean}
With the definition above, when $x_1,\dots,x_n$ are identically distributed, we have 
$$
\Expect ~ \zeta_{x_{i_j},p_j} = 0
$$
\end{lemma}
\begin{proof}
Since $x_1,\dots,x_n$ are identically distributed, for any $j=1,\dots,m$, we have
\eqals{
  \Expect~ k_{x_{i_j},p_j} \otimes k_{x_{i_j},p_j} & = \frac{1}{n}\sum_{i_j=1}^n \int_{\P\times\X} k_{x_{i_j},p_j} \otimes k_{x_{i_j},p_j} ~d\pp(p_j|x_{i_j})d\rhox(x_{i_j}) \\
  & = \int_{\P\times\X} k_{x,p} \otimes k_{x,p} ~d\pp(p|x)d\rhox(x)\\
  & = C,
}
as desired
\end{proof}

\begin{lemma}\label{lem:expectation-bound-c-chat}
With the definitions of Section~\ref{sec:notation}
let  $Q_1 = \Expect\|\zeta_{x,p}\|_\HS^2$ and
\eqal{
  \mathfrak{C} = \int_{\P\times\X} \zeta_{x,p} \zeta_{x,p'} ~d\pp(p|x)d\pp(p'|x)d\rhox(x)
} 
\eqals{
  \Expect\|\Chat - C\|_\HS^2 = \frac{Q_1}{m} + \frac{ (m-1)}{m} \frac{\tr(\mathfrak{C})}{n}.
}
\end{lemma}

\begin{proof}


From the definition of $\Chat$, we have

\eqal{\label{eq:expectation-bound-c-chat}
  \Expect\|\Chat - C\|_\HS^2 = \Expect\|\frac{1}{m} \sum_{j=1}^m \zeta_{x_{i_j},p_j}\|_\HS^2 = \frac{1}{m^2} \sum_{j,h=1}^m \Expect ~\tr \left( ~ \zeta_{x_{i_j},p_j} \zeta_{x_{i_h},p_h}\right)
}
We consider separately the elements in the sum that correspond to the case $j=h$ and $j\neq h$. 

\paragraph{Case $j = h$} We have
\eqals{
  \Expect~\tr\left(\zeta_{x_{i_j},p_j}\zeta_{x_{i_h},p_h}\right) = \Expect\|\zeta_{x_{i_j},p_j}\|_\HS^2 = Q_1
}

\paragraph{2. Case $j\neq h$} We have $\Expect~\tr\left( \zeta_{x_{i_j},p_j} \zeta_{x_{i_h},p_h}\right) = \frac{1}{n^2}\sum_{i_j,i_h = 1}^n R_{i_j,i_h}^{j,h}$ where
\eqals{
  R_{u,v}^{j,h} = \int_{\P\times\X} \tr(\zeta_{x_{u},p_j}\zeta_{x_{v},p_h}) ~d\pp(p_j|x_{u})d\pp(p_h|x_{v})d\rhox(x_{1})\cdots d\rhox(x_{n}).
}
We consider separately the case $i_j=i_h$ and $i_j\neq i_h$.

\paragraph{Case $j\neq h$ and $i_j=i_h$} We have that 
\eqals{
  R_{i_j,i_j}^{j,h} & = \int_{\P\times\X} \tr\left( \zeta_{x_{i_j},p_j} \zeta_{x_{i_j},p_h} \right) ~d\pp(p_j|x_{i_j})d\pp(p_h|x_{i_j})d\rhox(x_{i_j}) \\
  & = \int_{\P\times\X} \tr\left( \zeta_{x,p} \zeta_{x,p'} \right) ~d\pp(p|x)d\pp(p'|x)d\rhox(x) = \tr(\mathfrak{C}).
}

\paragraph{2.2 Case $j\neq h$ and $i_j\neq i_h$} We have that 
\eqals{
  R_{i_j,i_h}^{j,h} & = \int \tr\left( \zeta_{x_{i_j},p_j} \zeta_{x_{i_h},p_h} \right) ~d\pp(p_j|x_{i_j})d\pp(p_h|x_{i_h})d\rhox(x_{i_j})d\rhox(x_{i_h}) \\
  & = \int \tr\left( \zeta_{x,p} \zeta_{x',p'} \right) ~d\pp(p|x)d\pp(p'|x')d\rhox(x)d\rhox(x') \\
  & = \tr\left( \int  \zeta_{x,p} ~d\pp(p|x)d\rhox(x)~ \int \zeta_{x',p'} ~d\pp(p'|x')d\rhox(x') \right) \\
  & = \|\Expect~ \zeta_{x,p}\|_\HS^2 = 0
}
where the last equality follows from the fact that the $\zeta_{x,p}$ have zero mean according to \cref{lem:zeta-zero-mean}.

\paragraph{Combining the above cases} Note that in \eqref{eq:expectation-bound-c-chat}, Case $1$ occurs $m$ times and Case $2$ occurs the remaining $m(m-1)$ times. Therefore, we have
\eqals{
  \Expect\|\Chat - C\|_\HS^2 = \frac{Q_1}{m} + \frac{m-1}{m} \frac{1}{n^2}\sum_{i_j,i_h=1}^n R_{i_j,i_h}^{j,h}
}
Now, for the second term on the right hand side, Case $2.1$ occurs $n$ times while Case $2.2$ occurs the remaining $n(n-1)$ times, leading to the desired result.
\end{proof}

\begin{lemma}\label{lem:characterization-Cfrak}
With the notation of \cref{lem:expectation-bound-c-chat} and the definition of $\mathsf{q}$ in \eqref{eq:def-constants-q}, we have
\eqals{
\tr(\mathfrak{C}) = \mathfrak{c}_1 - \mathfrak{c}_2 = \mathsf{q},
}
where
\eqals{
  \mathfrak{c}_1 & = \int k\big((x,p),(x,p')\big)^2 ~ d\pp(p|x)d\pp(p'|x)d\rhox(x) \\
  \mathfrak{c}_2 & = \int k\big((x,p),(x',p')\big)^2 ~d\pp(p|x)d\pp(p'|x')d\rhox(x)d\rhox(x').
}
\end{lemma}

\begin{proof}
Note that by definition of $\zeta$ and the reproducing property of the kernel $k$, for any $x,x'\in\X$ and $p,p'\in\P$ the following holds
\eqals{
\tr(\zeta_{x,p} \zeta_{x',p'}) & = k\big((x,p),(x',p')\big)^2 - \tr\left(C ~ \Big(k_{x,p}\otimes k_{x,p}\Big) \right) \\ 
& \quad - \tr\left(C ~ \Big(k_{x',p'}\otimes k_{x',p'}\Big) \right) + \tr(C^2).
}
Then, by definition of $C = \Expect ~ k_{x,p}\otimes k_{x,p}$, we have
\eqals{
  \tr(\mathfrak{C}) & = \int \tr\left( \zeta_{x,p} \zeta_{x,p'} \right) ~d\pp(p|x)d\pp(p'|x)d\rhox(x) \\
  & = - \tr(C^2) +  \int k\big((x,p),(x,p')\big)^2 ~d\pp(p|x)d\pp(p'|x)d\rhox(x) \\
  & = - \tr(C^2) + \int k\big((x,p),(x,p')\big)^2 ~ d\pp(p|x)d\pp(p'|x)d\rhox(x) \\
  & = \mathfrak{c}_1 - \tr(C^2).
}
To conclude, 
\eqals{
  \tr(C^2) & = \tr\left(\left(\int k_{x,p}\otimes k_{x,p} ~d\pp(p|x)d\rhox(x)\right)\left(\int k_{x',p'}\otimes k_{x',p'} ~d\pp(p'|x')d\rhox(x')\right)\right)\\
  & =  \int k\big((x,p),(x',p')\big)^2 ~d\pp(p|x)d\pp(p'|x')d\rhox(x)d\rhox(x')\\
  & = \mathfrak{c}_2.
}
The last step consists in noting that $\mathfrak{c}_1 - \mathfrak{c}_2$ is exactly the definition of $\msf{q}$ in \eqref{eq:def-constants-q}.
\end{proof}

\subsection[alt-text]{Bounding $\Expect \beta_2^2$}\label{sec:controlling-beta2}

The analysis for $\Expect \beta_2^2$ is analogous to that of $\Expect \beta_1^2$. For completeness we report it below. Denote $\eta_{x_{i_j},p_j,w_j} = k_{x_{i_j},p_j}\otimes \varphi(w_j) - B$. We show that $\Expect ~\eta_{x_{i_j},p_j,w_j} = 0$. 

\begin{lemma}\label{lem:eta-zero-mean}
With the definition above, when $x_1,\dots,x_n$ are identically distributed, we have 
$$
\Expect ~ \eta_{x_{i_j},p_j,w_j} = 0
$$
\end{lemma}
\begin{proof}
Since $x_1,\dots,x_n$ are identically distributed, for any $j=1,\dots,m$, we have
\eqals{
  \Expect~ k_{x_{i_j},p_j} \otimes \varphi(w_j) & = \frac{1}{n}\sum_{i_j=1}^n \int k_{x_{i_j},p_j} \otimes \varphi(w_j) ~d\mu(w_j|y_{i_j},x_{i_j},p_j)d\pp(p_j|x_{i_j})d\rho(y_{i_j},x_{i_j}) \\
  & = \int k_{x,p} \otimes \varphi(w) ~d\mu(w|y,x,p)d\pp(p|x)d\rho(y,x) \\
  & = B,
}
as desired.
\end{proof}

\begin{lemma}\label{lem:expectation-bound-b-bhat}
Let  $Q_2 = \Expect\|\eta_{x,p,w}\|_\HS^2$ and
\eqal{
  \mathfrak{B} = \int\eta_{x,p,w}^* \eta_{x,p',w'} ~d\mu(w|y,x,p)d\mu(w'|y,x,p')d\pp(p|x)d\pp(p'|x)d\rho(y,x) 
} 
\eqals{
  \Expect\|\Bhat - B\|_\HS^2 = \frac{Q_2}{m} + \frac{ (m-1)}{m} \frac{\tr(\mathfrak{B})}{n}.
}
\end{lemma}

\begin{proof}


From the definition of $\Bhat$, we have

\eqal{\label{eq:expectation-bound-b-bhat}
  \Expect\|\Bhat - B\|_\HS^2 = \Expect\|\frac{1}{m} \sum_{j=1}^m \eta_{x_{i_j},p_j,w_j}\|_\HS^2 = \frac{1}{m^2} \sum_{j,h=1}^m \Expect ~\tr \left( ~ \eta_{x_{i_j},p_j,w_j}^* \eta_{x_{i_h},p_h,w_h}\right)
}
We consider separately the elements in the sum that correspond to the case $j=h$ and $j\neq h$.

\paragraph{1. Case $j = h$} We have
\eqals{
  \Expect~\tr\left(\eta_{x_{i_j},p_j,w_j}^*\eta_{x_{i_h},p_h,w_h}\right) = \Expect\|\eta_{x_{i_j},p_j,w_j}\|_\HS^2 = Q_2.
}

\paragraph{2. Case $j\neq h$} We have $\Expect~\tr\left( \eta_{x_{i_j},p_j,w_j}^* \eta_{x_{i_h},p_h,w_h}\right) = \frac{1}{n^2}\sum_{i_j,i_h = 1}^n Z_{i_j,i_h}^{j,h}$ where
\eqals{
  Z_{u,v}^{j,h} = \int \tr(\eta_{x_{u},p_j,w_j}^* & \eta_{x_{v},p_h,w_h})  ~d\mu(w_j|y_{i_j},x_{i_j},p_j)d\mu(w_h|y_{i_h},x_{i_h},p_h) \times \\
  & \times d\pp(p_j|x_{u})d\pp(p_h|x_{v})d\rho(y_1,x_{1})\cdots d\rho(y_n,x_{n}).
}
We consider separately the case $i_j=i_h$ and $i_j\neq i_h$.\\

\paragraph{2.1 Case $j\neq h$ and $i_j=i_h$} We have that 
\eqals{
  Z_{i_j,i_j}^{j,h} & = \int \tr\left( \eta_{x_{i_j},p_j,w_j}^* \eta_{x_{i_j},p_h,w_h} \right) ~d\mu(w_j|y_{i_j},x_{i_j},p_j)d\mu(w_h|y_{i_j},x_{i_j},p_h) \times \\
  & \qquad\qquad\qquad\qquad\qquad\qquad\times d\pp(p_j|x_{i_j})d\pp(p_h|x_{i_j})d\rho(y_{i_j},x_{i_j}) \\
  & = \int \tr\left( \eta_{x,p,w}^* \eta_{x,p',w'} \right) ~d\mu(w|y,x,p)d\mu(w'|y,x,p')d\pp(p|x)d\pp(p'|x)d\rho(y,x) \\
  & = \tr(\mathfrak{B}).
}

\paragraph{2.2 Case $j\neq h$ and $i_j\neq i_h$} We have that 
\eqals{
  Z_{i_j,i_h}^{j,h} & = \int \tr\left( \eta_{x_{i_j},p_j,w_j}^* \eta_{x_{i_h},p_h,w_h} \right) ~d\mu(w_j|y_{i_j},x_{i_j},p_j)d\mu(w_h|y_{i_h},x_{i_h},p_h) \times \\
  & \qquad\qquad\qquad\qquad\qquad\qquad\times d\pp(p_j|x_{i_j})d\pp(p_h|x_{i_h})d\rho(y_{i_j},x_{i_j})d\rho(y_{i_h},x_{i_h}) \\
  & = \int \tr\left( \eta_{x,p,w}^* \eta_{x',p',w'} \right) ~d\mu(w|y,x,p)d\mu(w'|y',x',p') \times \\
  & \qquad\qquad\qquad\qquad\qquad\qquad\times d\pp(p|x)d\pp(p'|x')d\rho(y,x)d\rho(y',x') \\
  & = \tr\left( \int \eta_{x,p,w}^* ~d\mu(w|y,x,p)d\pp(p|x)d\rho(y,x)\times \right. \\
  &  \qquad\qquad\qquad\left. \times\int \eta_{x',p',w'}  ~d\mu(w'|y',x',p') d\pp(p'|x')d\rho(y',x') \right) \\
  & = \|\Expect ~\eta_{x,p,w}\|_\HS^2 = 0,
}
where the last equality follows from the fact that the $\eta_{x,p,w}$ have zero mean according to \cref{lem:eta-zero-mean}.

\paragraph{Combining the above cases} Note that in \eqref{eq:expectation-bound-b-bhat}, Case $1$ occurs $m$ times and Case $2$ occurs the remaining $m(m-1)$ times. Therefore, we have
\eqals{
  \Expect\|\Bhat - B\|_\HS^2 = \frac{Q_2}{m} + \frac{m-1}{m} \frac{1}{n^2}\sum_{i_j,i_h=1}^n Z_{i_j,i_h}^{j,h}
}
Now, for the second term on the right hand side, Case $2.1$ occurs $n$ times while Case $2.2$ occurs the remaining $n(n-1)$ times, leading to the desired result.
\end{proof}

\begin{lemma}\label{lem:characterization-Bfrak}
With the notation of \cref{lem:expectation-bound-b-bhat}, we have
\eqals{
\tr(\mathfrak{B}) = \mathfrak{b}_1 - \mathfrak{b}_2
}
where
\eqals{
  \mathfrak{b}_1 & = \int \scal{\gstar(x,p)}{\gstar(x,p')}_\hh~ k\big((x,p),(x,p')\big) ~ d\pp(p|x)d\pp(p'|x)d\rhox(x) \\
  \mathfrak{b}_2 & = \int \scal{\gstar(x,p)}{\gstar(x',p')}_\hh~ k\big((x,p),(x',p')\big) ~d\pp(p|x)d\pp(p'|x')d\rhox(x)d\rhox(x').
}
\end{lemma}

\begin{proof}
Note that by definition of $\eta$ and the reproducing property of the kernel $k$, for any $x,x'\in\X$, $p,p'\in\P$ and $w,w'\in\pY$ the following holds
\eqals{
\tr(\eta_{x,p,w}^* \eta_{x',p',w'}) & = \scal{\varphi(w)}{\varphi(w')}_\hh k\big((x,p),(x',p')\big) - \tr\left(B^* ~ \Big(k_{x,p}\otimes\varphi(w)\Big) \right) \\ 
& \quad - \tr\left(B^* ~ \Big(k_{x',p'}\otimes\varphi(w')\Big) \right) + \tr(B^* B).
}
Then, by definition of $B = \Expect ~ k_{x,p}\otimes\varphi(w)$, we have
\eqals{
  \tr(\mathfrak{B}) & = \int \tr\left( \eta_{x,p,w}^* \eta_{x,p',w'} \right) ~d\mu(w|y,x,p)d\mu(w'|y,x,p')d\pp(p|x)d\pp(p'|x)d\rho(y,x) \\
  & = - \tr(B^* B) +  \int \scal{\varphi(w)}{\varphi(w')}_\hh k\big((x,p),(x',p')\big) ~d\mu(w|y,x,p)d\mu(w'|y,x,p') \times \\
  & \qquad\qquad\qquad\qquad\qquad\qquad\qquad\qquad\qquad\qquad \times d\pp(p|x)d\pp(p'|x)d\rho(y,x) \\
  & = - \tr(B^* B) + \int \scal{\gstar(x,p)}{\gstar(x,p')}_\hh~ k\big((x,p),(x,p')\big) ~ d\pp(p|x)d\pp(p'|x)d\rhox(x) \\
  & = \mathfrak{b}_1 - \tr(B^* B),
}
where in the third equality we used the definition of $\gstar(x,p) = \int \varphi(w)~d\mu(w|y,x,p)d\rho(y|x)$. Moreover, since $B$ can be written in terms of $\gstar$ as 
\eqals{
  B = \int k_{x,p}\otimes \gstar(x,p) ~d\pp(p|x)d\rhox(x)
}
we have 
\eqals{
  \tr(B^*B) & =  \int \scal{\gstar(x,p)}{\gstar(x',p')}_\hh~ k\big((x,p),(x',p')\big) ~d\pp(p|x)d\pp(p'|x')d\rhox(x)d\rhox(x')\\
  & = \mathfrak{b}_2,
}
as desired. 
\end{proof}

\subsection{Learning bound in expectation}\label{sec:bound-in-expect}

We introduce here the assumption that the target function $\gstar$ of the learning problem belongs to the RKHS where we are performing the optimization. 

\begin{assumption}\label{asm:gstar-in-G}
There exists a $\mathsf{G} \in \hh \otimes \ff$, such that almost everywhere on $X \times P$,
$$\mathsf{G}k_{x,p} = g^*(x,p).$$
\end{assumption}
\noindent The following results will leverage the assumption above. 
\begin{lemma}\label{lm:frakB-propto-frakC}
Under \cref{asm:gstar-in-G}, 
\eqals{
  \tr(\mathfrak{B}) ~~\leq~~ \|\mathsf{G}\|^2 \tr(\mathfrak{C}),
}
\end{lemma}

\begin{proof}
We begin first observing that $\mathfrak{C}$ is positive semidefinite since 
\eqals{
\mathfrak{C} & = \int \zeta_{x,p} \zeta_{x,p'} ~d\pp(p|x)d\pp(p'|x)d\rhox(x)  = \Expect ~\zeta_x \zeta_x
}
is the expectation of the random variable $\zeta_x \zeta_x$, where $\zeta_x = \int \zeta_{x,p}~d\pp(p|x)$ is positive semidefinite. Moreover, by the definition of $\mathfrak{C}$ in terms of $\zeta_{x,p} = k_{x,p}\otimes k_{x,p} - C$, we have
\eqals{
\mathfrak{C} & = \int \Big(k_{x,p}\otimes k_{x,p'}\Big) ~  \big((x,p),(x,p')\big) - \Big(k_{x,p}\otimes k_{x,p}\Big) C ~ d\pp(p|x)\pp(p'|x)\rhox(x)\\
& \quad + \int C^2 - C \Big(k_{x,p'}\otimes k_{x,p'}\Big) ~d\pp(p|x)\pp(p'|x)\rhox(x)\\
& = - C^2 + \int \Big(k_{x,p}\otimes k_{x,p'}\Big) ~  \big((x,p),(x,p')\big) ~d\pp(p|x)\pp(p'|x)\rhox(x) 
}
where we have used the definition of $C=\Expect ~ k_{x,p}\otimes k_{x,p}$. 

Now note that under \cref{asm:gstar-in-G}, for any $x,x'\in\X$ and $p,p'\in\P$
\eqals{
  \scal{\gstar(x,p)}{\gstar(x',p')}_\hh = \scal{G k_{x,p}}{G k_{x',p'}}_\hh = \tr\left(G^* G ~ \Big(k_{x,p} \otimes k_{x',p'}\Big)\right).
}
Therefore, substituting the above equation in $\mathfrak b_1$ and $\mathfrak b_2$ defined in \cref{lem:characterization-Bfrak}, we have
\eqals{
  \tr(\mathfrak{B}) & = \mathfrak{b}_1 - \mathfrak{b}_2 \\ 
  & = \tr\left(G^*G \left[ \int \Big(k_{x,p} \otimes k_{x,p'}\Big) ((x,p),(x,p')) ~d\pp(p|x)\pp(p'|x)\rhox(x) -  C^2\right] \right) \\
  & = \tr(G^*G ~ \mathfrak{C})\\
  & \leq \|G\|^2 \tr(\mathfrak{C})
}
where the last inequality follows from the fact that both $G^*G$ and $\mathfrak{C}$ are positive semidefinite.
\end{proof}

\begin{theorem}\label{thm:rates-general}
  $$
  \mathbb{E}~{\cal E}(\fhat~)- {\cal E}(\fstar)  ~~\leq~~ 2~\closs \mathsf{g} ~\left[\la^{1/2} + 2\sqrt{2}\left(1 +  \left(\frac{\mathsf{r^2}}{\la^2 m} + \frac{\mathsf{q}}{\la^2 n}\right)^{1/2}\right)^{1/2} \left(\frac{\mathsf{r^2}}{\la m} + \frac{\mathsf{q}}{\la n}\right)^{1/2}\right].
  $$
  In particular when $\la \geq \sqrt{\frac{\mathsf{r^2}}{m} + \frac{\mathsf{q}}{n}}$, then
  $$\mathbb{E}~{\cal E}(\fhat~)- {\cal E}(\fstar)  ~~\leq~~ 12 ~\closs \mathsf{g} ~\left(\frac{\mathsf{r^2}}{\la m} + \frac{\mathsf{q}}{\la n} + \la\right)^{1/2}.$$
\end{theorem}
\begin{proof}
  By the comparison inequality in \cref{thm:comparison}, we have that
  $$\mathbb{E}~{\cal E}(\fhat~)- {\cal E}(\fstar) \leq 2\closs~\Expect \|\ghat - \gstar\|_\LXPH.$$
  To bound $\Expect \|\ghat - \gstar\|_\LXPH$ we need to control some auxiliary quantities. With the notation of \cref{thm:analytic-decomposition} and \cref{lem:expectation-bound-c-chat,lem:expectation-bound-b-bhat,lm:frakB-propto-frakC}, we have
  $$\Expect{\beta_1^2} \leq \frac{Q_1}{m} + \frac{\tr(\mathfrak{C})}{n} =: V, \quad \Expect{\beta_2^2} \leq \|\mathsf{G}\|V.$$
  In particular note that $\tr(\mathfrak{C}) = \msf{q}$, by \cref{lem:characterization-Cfrak} and that by definition of $Q_1$, $\msf{r}$ and $C$ we have
  \eqals{
  Q_1  & := \Expect{k_{x,p}\otimes k_{x,p} - C}_\HS^2 \\
 &  = \tr\left(\Expect~{((x,p),(x,p))(k_{x,p} \otimes k_{x,p}) - 2 C (k_{x,p} \otimes k_{x,p}) + C^2}\right) \\
  & =\tr\left(\Expect~{((x,p),(x,p))(k_{x,p} \otimes k_{x,p}) -  C^2}\right) \leq \msf{r}\tr\left(\Expect~{(k_{x,p} \otimes k_{x,p})}\right) \leq \msf{r}^2.
}
  Moreover, by \cref{asm:gstar-in-G} we have that $G = S\mathsf{G}$ and so
  $${\cal A}_{1/2}(\la) = \|L_\la^{-1/2} G\|_{\HS(\hh,\Ltwo)} = \|L_\la^{-1/2} S \mathsf{G}\|_{\HS(\hh,\Ltwo)} \leq  \|L_\la^{-1/2} S\| \| \mathsf{G}\|_{\HS(\hh,\F)} \leq \|\mathsf{G}\|_{\HS(\hh,\F)}.$$
  Analogously 
  $${\cal A}_{1}(\la) = \|L_\la^{-1} G\|_{\HS(\hh,\Ltwo)} \leq \|L_\la^{-1/2}\| \|L_\la^{-1/2}G\|_{\HS(\hh,\Ltwo)} = \la^{-1/2} {\cal A}_{1/2}(\la) \leq \la^{-1/2}\| \mathsf{G}\|_{\HS(\hh,\F)}.$$
  By plugging the bounds above in the result of \cref{thm:bound-expectation-depends-betas}, we have
  $$ \Expect \|\ghat - \gstar\|_\LXPH  \leq 2\sqrt{2}\|\mathsf{G}\|_{\HS(\hh,\F)}\sqrt{1 +  \frac{V^{1/2}}{\la} }\sqrt{\frac{V}{\la}} + \|\mathsf{G}\|_{\HS(\hh,\F)} \la^{1/2}.$$
  By selecting $\la \geq V^{1/2}$, we have
  \eqals{
  \Expect \|\ghat - \gstar\|_\LXPH  &\leq 4\|\mathsf{G}\|_{\HS(\hh,\F)}\sqrt{\frac{V}{\la}} + \|\mathsf{G}\|_{\HS(\hh,\F)} \la^{1/2} \\
  & \leq 4 \|\mathsf{G}\|_{\HS(\hh,\F)} \left(\sqrt{\frac{V}{\la}} + \la^{1/2}\right)  \\
  &\leq 4\sqrt{2} \|\mathsf{G}\|_{\HS(\hh,\F)} \left(\frac{V}{\la} + \la\right)^{1/2},
  }
  since $a^{1/2} + b^{1/2} \leq \sqrt{2(a + b)}$ for any $a,b > 0$. 
\end{proof} 
We conclude with a corollary of \cref{thm:rates-general} that frames the result within the notation and setting of the main paper and which will be useful to prove \cref{thm:rates-improved-with-parts}. 

In particular, in the following we will consider the standard assumption in the context of non-parametric estimation \cite{caponnetto2007} that $g^* \in \gg = \hh \otimes \F $, where $\F$ is the reproducing kernel Hilbert space \cite{aronszajn1950theory} associated to the kernel in \cref{eq:estimator-alpha}. The learning rates of $\fhat$ will depend on the following four constants $\mathsf{g}, \mathsf{r}, \mathsf{c}_\loss, \mathsf{q}$, where
\eqalsplit{\label{eq:constants-g-r-cDelta}
  \mathsf{g} = \|\gstar\|_\gg, \qquad \mathsf{r} = \sup_{x \in\X, p \in \P} k((x,p),(x,p)), \qquad
  \mathsf{c}_\loss^2 = \sup_{z\in\Z,x\in\X} \Expect_{p|x} \|\psi(z,x,p)\|_\hh^2.
}
Note that the quantities above are rather natural: $\mathsf{r}$ is an upper bound on the kernel $k$, $\mathsf{c}_\loss$ measures the ``complexity'' of the loss $\loss$ and $\msf{g}$ quantifies the regularity of $\rho$ in terms of the hypothesis space $\F$ associated to $k$. We will see in \cref{prop:gstar-as-gbs} that the latter is related to between-locality. Finally,
\eqalsplit{\label{eq:def-constants-q}
  \mathsf{q} = \Expect_{x,x'} \Expect_{p,q|x, r|x'} ~~\scov_{p,q}(x,x') \qquad 
  \scov_{p,q}(x,x') = \left[k((x,p),(x,q))^2 - k((x,p),(x',r))^2\right]
}
where $\Expect_{p,q|x} [\cdot]$ is a shorthand for $\sum_{p,q \in P} \pi(p|x) \pi(q|x) [\cdot]$ (analogously for $\Expect_{r|x}$). This quantity will be key to capture and leverage the within-locality assumption. In particular, it will allow us to quantify explicitly the advantages of using our locality-aware estimator.

\TRates*

\begin{proof}
    The desired result corresponds to the second statement of Theorem~\ref{thm:rates-general}.
\end{proof}
\Cref{thm:rates} characterizes the learning rates of $\fhat$ under standard regularity assumption on the problem without relying on locality assumptions. In particular, we note that when $m \propto n$ and $\la \propto n^{-1/2}$, the bound recovers the excess risk bounds of structure prediction {\em without parts} \cite{ciliberto2016,ciliberto2017consistent} of order $O(n^{-1/4})$.

\section{Learning Rates with the effect of parts}\label{sec:app-learning-rates-parts}

In this section we prove Thm.~\ref{thm:rates-improved-with-parts}, studying the effect of between-locality and within-locality on the learning problem. In particular, we consider here the natural generalization of between-locality \cref{asm:between-locality} to the case where the parts of $y$ are sampled non-deterministically from $\mu$. 

\begin{assumption}\label{asm:interaction-mu-rho}
  There exist two spaces $\pX$ and $\pY$ of parts on $\X$ and $\Y$ respectively and a conditional probability distribution  $\mubar$ on $\pY$ with respect to $\pX$, such that 
  \eqals{\label{eq:interaction-mu-rho}
    \mubar(w|x_p) = \int \mu(w|y,x,p)d\rho(y|x)
  }
\end{assumption}

\noindent Clearly, \cref{asm:interaction-mu-rho} formalizes the concept of between-locality and recovers it when $\mu$ corresponds to 
\eqals{
  \mu(\cdot|y,x,p) = \delta_{y_p}(\cdot)
}
where $\delta$ denotes the Dirac's delta on the point $y_p\in\pY$. Indeed, in this case we are requiring $w = y_p$ to depend exclusively on $x_p$ for any $p\in\P$, hence to be conditionally independent with respect to $x$. Moreover, we are requiring such distribution $\bar\mu$ to be the same for any $p\in\P$, hence recovering \cref{asm:between-locality}. The following result is therefore a generalization of \cref{prop:gstar-as-gbs}, which is recovered as a corollary.

\begin{lemma}\label{lem:general-gstar-as-gbs}
  Under \cref{asm:interaction-mu-rho}, $\gstar$ is such that $\gstar(x,p) = \gtstar(x_p)$ for any $x\in\X$ and $p\in\P$, where $\gtstar:\pX\to\hh$ is such that
  \eqals{
    \gtstar(\xi) = \int \varphi(w)~d\mubar(w|\xi)
  }
  almost surely on $\pX$.
\end{lemma}

\begin{proof}
  The result follows directly from \cref{asm:interaction-mu-rho} and the definition of $\gstar$
  \eqals{
    \gstar(x,p) = \int \varphi(w) ~d\mu(w|y,x,p)d\rho(y|x) = \int \varphi(w) ~d\mubar(w|x_p) = \gtstar(x_p),
  }
  as desired. 
\end{proof}

\begin{assumption}\label{asm:kernel-kbar}
  Denote by $\kbar:\pX\times\pX\to\R$ the reproducing kernel on $\pX$ with associated rkhs denoted by $\ggbar$, defined as for all $x,x'\in\X$ and $p,p'\in\P$
  \eqal{
    \big((x,p),(x',p')\big) = \kbar(x_p,x'_{p'})
  } 
\end{assumption}

\begin{assumption}\label{asm:k-and-g-in-Gbar}
  There exists $A_0 \in \hh \otimes \ggbar$ such that the function $\gtstar: \pX \to \hh$ can be written as
  $$\gtstar(\eta) = A_0 \kbar_{\eta}.$$
\end{assumption}

\begin{lemma}\label{lm:gstar-in-F-given-gbar-in-F}
  Under \cref{asm:kernel-kbar}, we have that ${\cal F} = \{g \circ \ix ~|~ g \in \ggbar\}$, with inner product
  $\scal{g \circ \ix }{g' \circ \ix}_{\cal F}  = \scal{g}{g'}_{\ggbar}$ is a reproducing kernel Hilbert space on $\X \times P$, with kernel $k((x,p),(x',p')) = \kbar(x_p, x'_{p'})$. Moreover there exists a linear unitary operator $U: {\cal \ggbar} \to {\cal F}$ such that $Ug = g \circ \ix \in {\cal F}$ for any $g \in \ggbar$.
  
  In particular under \cref{asm:interaction-mu-rho,asm:kernel-kbar,asm:k-and-g-in-Gbar}, we have that \cref{asm:gstar-in-G} is satisfied for $\mathsf{G} = A_0 U^*$, and
  $$\|g^*\|_{\hh \otimes {\cal F}} := \|\mathsf{G}\|_{\HS({\cal F}, \hh)} = \|A_0\|_{\HS(\bar{\gg}, \hh)} = \|\gbs\|_{\hh \otimes \bar{\gg}}.$$
\end{lemma}
\begin{proof}
  By definition $\ggbar$ is the RKHS associated to the kernel $\kbar$ on $\pX$, where the scalar product $\scal{\cdot}{\cdot}_\ggbar$ is defined such that
  $\scal{\kbar_{\eta}}{\kbar_\zeta}_\ggbar = \kbar(\eta,\zeta)$, for any $\eta, \zeta \in \pX$ and $\ggbar$ is the closure of $\ggbar_0 = \Span\{\kbar(\eta,\cdot)~|~ \eta \in \pX\}$ w.r.t. $\scal{\cdot}{\cdot}_\ggbar$. Similarly ${\cal F}$ is the RKHS associated to the kernel $k$ such that the scalar product $\scal{\cdot}{\cdot}_\F$ is defined as $\scal{k_{x,p}}{k_{x',p'}}_\F = \kbar(\ix(x,p),\ix(x',p'))$, for all $(x,p),(x',p') \in X\times P$.
  Note that by definition of $\F$, we have that $\F$ is the closure of $\F_0$ w.r.t. $\scal{\cdot}{\cdot}_\F$, with
  \eqals{
    \F_0 & = \Span\{k((x,p), (\cdot, \cdot))~|~(x,p) \in  \X \times P\} \\
    & = \Span\{\kbar(\ix(x,p), \ix(\cdot, \cdot))~|~(x,p) \in  \X \times P\} \\
    & = \Span\{\kbar(\eta, \ix(\cdot, \cdot))~|~ \eta \in  \pX\} \\
    & = \ggbar_0 \circ \ix.
  } 
  Now, since for any $\eta, \zeta \in \pX$ there exist $(x,p), (x',p') \in \pX$ such that $\eta = \ix(x,p), \zeta = \ix(x',p')$, we have that,
  \eqals{
    \scal{\kbar(\eta, \ix(\cdot,\cdot))}{\kbar(\zeta, \ix(\cdot,\cdot))}_\F & = \scal{\kbar(\ix(x,p), \ix(\cdot,\cdot))}{\kbar(\ix(x',p'), \ix(\cdot,\cdot))}_\F \\
    & = \kbar(\ix(x,p), \ix(x', p')) = \kbar(\eta, \zeta) = \scal{\kbar_\eta}{\kbar_\zeta}_\ggbar.
  }
  So, let $f, f' \in \F_0$, by definition we have $f = g \circ \ix$ and $f' = g' \circ \ix$ with $g,g' \in \ggbar_0$. Moreover by definition of $g,g'$ there exist
  $n,m \in \N$ and $\eta_1,\dots, \eta_n, \zeta_1,\dots, \zeta_{m} \in \pX$ and $\alpha_1, \dots, \alpha_n, \beta_1,\dots, \beta_m \in \R$  such that
  $g(\cdot) = \sum_{i=1}^n \alpha_i\kbar(\eta_i, \cdot)$ and analogously $g'(\cdot) = \sum_{j=1}^m \beta_j\kbar(\zeta_j, , \cdot)$.
  
  Now we show that $\scal{g \circ \ix}{g' \circ \ix}_\F = \scal{g}{g'}_\ggbar$ for $g,g' \in \ggbar_0$ and then we extend it to $\ggbar$. First we recall that the composition on the right is linear, indeed $$(\alpha f + \beta g) \circ h = \alpha (f \circ h) + \beta (g \circ h),$$
  for any $\alpha, \beta \in \R$, any function $f,g: A \to \R$ and $h: B \to A$, and $A, B$ two sets. Then we have
  \eqals{
    \scal{f}{f'}_\F &= \scal{g \circ \ix }{g' \circ \ix}_\F = \scal{\left(\sum_{i=1}^n \alpha_i\kbar(\eta_i, \cdot)\right) \circ \ix }{\left(\sum_{j=1}^m \beta_j\kbar(\zeta_j, \cdot)\right) \circ \ix} \\
    & = \scal{\sum_{i=1}^n \alpha_i \kbar(\eta_i, \ix(\cdot, \cdot))}{\sum_{j=1}^m \beta_j \kbar(\zeta_j, \ix(\cdot, \cdot))} \\
    & = \sum_{i=1}^n\sum_{j=1}^m \alpha_i \beta_j \scal{\kbar(\eta_i, \ix(\cdot,\cdot))}{\kbar(\zeta_j, \ix(\cdot,\cdot))}_\F \\
    & = \sum_{i=1}^n\sum_{j=1}^m \alpha_i \beta_j \scal{\kbar_{\eta_i}}{\kbar_{\zeta_j}}_\ggbar  = \scal{\sum_{i=1}^n \alpha_i \kbar_{\eta_i}}{ \sum_{j=1}^m \beta_j \kbar_{\zeta_j}}_\ggbar \\
    & = \scal{g}{g'}_\ggbar.
  }
  By noting that
  $$\|g_n \circ \ix - g_m \circ \ix\|_\F = \|(g_n - g_m) \circ \ix \|_\F = \|g_n - g_m \|_\ggbar$$
  for any Cauchy sequence $(g_n)_{n \in \N}$ in $\ggbar_0$, and the fact that $\F_0 = \ggbar \circ \ix$ and that $\scal{g \circ \ix}{g \circ \ix}_\F = \scal{g}{g'}_\ggbar$, for $g,g' \in \ggbar_0$, then we have that $\F = \ggbar \circ \ix$, and that $\scal{g \circ \ix}{g \circ \ix}_\F = \scal{g}{g'}_\ggbar$, for $g,g' \in \ggbar$. 
  
  Now denote by $U: \ggbar \to \F$ the operator such that $U g = g \circ \ix$. First note that $U$ is linear, indeed 
  $$U (\alpha g + \beta h) = (\alpha g + \beta h) \circ \ix = \alpha (g \circ \ix) + \beta (h \circ \ix) = \alpha U g + \beta U h,$$
  for any $g, h \in \ggbar$ and $\alpha, \beta \in \R$. Moreover we show that $U$ is a partial isometry, indeed
  $$\|U g\|_\F^2 = \|g \circ \ix\|_G^2 = \scal{g \circ \ix}{g \circ \ix}_\F = \scal{g}{g}_\ggbar = \|g\|^2_\ggbar.$$ 
  
  Finally by applying the result above to $g^*$ and $\gbs$, under \cref{asm:interaction-mu-rho,asm:kernel-kbar,asm:k-and-g-in-Gbar}, we have that $\mathsf{G} = A_0 U^*$ and so, by using the isomorphism between $\hh \otimes \ff$ and $\HS(\ff,\hh)$, we have
  $$\|g^*\|_{\hh \otimes {\cal F}} := \|\mathsf{G}\|_{\HS({\cal F}, \hh)} = \|A_0\|_{\HS(\bar{\gg}, \hh)} = \|\gbs\|_{\hh \otimes \bar{\gg}},$$
as desired. 
\end{proof}

\begin{assumption}\label{asm:p-independent}
  The distribution $\pp(\cdot|x) = \pp(\cdot|x')$ for any $x,x'\in\X$. For the sake of simplicity we will denote it by $\pp(\cdot)$.
\end{assumption}

\begin{lemma}\label{lem:q-as-p-q-appendix}
  Under \cref{asm:p-independent}, the following hold
  \eqals{
    \msf{q} ~~=~~ \Expect_{p q} ~\mathfrak{c}_{p q}, 
  }
  where, for $p,q \in P$
  \eqals{
    \mathfrak{c}_{pq} ~~=~~ \Expect_{x,x'}~ \big[k((x,p), (x,q))^2 - k((x,p),(x',q))^2\big].
  }
\end{lemma}

\begin{proof}
First note that with the definitions of \cref{lem:characterization-Cfrak},  we have 
$$\msf{q} = \mathfrak{c}_1 - \mathfrak{c}_2$$
by \cref{lem:characterization-Cfrak} .
Under \cref{asm:p-independent} we can denote $\pp(\cdot|x) = \pp(\cdot)$ without ambiguity. Then with the notation of \cref{lem:characterization-Cfrak}, we have 
\eqals{
  \mathfrak{c}_1 & = \int k\big((x,p),(x,q)\big)^2 ~ d\pp(p)d\pp(q)d\rhox(x) \\
  & = \Expect_{p,q} ~ \int k\big((x,p),(x,q)\big)^2 ~d\rhox(x)\\
  & = \Expect_{p,q} \Expect_x k\big((x,p),(x,q)\big)^2
}
Analogously for $\mathfrak{c}_2$
\eqals{
    \mathfrak{c}_2 & = \int k\big((x,p),(x',q)\big)^2 ~ d\pp(p)d\pp(q) ~d\rhox(x)\rhox(x') \\
  & = \Expect_{p,q} ~ \int k\big((x,p),(x,q)\big)^2 ~d\rhox(x)\rhox(x')\\
  & = \Expect_{p,q} \Expect_{x,x'} k\big((x,p),(x,q)\big)^2 
}
\end{proof}
As an immediate corollary in the case where $\P$ has finite cardinality, we have

\begin{corollary}\label{cor:q-as-sigma-p-q}
  Under the same assumptions of \cref{thm:rates}, let $k$ denote the restriction kernel defined in \cref{eq:restriction-kernel} in terms of $\bar k:\pX\times\pX\to\R$. Let $\pi(p|x) = \frac{1}{|P|}$ for any $x\in\X$ and $p\in\P$. Then, the constant $\msf q$ in \cref{eq:def-constants-q} can be factorized as
  \eqalsplit{\label{eq:intra-locality-measures}
    & \msf q = \frac{1}{|P|^2} \sum_{p,q \in P}  \scov_{p,q}, \\
    & \scov_{p,q} =  \Expect_{x,x'} \left[ ~\bar k(x_p,x_q)^2 - \bar k(x_p,x'_q)^2 ~ \right].
  }
\end{corollary}

\subsection{Proof of Theorem~\ref{thm:rates-improved-with-parts}}\label{sec:proof-rates-locality}

\begin{proof}
  This proof consists in applying Theorem~\ref{thm:rates} with $\la = \sqrt{{\msf r}^2/m + {\msf q}/n}$, and taking into account between-locality and within-locality.

  First, under the between-locality condition formalized in our measure theoretic setting as Assumption~\ref{asm:interaction-mu-rho}, there exists a $\bar{g}^*:\pX \to \hh$ such that $g^*(x,p) = \bar{g}^*(x_p)$ for any $x \in \X$ and $p \in P$ as proven by \cref{lem:general-gstar-as-gbs}. 
  So the restriction kernel can learn $\bar{g}^*$ if it is rich enough, that is $\bar{g}^* \in \hh \otimes \bar{\F}$ (here formalized as Assumption~\ref{asm:k-and-g-in-Gbar}, with $\bar{\ff}$ denoted by $\bar{\gg}$). Then we can apply \cref{lm:gstar-in-F-given-gbar-in-F}, that guarantees the applicability of Theorem~\ref{thm:rates}.
  
  Second, by the assumption on the fact that $\pi(p|x) = 1/|P|$, we can apply \cref{cor:q-as-sigma-p-q} and then the within-locality condition of Assumption~\ref{asm:within-locality}, obtaining the desired result.
\end{proof}

\section{Universal Consistency}\label{sec:app-consistency}
A natural question is how to design a structured prediction estimator that is both able to leverage the locality assumptions, when they hold, and be universally consistent even when there is no locality. The following remark addresses this questions and concludes our theoretical analysis.
\begin{theorem}[Universal Consistency]\label{thm:universal}
  Let $\loss$ be SELF and $Z$ a compact set. Let $k$ be a bounded continuous universal kernel on $\X\times\P$. Let $\fhat_n$ as in \cref{eq:estimator} with i.i.d.~training set and auxiliary dataset sampled according to \cref{sec:algorithm}, with $m \propto n$. Then
  \eqal{
    \lim_{n\to\infty} \E(\fhat_n~) = \inf_{f:\X \to \Z} \E(f) \quad \textrm{with probability }1.
  }
\end{theorem}
\begin{proof}
\cref{sec:universality-proof} is devoted to the proof.
\end{proof}
The requirement of universality for the kernel is a standard assumption (see \cite{steinwart2008}). An example of continuous universal kernel on $\X \times P$ is $k((x,p),(x',p')) = k_0(x,x') ~ \delta_{p,p'}$ where $k_0$ is any unversal kernel on $\X$, e.g. the Gaussian $k_0(x,x') = \exp(-\|x-x'\|^2)$.

While the proposed estimator is consistent with the kernel described above, it is not able to benefit from the effect of locality. In the following we comment on how to obtain a kernel that guarantees both universal consistency while leveraging locality at the same time. 

\begin{remark}[Universal and Local Kernels]\label{rem:universal-and-local-kernels}
By construction, the restriction kernel allows to learn only functions $\gstar:\X\times\P\to\hh$ such that $\gstar(x,p) = \gbs(x_p)$. Consequently, the corresponding structured prediction estimator is not universal. However, in \cref{thm:rates-improved-with-parts} we have observed that under the locality assumptions, the restriction kernel achieves significantly faster rates with respect to universal kernels that are not tailored to account for the part structure on the input.

Interestingly, it is possible to design a kernel able to take the best of both worlds, leading to an estimator that is universal but also able to leverage the parts-based structure of a learning problem when possible. We obtain this kernel as the sum $k_B = k_U + k_L$ of a universal kernel $k_U$ on $\X\times\P$ and a restriction (or ``local'') kernel $k_L$. Indeed, as shown in \cref{sec:app-sum-kernel}, the kernel $k_B$ is universal, hence \cref{thm:universal} applies to the corresponding estimator $\fhat$. Moreover, under the locality assumptions, a result identical to \cref{thm:rates-improved-with-parts} holds for the estimator trained with $k_B$.
\end{remark}

\subsection{Proof of \cref{thm:universal}}\label{sec:universality-proof}

  The proof is exactly the same as in Theorem 4 Section B.3 of the supplementary materials of \cite{ciliberto2016}, where instead of using their comparison inequality (their Thm.~2) we use our \cref{thm:comparison} and instead using their Lemma~18 we use our Theorem~\ref{thm:last-result-universality} that is proven at the end of this section. First we introduce some concentration inequalities for separable Hilbert spaces.
  
\begin{proposition}\label{prop:pinelis}
  Let $\delta\in(0,1]$ and $m\in\N$. Let $\hh$ be a separable Hilbert space. Let $\zeta_1,\dots,\zeta_m$ be independently distributed $\hh$-valued random variables. Let $R>0$ be such that  $\esssup \|\zeta_j\|_\hh\leq R$ for every $j=1,\dots,m$. Then,
  \eqal{
    \nor{\frac{1}{m}\sum_{j=1}^m \Big[\zeta_j - \Expect~\zeta_j\Big]}_\hh \leq \frac{4R\log\frac{3}{\delta}}{\sqrt{m}}
  }
  with probability at least $1-\delta$.
\end{proposition} 

\begin{proof}
  By applying Lemma $2$ of \cite{smale2007learning} with constants $\widetilde M = R$ and $\sigma^2 =\sup_j \Expect \|\zeta_j\|^2 \leq R^2$, we obtain 
  \eqal{
    \nor{\frac{1}{m}\sum_{j=1} \big[\zeta_j - \Expect\zeta_j \big]}_\HS \leq \frac{2R\log\frac{2}{\delta}}{m} + \sqrt{\frac{2R^2\log\frac{2}{\delta}}{m}}
  }
  with probability at least $1-\delta$. Now, $\log\frac{2}{\delta}\leq\log\frac{3}{\delta}$ and $\log{3}{\delta}\geq1$ for any $\delta\in(0,1]$. Then, we can bound the above inequality by 
  \eqal{
    \frac{2R\log\frac{2}{\delta}}{m} + \sqrt{\frac{2R^2\log\frac{2}{\delta}}{m}} \leq \frac{4R\log\frac{3}{\delta}}{\sqrt{m}},
  }
  as desired. 
\end{proof}

\begin{remark}[Pinelis Inequality for Hilbert-Schmidt Operators]\label{rem:pinelis-operators}
  We recall that the space of Hilbert-Schmidt operators between two separable Hilbert spaces is itself a separable Hilbert space with the Hilbert-Schmidt norm. Therefore, Pinelis inequality in \cref{prop:pinelis} is directly applicable.
\end{remark}

\begin{lemma}\label{lemma:bounding-C-Chat}
Let $C$ and $\Chat$ and $\kappa = \sup_{x,p} \|k_{x,p}\|_\ff$ defined as \cref{sec:notation}. Let $\delta\in(0,1]$. Then 
\eqals{
  \|\Chat - C\| \leq 4\kappa^2 \left(\frac{1}{\sqrt m} + \frac{1}{\sqrt n}\right)\log\frac{6}{\delta}
}
with probability at least $1 - \delta$.
\end{lemma}

\begin{proof}
Given a dataset $(x_i)_{i=1}^n$, we introduce the operator $\Ctilde:\ff\to\ff$ defined as 
\eqals{
    \Ctilde = \frac{1}{n}\sum_{i=1}^n \int_{\P} k_{x_i,p}\otimes k_{x_i,p}~d\pp(p|x_i).
}    
and consider the following decomposition
\eqals{
  \|\Chat - C\| \leq \|\Chat - \Ctilde\| + \|\Ctilde - C\|.
}
Let $\tau = \delta/2$, in the following we separately bound the terms above in probability and then take the intersection bound.

\paragraph{Bounding $\|\Chat - \Ctilde\|$} For any $j=1,\dots,m$ let $\zeta_j = k_{x_{i_j},p_j} \otimes k_{x_{i_j},p_j}$ with $i_j$ and $p_j$ independently sampled respectively from: the uniform distribution on $\{1,\dots,n\}$ and the conditional probability $\pp(\cdot|x_{i_j})$. Therefore, for any $j=1,\dots,m$
\eqals{
  \Chat = \frac{1}{m} \sum_{j=1}^m \zeta_j, \qquad \Ctilde = \Expect~ \zeta_j = \frac{1}{n} \sum_{i=1}^n \int_\P k_{x_i,p}\otimes k_{x_i,p} ~ d\pp(p|x_i)
}
and
$$
 \esssup \|\zeta_j\|_\HS \leq \sup_{x\in\X,p\in\P} \scal{k_{x,p}}{k_{x,p}}_\ff \leq \sup_{x\in\X,p\in\P} \|k_{x,p}\|_\ff^2 \leq \kappa^2 
$$
We apply Pinelis inequality (see \cref{rem:pinelis-operators}), leading to 
\eqal{
  \|\Chat - \Ctilde\|\leq\|\Chat - \Ctilde\|_\HS = \nor{\frac{1}{m}\sum_{j=1} \big[\zeta_j - \Expect\zeta_j \big]}_\HS \leq \frac{4\kappa^2\log\frac{3}{\tau}}{\sqrt m}
}
with probability at least $1-\tau$.

\paragraph{Bounding $\|\Ctilde - C\|$} For $i=1,\dots,n$ let $\eta_i = \int_\P k_{x_i,p}\otimes k_{x_i,p} ~ d\pp(p|x_i)$ with $x_i$ independently sampled from $\rhox$. Therefore, for every $i=1,\dots,n$,
\eqal{
  \Ctilde = \frac{1}{n}\sum_{i=1}^n \eta_i, \qquad C = \Expect ~ \eta_i = \int_{\X\times\P} k_{x,p}\otimes k_{x,p}~d\pp(p|x)d\rhox(x)
}
and
$$
 \esssup \|\eta_i\|_\HS \leq \sup_{x\in\X,p\in\P} \|k_{x,p}\|_\ff^2 \leq \kappa^2.
$$
We apply again Pinelis inequality, obtaining
\eqal{
  \|\Ctilde-C\|\leq\|C - \Ctilde\|_\HS = \nor{\frac{1}{n}\sum_{i=1}^n \big[\eta_i - \Expect\eta_i\big]}_\HS\leq \frac{4\kappa^2\log \frac{3}{\tau}}{\sqrt n}
}
with probability at least $1-\tau$.

By taking the intersection bound of the two events above, we obtain 
\eqals{
  \|\Chat - C\|_\HS \leq \frac{4\kappa^2\log \frac{3}{\tau}}{\sqrt m} + \frac{4\kappa^2\log \frac{3}{\tau}}{\sqrt n}
}
with probability at least $1 - 2\tau$. By recalling $\tau = \frac{\delta}{2}$ we obtain the desired result. 
\end{proof}

\begin{lemma}\label{lemma:bounding-B-Bhat}
Let $B$, $\Bhat$, $\kappa = \sup_{x,p} \|k_{x,p}\|_\ff$ and $q = \sup_{w} \|\varphi(w)\|_\hh$ defined as \cref{sec:notation}. Let $\delta\in(0,1]$. Then 
\eqals{
  \|\Bhat - B\|_\HS \leq  4\kappa q \left(\frac{1}{\sqrt m} + \frac{1}{\sqrt n}\right)\log\frac{6}{\delta}
}
with probability at least $1 - \delta$.
\end{lemma}

\begin{proof} Given $(x_i,y_i)_{i=1}^n$ a dataset, we introduce the operator $\Btilde:\hh\to\ff$ defined as 
\eqals{
    \Btilde = \frac{1}{n}\sum_{i=1}^n \int_{\P} k_{x_i,p}\otimes\varphi(w) ~ d\mu(w|y_i,x_i,p)d\pp(p|x_i).
}    
and consider the following decomposition
\eqals{
  \|\Bhat - B\|_\HS \leq \|\Bhat - \Btilde\|_\HS + \|\Btilde - B\|_\HS.
}
Let $\tau = \delta/2$, in the following we separately bound the terms above in probability and then take the intersection bound.

\paragraph{Bounding $\|\Bhat - \Btilde\|_\HS$} For any $j=1,\dots,m$ let $\xi_j = k_{x_{i_j},p_j}\otimes\varphi(w_j)$ with $i_j,p_j$ and $w_j$ independently sampled respectively from: the uniform distribution on $\{1,\dots,n\}$, the conditional probability $\pp(\cdot|x_{i_j})$ and the conditional probability $\mu(\cdot|y_{i_j},x_{i_j},p_j)$. Therefore, for any $j=1,\dots,m$
\eqal{
  \Bhat = \frac{1}{m} \sum_{j=1}^m \xi_j, \qquad \Btilde = \Expect~\xi_j = \frac{1}{n} \sum_{i=1}^n \int_{\pY\times\P} k_{x_i,p}\otimes\varphi(w)~d\mu(w|x_i,y_i,p)d\pp(p|x_i),
}
moreover
\eqal{
  \textrm{ess sup } \|\xi_j\|_\HS \leq \sup_{x,p,w} \|k_{x,p}\otimes\varphi(w)\|_\HS = \sup_{x,p,w}  \|k_{x,p}\|_\ff\|\varphi(w)\|_\hh \leq \kappa q.
}
We apply Pinelis inequality (see \cref{rem:pinelis-operators}), leading to 
\eqal{
  \|\Bhat - \Btilde\|_\HS = \nor{\frac{1}{m}\sum_{j=1}^m \big[\xi_j - \Expect~\xi_j \big]}_\HS \leq \frac{4\kappa q\log \frac{3}{\tau}}{\sqrt m}
}
with probability at least $1-\tau$.

\paragraph{Bounding $\|B - \Btilde\|_\HS$} For any $i=1,\dots,n$, let $\nu_i = \int_{\pY\times\P} k_{x_i,p}\otimes \varphi(w)~d\mu(w|y_i,x_i,p)d\pp(p|x_i)$ with $(x_i,y_i)$ independently sampled from $\rho$. Then, for any $i=1,\dots,n$ 
\eqal{
  \Expect ~\nu_i & = \int_{\pY\times\Y\times\X\times\P} k_{x,p}\otimes \varphi(w)~d\mu(w|y_i,x_i,p)d\pp(p|x_i)d\rho(y,x) \\ 
  & = \int_{\X\times\P} k_{x,p}\otimes \Big[\int_{\pY\times\Y}\varphi(w)~d\mu(w|y_i,x_i,p)d\rho(y|x)\Big]~d\pp(p|x_i)d\rhox(x) \\
  & = \int_{\X\times\P} k_{x,p}\otimes\gstar(x,p)~d\pp(p|x_i)d\rhox(x) \\ 
  & = B
}
and $\Btilde = \frac{1}{n}\sum_{i=1}^n \nu_i$. Moreover,
\eqal{
  \textrm{ess sup } \|\nu_i\|_\HS & \leq \sup_{x,y} \int_{\pY\times\P} \|k_{x,p}\otimes\varphi(w)\|_\HS ~d\mu(w|y,x,p)d\pp(p|x) \\
  & = \sup_{x,y} \int_{\pY\times\P} \|k_{x,p}\|_\ff\|\varphi(w)\|_\hh ~d\mu(w|y,x,p)d\pp(p|x) \\
  & \leq \kappa q ~ \sup_{x,y} \int_{\pY\times\P} ~d\mu(w|y,x,p)d\pp(p|x) \\
  & = \kappa q\\
}
Therefore, applying again Pinelis inequality,
\eqal{
  \|B - \Btilde\|_\HS = \nor{\frac{1}{n} \sum_{i=1}^n \big[ \nu_i - \Expect\nu_i \big] }_\HS \leq \frac{4\kappa q\log \frac{3}{\tau}}{\sqrt n}
}
with probability at least $1-\tau$.

By taking the intersection bound of the two events above, we obtain 
\eqals{
  \|\Bhat - B\|_\HS \leq \frac{4\kappa q\log \frac{3}{\tau}}{\sqrt m} + \frac{4\kappa q\log \frac{3}{\tau}}{\sqrt n}
}
with probability at least $1 - 2\tau$ as desired.
\end{proof}

\begin{theorem}\label{thm:last-result-universality}
Let $\delta\in(0,1]$. Let $Q>0$, $n\in\N$, $c_Q = 1 + 1/\sqrt{Q}$ and $m = Qn$. Then
\eqal{
\|\ghat - \gstar\|_\LXPH \leq \frac{4\kappa^2 c_Q(\|L_\la^{-1/2}G\|_\HS + \frac{q}{\kappa})\log\frac{12}{\delta}}
{\sqrt{\la n}} \left(1 + 2\kappa\sqrt{\frac{c_Q\log\frac{12}{\delta}}{\la \sqrt{n}}}\right) + \la\|L_\la^{-1}G\|_\HS
}
with probability at least $1-\delta$.
\end{theorem}

\begin{proof}
In \cref{thm:analytic-decomposition} we have bounded $\|\ghat - \gstar\|_\LXPH$ in terms of an analytic expression of $\|C - \Chat\|$ and $\|B - \Bhat\|_\HS$. We bound these two terms with probability $1-\tau$ with $\tau = \delta/2$ via \cref{lemma:bounding-C-Chat} and \cref{lemma:bounding-B-Bhat}. We further take the intersection bound to obtain the desired result. 
\end{proof}





\section{Equivalence between SELF and SELF by Parts without assumptions}\label{sec:app-equivalence-self}

\subsection{SELF without Parts}

We begin by briefly recalling the SELF framework in \cite{ciliberto2016}. We will see that this is a special case of the setting proposed in this work for a special choice of the kernel on $\X\times\P$. 

We recall the definition of SELF introduced in \cite{ciliberto2016} and consider the formulation in \cite{ciliberto2017consistent}.

\begin{definition}\label{def:self-original}
  A function $\loss:\Z\times\Y\to\R$ is a Structure Encoding Loss Function (SELF) if there exist a Hilbert space $\bar\hh$ and two maps $\bar\psi:\Z\to\hh$ and $\bar\varphi:\Y\to\hh$ such that 
  \eqals{
    \loss(z,y) = \scal{\bar\psi(z)}{\bar\varphi(y)}_{\bar\hh}
  }
  for all $z\in\Z, y\in\Y$. 
\end{definition}
Below we show that the definition of SELF by parts introduced in this work is a refinement of the original one. Since the original definition of SELF did not account for the possibility of $\loss$ do depend also on the input, below we consider only the case $\loss(z,y|x) = \loss(z,y)$. In particular we will assume in \cref{def:self-simple} that $\pp(p|x) = \pp(p|x')$ for any $x,x'\in\X$, $p\in\P$ and denote it $\pp(p)$.

\begin{lemma}\label{lem:self-old-and-new}
  Let $\loss:\Z\times\Y\to\R$ satisfy \cref{def:self-simple} with 
  \eqals{
    \loss(z,y) = \sum_{p\in\P} \ell(z,y|p) \pp(p) = \sum_{p\in\P} ~\scal{\psi(z,p)}{\varphi(y_p)}_\hh
  }
  Then $\loss$ satisfies the original SELF definition \cref{def:self-original}, with $\bar\hh = \hh\otimes\R^P$ and maps $\bar\psi:\Z\to\bar\hh$ and $\bar\varphi:\Y\to\bar\hh$ such that
  \eqals{
    \bar\psi(z) = (\sqrt{\pp(p)}\psi(z,p))_{p\in\P} \qquad \textrm{and} \qquad (\sqrt{\pp(p)}\varphi(y_p))_{p\in\P}
  }
  In particular, we have that the constant $\closs$ is 
  \eqals{
    \closs = \sqrt{\sup_{z\in\Z} \sum_{p\in\P} \pp(p) \|\psi(z,p)\|_\hh^2} = \sup_{z\in\Z} \|\bar\psi(z)\|_{\bar\hh}.  
  }
\end{lemma}
\begin{proof}
  Recall that by construction $\bar\hh = \hh\otimes\R^P = \bigoplus_{p\in\P} \hh$. Therefore, any vector $\eta\in\bar\hh$ is the collection $(\eta_p)_{p\in\P}$ with $\eta_1,\dots,\eta_P \in\hh$ and the corresponding inner product with a $\zeta = (\zeta_p)_{p\in\P}\in\bar\hh$ is 
  \eqals{
    \scal{\eta}{\zeta}_{\bar\hh} = \sum_{p\in\P} \scal{\eta_p}{\zeta_p}_\hh.
  }
  Plugging the definition of $\bar\psi$ and $\bar\varphi$ in the definition of SELF by parts, we have 
  \eqals{
    \loss(z,y) & = \sum_{p\in\P} \pp(p) \scal{\varphi(z,p)}{\psi(y_p)}_\hh \\
    & = \sum_{p\in\P}  \scal{\sqrt{\pp(p)}\varphi(z,p)}{\sqrt{\pp(p)}\psi(y_p)}_\hh \\
    & = \scal{\bar\psi(z)}{\bar\varphi(y)}_{\bar\hh}
  }
  as required.
\end{proof}

\subsection{SELF Solution}

Given a loss $\loss$ that is a SELF by parts, we have already observed that the solution $\fstar:\X\to\Z$ of the structured prediction problem in \eqref{eq:base-fstar}, can be characterized in terms of a function $\gstar:\X\times\P\to\hh$ introduced in \eqref{eq:char-fstar}. Based on the relation highlighted by \cref{lem:self-old-and-new}, we have the following equivalent characterization 
\eqals{
  \fstar(x) = \argmin_{z\in\Z} ~\scal{\bar\psi(z)}{\hstar(x)}_{\bar\hh}
}
where now $\hstar:\X\to\bar\hh$ is conditional mean embedding of $\bar\varphi(y)$ in $\bar\hh$ with respect to the conditional distribution $\rho(y|x)$. In particular, let $e_p\in\R^P$ denote the $p$-th element of the canonical basis in $\R^P$. Then, for any $\eta\in\hh$, $x\in\X$ and $p\in\P$, we have
\eqals{
  \scal{\hstar(x)}{\eta\otimes e_p}_{\bar\hh} & = \scal{\int \bar\varphi(y)~d\rho(y|x)}{\eta\otimes e_p} \\
  = \sqrt{\pp(p)} \scal{\int \varphi(y_p)~d\rho(y|x)}{\eta}_\hh\\
  = \sqrt{\pp(p)}\scal{\gstar(x,p)}{\eta}_\hh,
}
and in particular, 
\eqals{
  \hstar(x) = (\sqrt{\pp(p)}\gstar(x,p))_{p\in\P}.
}
We conclude that 
\eqals{
  \|\hstar\|_{\Ltwo(\X,\rhox,\bar\hh)}^2 & = \int \scal{\hstar(x)}{\hstar(x)}_{\bar\hh} ~d\rhox(x) \\
  & = \int \sum_{p\in\P} \scal{\sqrt{\pp(p)}\gstar(x,p)}{\sqrt{\pp(p)}\gstar(x,p)}_{\hh} ~d\rhox(x) \\
  &= \int \sum_{p\in\P} \pp(p)\scal{\gstar(x,p)}{\gstar(x,p)}_{\hh} ~d\rhox(x) \\
  & = \|\gstar\|_{\Ltwo{\X,\pp\rhox,\hh}}^2.
}

\subsection[alt-text]{If $\gstar$ is ``simple'' (e.g. \cref{asm:between-locality} holds)}\label{sec:kernel-comparison}

Let $\bar k$ be a kernel on $\X$ with RKHS $\F$. Let $k$ be a kernel on $\X \times P$ defined as $k((x,p),(x',p')) = \bar k(x,x') \delta_{p,p'}$, for $x, x' \in \X$, $p, p' \in P$. Note that the RKHS associated to $k$ is $\F \otimes \R^{P}$ with $k_{x,p} = \bar k_x \otimes e_p$ and $e_p\in\R^P$ the $p$-th element of the canonical basis of $\R^P$. 


\begin{lemma}
  Let $\sf G\in\hh\otimes\F\otimes\R^P$ be such that $g^*(x,p) = \sf G k_{x,p}$ for any $x\in\X$ and $p\in\P$. Let $G_1,\dots,G_P\in\hh\otimes\F$ the operator such that $G_p \eta = G (\eta\otimes e_p)$ for any $p\in\P$ and $\eta\in\F$. Then, 
  \begin{itemize}
    \item $\msf G = \sum_{p\in\P} G_p \otimes e_p$.
    \item For any $x\in\X$, $\hstar(x) = \msf H \bar k_x$ with $\msf H = \sum_{p\in\P} e_p \otimes \sqrt{\pp(p)} \msf G_p \in\R^P \otimes \hh \otimes \F$.
    
  \end{itemize}
  In particular 
  \eqals{
    \|\msf G\|_{\HS(\F\otimes\R^P,\hh)}^2 = \sum_{p\in\P} \|\msf G_p\|_{\HS(\F,\hh)}^2 \qquad \textrm{and} \qquad \|H\|_{\HS(\F,\hh\otimes\R^P)}^2 = \sum_{p\in\P} \pp(p)\|\msf G_p\|_{\HS(\F,\hh)}^2.
  }
\end{lemma}

\begin{lemma}
  Let $\sf G\in\hh\otimes(\F\otimes\R^P)$ be such that $g^*(x,p) = \sf G k_{x,p}$ for any $x\in\X$ and $p\in\P$. Let $G_1,\dots,G_P\in\hh\otimes\F$ the operator such that $G_p \eta = G (\eta\otimes e_p)$ for any $p\in\P$ and $\eta\in\F$. Then, there exists an operator $\msf H\in(\hh\otimes\R^P)\otimes\F$, such that 
  \begin{itemize}
    \item $\msf H \bar k_x = h^*(x)$ for all $x\in\X$. 
    \item $\|G\|_{\HS(\F\otimes\R^P,\hh)}^2 = \sum_{p\in\P} \|\msf G_p\|_{\HS(\F,\hh)}^2$.
    \item $\|H\|_{\HS(\F,\hh\otimes\R^P)}^2 = \sum_{p\in\P} \pp(p)\|\msf G_p\|_{\HS(\F,\hh)}^2$.
  \end{itemize}
\end{lemma}

\begin{proof}
  Note that since $e_p$ form a basis of $\R^P$, we can write $G = \sum_{p\in\P} G_p \otimes e_p$ and therefore 
  \eqals{
    \|G\|_{\HS(\F\otimes\R^P,\hh)}^2 = \sum_{p\in\P} \|\msf G_p\|_{\HS(\F,\hh)}^2
  }
  as required.
  
  Now, by definition of $\hstar$ and the relation with $\gstar$, we have that 
  \eqals{
    \hstar(x) & = (\sqrt{\pp(p)}~\gstar(x,p))_{p\in\P} \\
    &  = (\sqrt{\pp(p)}~\msf G k_{x,p} )_{p\in\P} \\
    & = \big(\sqrt{\pp(p)}~\msf G (\bar k_{x} \otimes e_p ) \big)_{p\in\P} \\
    & = \big(\sqrt{\pp(p)}\msf G_p \bar k_x\big)_{p\in\P}\\
    & = \msf H \bar k_x,
  }
  where we have denoted with $\msf H \in (\hh\otimes\R^P)\otimes\F$, the operator from $\F$ to $\hh\otimes\R^P$, such that for any $\eta\in\F$ we have $\msf H = \big(\sqrt{\pp(p)}\msf G_p \eta \big)_{p\in\P}$. The required results follow directly from the construction of both $\msf G$ and $\msf H$ in terms of the $\msf G_p$ for $p\in\P$.
\end{proof}
We can therefore conclude the equivalence between the original SELF estimator with kernel $\bar k$ and the SELF estimator by parts considered in this work, with kernel $k$, under the assumption that $\gstar$ (and equivalently $\hstar$) belong to the corresponding RKHS.

\begin{theorem}
  The SELF estimator with kernel $\bar k$ has same rates as the SELF by parts with kernel $k$
\end{theorem}
%
%
%
%
%
%
%
%
%
%
For simplicity, assume $\pp(p|x) = \frac{1}{|P|}$ for every $x\in\X$ and $p\in\P$. From \eqref{eq:loss-base-parts} and the SELF assumption, we have 
\eqals{
  \loss(z,y|x) = \frac{1}{|P|} \sum_{p\in P} ~\scal{\psi(z_p,x_p,p)}{\varphi(y_p)}_\hh.
}
Denote $\bar\psi:\Z\times\X\to\hh\otimes\R^P$ and $\bar\varphi:\Y\to\hh\otimes\R^P$ the maps such that 
\eqals{
  \bar\psi(z,x) = \Big(\psi(z_p,x_p,p)\Big)_{p\in\P} \qquad   \bar\varphi(y) = \Big(\varphi(y_p)\Big)_{p\in\P} 
}
which can be interpreted as the concatenation of the different $\psi$ and $\varphi$ for $p\in\P$. Then we can rewrite $\loss$ in terms of the canonical inner product of $\hh\otimes\R^P$, 
\eqals{
  \loss(z,y|x) = \frac{1}{|P|} \scal{\bar\psi(z,x)}{\bar\varphi(y)}_{\hh\otimes\R^P}. 
}
We can now apply the approach proposed in this work to the case of a problem {\em with one single part} (or equivalently apply the SELF approach in \cite{ciliberto2016}). The target function of this problem is $h^*:\X\to\hh\otimes\R^P$ defined as
\eqals{
  h*(x) = \frac{1}{|P|}\int \bar\varphi(y) ~ d\rho(y|x) = \frac{1}{|P|} \Big(\int \varphi(y_p)~d\rho(y|x)\Big)_{p\in\P} = \frac{1}{|P|}(\gstar(x,p))_{p\in\P} \in\hh\otimes\R^P
}
and is the concatenation of all functions $\gstar(\cdot,p)$ for $p\in\P$.

Now, let us consider a rkhs $\F$ of functions $h:\X\to\R$ with associated kernel $k:\X\times\X\to\R$. Assume  that $h^*$ belongs to the space of vector valued functions $\F\otimes(\hh\otimes\R^P)$. In other words, there exists an Hilbert-Schmidt operator $H:\F\to \hh\otimes\R^P$ such that $H k_x = h^*(x)$ for any $p\in\P$. Note that this is equivalent to require that the function $g^*$ belongs to the space $(\F\otimes\R^P)\otimes\hh$, namely that there exists an Hilbert-Schmidt operator, such that $G:\F\otimes\R^P\to\hh$, such that, $G (k_x\otimes e_p) = g^*(x,p)$ for any $x\in\X$ and $p\in\P$, with $e_p\in\R^P$ denoting the $p$-th element of the canonical basis of $\R^P$. In particular, note that, for any $\eta\in\hh$, $p\in\P$ and $x\in\X$, we have
\eqals{
  \scal{H k_x}{\eta\otimes e_p}_\hh = \scal{h^*(x)}{\eta\otimes e_p} = \scal{h^*(x)_p}{\eta}_\hh = \scal{g^*(x,p)}{\eta}_\hh = \frac{1}{|P|}\scal{G (k_x \otimes e_p)}{\eta}.
}
We conclude that $H = \frac{1}{|P|} G$ and $\|H\|_{\HS} = \frac{1}{\sqrt{|P|}} \|G\|_{\HS}$. In particular, note that since $G\in(\F\otimes\R^P)\otimes\hh$, we have that for any $p\in\P$, the function $g(\cdot,p):\X\to\hh$ is such that $g(\cdot,p)\in\F\otimes\hh$. Therefore we have 
\eqals{
  \|G\|_{\HS} = \sqrt{\sum_{p\in\P} \|g^*(\cdot,p)\|_{\F\otimes\hh}^2}.
}
Interestingly, if all the functions $g^*(\cdot,p)$ have same norm $\mathfrak{g} = \|g^*(\cdot,p)\|_{\F\otimes\hh}$ in $\F\otimes\hh$, we have 
\eqals{
  \|H\|_\HS = \frac{1}{|P|}\|G\|_\HS = \frac{1}{\sqrt{P}}\sqrt{\sum_{p\in\P} \mathfrak{g}^2} = \mathfrak{g}.
}

\subsection{The best of both worlds}\label{sec:app-sum-kernel}

Here we formalize the comment in \cref{rem:universal-and-local-kernels}, where we introduced the kernel $k_B = k_U + k_L$ that is sum of a bounded universal continuous kernel $k_U$ over $\X\times\P$ and a bounded restriction (or ``local'') kernel $k_L$, satisfying \cref{eq:restriction-kernel}. In particular we show that $k_B$ is universal but at the same time allows to train a structured prediction estimator $\fhat$ that is able to leverage the locality of the learning problem, when available. For simplicity, we assume the input space $\X$ to be compact and the set of parts indices $\P$ to be finite. 

Let $\ff_B, \ff_U$ and $\ff_L$ denote the RKHSs of respectively $k_B$, $k_U$ and $k_L$. According to \cite{aronszajn1950theory}, we know that $\ff_B \supseteq \ff_U \cup \ff_L$ and moreover that for any $h\in\ff_B$, the norm is such that 
\eqal{\label{eq:norms-sum-comparison}
  \|h\|_{\ff_B}^2 = \min_{h = h_U + h_L} ~ \|h_U\|_{\ff_U}^2 + \|h_L\|_{\ff_L}^2,
}
with $h_U \in \ff_U, h_L \in \ff_L$.
We immediately see that $k_B$ is universal. Indeed, since $k_U$ is universal, $\ff_U$ is dense in the space of continuous functions on $\X$ and consequently also $\ff_B \supseteq \ff_U$ is.

\noindent The following result is analogous to \cref{cor:q-as-sigma-p-q} and shows that the kernel $k_B$ is not only universal but also equivalent to $k_L$ in capturing the locality of the learning problem. 

\begin{lemma}\label{lem:q-as-sigma-p-q-universal}
    Denote by $k = k_B = k_U + k_L$ the sum kernel, where $k_U$ and $k_L$ are the universal and restriction kernels on $\X\times\P$, with $k_L$ as in \cref{eq:restriction-kernel} in terms of respectively $\bar k:\pX\times\pX\to\R$ and $k_0:\X\times\X\to\R$. Let $\bar{\msf{r}} = \sup_{\chi\in[\X]} \bar k(\chi,\chi)$ and $\msf{r}_0 = \sup_{x\in\X} k_0(x,x)$.

     Let $\pi(p|x) = \frac{1}{|P|}$ for any $x\in\X$ and $p\in\P$. Denote with $\bar\scov_{pq}$ the constant defined in \cref{eq:within-locality-measures} associated to the restriction kernel $k_L$. Then, the constant $\msf q$ in \cref{eq:def-constants-q} associated to $k_B$ can be factorized as
  \eqals{\label{eq:within-locality-measures-uni-local}
    \msf q = \frac{1}{|P|^2} \sum_{p,q \in P}  \scov_{p,q}, \quad \textrm{with} \quad \scov_{p,q} \leq \bar\scov_{p,q} ~+~ (4\bar{\msf{r}} + \msf{r}_0)\msf{r}_0~\delta_{p,q}.
  }
\end{lemma}
\begin{proof}
The proof of the result above follows by noting that, since $\pi$ is uniform, by \cref{lem:q-as-p-q-appendix}, for any $p,q\in\P$, $\scov_{p,q}$ is characterized by
\eqals{
  \scov_{p,q} & = \Expect_{x,x'} \left[ (\bar k(x_p,x_q) + k_0(x,x)\delta_{p,q})^2 - (\bar k(x_p,x'_q) + k_0(x,x')\delta_{p,q})^2\right] \\
  & = \bar\scov_{p,q} + \Expect_{x,x'} \left[ k_0(x,x)^2 - k_0(x,x')^2\ \right]\delta_{p,q} +\\
  & ~~~~ - 2 \Expect_{x,x'} \left[ \bar k(x_p,x_q)k_0(x,x) - \bar k(x_p,x'_q)k_0(x,x')\ \right]\delta_{p,q} \\
  & \leq \bar\scov_{p,q} + \delta_{p,q} ~ \sup_{x\in\X} k_0(x,x)^2 + 4 \delta_{p,q}~ \left[\sup_{\chi\in[\X]} \bar k(\chi,\chi) \sup_{x\in\X} k_0(x,x)\right] \\
  & \leq \bar\scov_{p,q} + (4\bar{\msf{r}} + \msf{r}_0)~\msf{r}_0~\delta_{p,q}
}
as desired. Note that the first inequality follows from the fact that $\bar k$ and $k_0$ are positive definite symmetric kernels. 
\end{proof}

 \noindent Interestingly, \cref{lem:q-as-sigma-p-q-universal} shows that the proposed sum kernel inherits the ability of the restriction kernel to capture the within- and between-locality of the learning problem. Combining this with the learning rates of \cref{thm:rates}, we obtain a result analogous to that of \cref{thm:rates-improved-with-parts}.

\begin{restatable}[Learning Rates \& Locality]{theorem}{TRatesImprovedWithPartsUniversalKernel}\label{thm:rates-improved-with-parts-universal-kernel}
    With the same notation of \cref{lem:q-as-sigma-p-q-universal} let $k_U$ be a bounded continouous universal kernel on $\X$, $k_L$ be the restriction kernel based on the reproducing kernel $\bar{k}$ on $\pX$ and let $\bar{\F}$ be the RKHS associated to $\bar{k}$. Let $\fhat$ be the structured prediction estimator of \cref{eq:estimator} learned with kernel $k = k_B = k_U + k_L$. Then
    \bi
  	\item $\widehat{f}$ is universally consistent,
  	\item Under \cref{asm:within-locality,asm:between-locality} and $\pi(p|x) = \frac{1}{|P|}$ for $x \in \X, p \in P$, let $\gbs$ be defined as in \cref{prop:gstar-as-gbs} and $\gbs \in \hh \otimes \bar{\ff}$. Denote by $\bar{\msf{g}}$ the norm $\bar{\msf{g}}  = \|\gbs\|_{\hh \otimes \bar{\ff}}$. When $\la = (\mathsf{r^2}/m + \mathsf{q}/n)^{1/2}$, then
      	\eqal{
      		\Expect~\left[\E(\fhat~)- \E(\fstar)\right]  ~\leq~ 12 ~\closs~ \mathsf{\bar{g}} ~ \msf r^{1/2} \msf~\left(\frac{1}{m} + \frac{c_1}{|P|n} + \frac{\sum_{p\neq q} e^{-\gamma d(p,q)}}{|P|^2 n}\right)^{1/4},
      	}
      	where $\msf{r} = \msf{r}_0 + \bar{\msf{r}}$, with $\msf{r}_0, \bar{\msf{r}}$ defined as in \cref{lem:q-as-sigma-p-q-universal} and $c_1 = 1+ (4\bar{\msf{r}} + \msf{r}_0)~\msf{r}_0/\msf{r}^2$.
    \ei
\end{restatable}


\begin{proof}
	Let $\ff_B, \ff_U$ and $\ff_L$ denote the RKHSs of respectively $k_B$, $k_U$ and $k_L$.

	First, as discussed at the beginning of this section, the kernel $k = k_B := k_U + k_L$ is universal, since $\ff_U \subseteq \ff_B$ (see \cite{aronszajn1950theory}) and $\ff_U$ is dense in the continuous functions on $\X \times P$. Then we can directly apply \cref{thm:universal} obtaining the unversal consistency for $\widehat{f}$. 
	
	Second, under \cref{asm:between-locality}, by \cref{prop:gstar-as-gbs}, we have that there exists $\bar{g}^*:\pX \to \hh$ such that $g^*$, defined as in \cref{eq:char-fstar}, is characterized by $g^*(x,p) = \bar{g}^*(x_p)$. S ince we assume that $\bar{g}^* \in \hh \otimes \bar{\ff}$ and we are using a restriction kernel under between-locality, we can apply \cref{lm:gstar-in-F-given-gbar-in-F} (where we used $\bar{\gg}$ to denote $\bar{\ff}$ and $\ff$ to denote $\ff_L$ and $\bar{g}^* \in \hh \otimes \bar{\ff}$ is expressed more formally by \cref{asm:k-and-g-in-Gbar}), then $g^* \in \hh \otimes \ff_L $ and $\|g^*\|_{\hh \otimes \ff_L} = \|\bar{g}^*\|_{\hh \otimes \bar{\ff} }$. Now, according to \cref{eq:norms-sum-comparison} (see \cite{aronszajn1950theory}), for any function $h \in \ff_L$ we have
	$$\|h\|_{\ff_B} := \min\{ \|h_U\|_{\ff_U} + \|h_L\|_{\ff_L} ~|~ h = h_U + h_L, h_U \in \ff_U, h_L \in \ff_L\} \leq \|h\|_{\ff_L} ,$$
	since $h$ can be always decomposed as $h = h_L + h_U$ with $h_L = h$ and $h_U = 0$, then $\|g^*\|_{\hh \otimes \ff_B } \leq \|g\|_{\hh \otimes\ff_L}$. So
	$$\|g^*\|_{\hh \otimes \ff_B} \leq \|\bar{g}^*\|_{\hh \otimes \bar{\ff}}.$$
	Now we are ready to apply \cref{thm:rates}, with $\la = \sqrt{\msf{r}^2/m + \msf{q}/n}$ obtaining
	\eqal{
		\Expect~\left[\E(\fhat~)- \E(\fstar)\right]  ~\leq~ 12 ~\closs~ \mathsf{\bar{g}} ~ \left(\frac{r^2}{m} + \frac{q}{n} \right)^{1/4}.
	}
	Finally note that since $\pi(p|x) = \frac{1}{|P|}$ for $p \in P, x \in \X$,we can apply \cref{lem:q-as-sigma-p-q-universal} 
	$$\frac{q}{n} = \frac{\msf{r}^2 c_1}{|P|n} + \frac{\msf{r}^2\sum_{p\neq q} e^{-\gamma d(p,q)}}{|P|^2 n},$$
	obtaining the desired result.
\end{proof}
The discussion above implies that under the locality assumptions, the rates in \cref{thm:rates-improved-with-parts-universal-kernel} are essentially equivalent to the ones of the estimator trained with only the restriction kernel in \cref{thm:rates-improved-with-parts}.

\section{Additional details on evaluating \texorpdfstring{$\fhat$}{alt-text}}\label{sec:app-algorithm}

\begin{algorithm}[t]
  \caption{ -- {\sc Learn} $\fhat$}\label{alg:self-learning-general}
  \begin{algorithmic}
    \State ~
    \State {\bfseries Input:} training set $(x_i,y_i)_{i=1}^n$, distributions $\pi(\cdot|x)$ and $\mu(\cdot|y,x,p)$, reproducing kernel $k$ on $\X\times\P$, hyperparameter $\la>0$, auxiliary dataset size $m\in\N$.
    \State~
    \State {\bfseries {\sc Generate}} the auxiliary dataset $(w_j,x_{i_j},p_j)_{j=1}^m$:
    \State \quad Sample $i_j$ uniformly from $\{1,\dots,n\}$
    \State \quad Sample $p_j\sim\pp(\cdot|x_{i_j})$
    \State \quad Sample $w_j\sim\mu(\cdot|y_{i_j},x_{i_j},p_j)$
    
    \State~
    
    \State {\bfseries {\sc Learn}} the coefficients for the score function $\alpha$:
    \State \quad $\K\in\R^{m \times m}$ with entries $\K_{jj'} = k\big((x_{i_j},p_j),(x_{i_{j'}},p_{j'})\big)$
    \State \quad $\mathbf{A} = (\K + m\la I)^{-1}$
    
    \State~
    \State {\bfseries Return:} $\alpha:\X\times\P\to\R^m$ such that $\alpha(x,p) = \mathbf{A} ~ v(x,p)$ with $v(x,p)\in\R^m$ is the vector with entries $v(x,p)_j = k\big((x_{i_j},p_j),(x,p)\big)$. 
    
  \end{algorithmic}
\end{algorithm}

\begin{algorithm}[t]
  \caption{ -- {\sc Evaluating} $\fhat$}\label{alg:sgm-hatf-evaluation}
  \begin{algorithmic}
    \State ~
    \State {\bfseries Input:} input $x\in\X$, distribution $\pp(\cdot|x)$, auxiliary dataset $(w_j,x_{i_j},p_j)_{j=1}^m$, score functions $\alpha:\X\times\P\to\R$, number of iterations $T$, step sizes $\{\gamma_t\}_{t\in\N}$.
    \State~
    \State {\bfseries{\sc Initialize}: $z_0 = 0$}
    
    \State ~
    
    \State {\bfseries For} $t=1$ to $T$
    \State \quad Sample $p\sim\pp(\cdot|x)$
    \State \quad $A(x,p) = \sum_{j=1}^m |\alpha_j(x,p)|$
    \State \quad Sample $j$ from $\{1,\dots,m\}$ with $\mathbb{P}(j=k) = |\alpha_k(x,p)|/A(x,p)$
    \State \quad $h_{j,p} = \sign(\alpha_j(x,p)) ~ A(x,p) ~ \ell(z,w_j|x,p)$
    \State \quad Choose $u \in \partial h_{j,p}(\cdot|x)(z_{t-1}) $
    \State \quad $z_t = \proj_\Z(z_{t-1} - \gamma_t u)$

    \State~
    
    \State {\bfseries Return}: $z_T$ 
  \end{algorithmic}
\end{algorithm}

According to \eqref{eq:estimator}, evaluating $\fhat$ on a test point $x\in\X$ consists in solving an optimization problem over the output space $\Z$. This is a standard procedure in structured prediction settings \cite{nowozin2011}, where a corresponding optimization method is derived on a case-by-case basis depending on the loss and the space $\Z$ (\cite{nowozin2011}). However, the specific form of the objective functional characterizing $\fhat$ in our setting allows to devise a general stochastic meta-algorithm to solve such problem. We observe that \eqref{eq:estimator} can be rewritten as
\eqal{
  \fhat(x) = \argmin_{z\in\Z} ~ \mathbb{E}_{(j,p)}~ h_{j,p}(z|x)
}
where for any $p\in\P$ and $j\in\{1,\dots,m\}$ we have introduced the functions $h_{j,p}:\Z\to\R$, such that 
\eqal{
  h_{j,p}(\cdot|x) = \big(~\sign(~\alpha_j (x,p)~)~ \mathsf{A}({x,p})\big)~ \ell(\cdot,w_j|x,p) 
}
%
with $\mathsf{A}(x,p) = \sum_{j=1}^m |\alpha_j(x,p)|$. In the expectation above, the variable $p$ is sampled according to $\pp(\cdot|x)$ and $j$ is sampled from the set $\{1,\dots,m\}$ with probability $\frac{|\alpha_j(x,p)|}{\mathsf{A}(x,p)}$. When the $h_{j,p}$ are (sub)differentiable, problems of the form of \eqref{eq:evaluating-as-sgd} can be addressed by stochastic gradient methods (SGM). In \cref{alg:sgm-hatf-evaluation} in the supplementary material we provide an example of such strategy.

\section{Additional examples of Loss Functions by Parts}\label{sec:example-loss}

Several structured prediction settings are recovered within the setting considered in this work and the associated loss functions have the form of \cref{eq:loss-base-parts}. Below recall some of the most relevant examples where the locality assumptions can be reasonaly expected to hold.

\paragraph{Hamming} A standard loss function used in structured prediction is the Hamming loss \cite{collins2004parameter,taskar2004max,cortes2014ensemble}, which for any factorization by parts can be written as in \eqref{eq:loss-base-parts} with $\L_p(z_p,y_p|x_p) = \delta(z_p \neq y_p)$, the function equal to $0$ if $z_p = y_p$ and $1$ otherwise. 
\begin{itemize}
  \item {\bf Computer Vision}. The Hamming loss is often used in computer vision \cite{nowozin2011,vedaldi2009structured}. For instance, in image segmentation~\cite{szummer2008learning} the goal is to label each pixel $p$ of an input image $x$, as background ($y_p = 0$) or foreground ($y_p = 1)$. Errors are measured as total number of mistakes $z_p\neq y_p$ over the total number of pixels.
  
  \item {\bf Hierarchical Classification}. In classification settings with a hierarchy \cite{tuia2011structured}, errors are weighted according to the semantic distance between two classes (e.g. classifying the image of a ``dog'' as a ``bus'' is worse than classifying it as a ``cat''). Assuming the hierarchy between classes to be represented as a tree, these loss functions can be written as the Hamming loss between the parts of a class $y = (y_{\rm root},\dots,y_{\rm leaf})$ seens as the collection of all the nodes in its hierarchy (e.g. ``cat'', ``feline'', ``mammal'', ``animate object'', ``entity''). 
  
  \item {\bf Planning}. In learning-to-plan applications \cite{ratliff2006maximum}, the goal is to predict a trajectory $z$ closest to a ground truth trajectory (typically provided by an expert). A trajectory is represented as a sequence of contiguous states $y=(y_{\rm start},\dots,y_{\rm end})$ and errors with respect to a predicted trajectory $z$ are measured in terms of the number of states that do not coincide, namely the hamming loss between the two sequences.
  
\end{itemize}
This loss has been extensively used in computer vision for applications such as pixel-wise classification \cite{szummer2008learning} or image segmentation \cite{alahari2008reduce}.

\paragraph{Precision/Recall, F$1$ Score}. The precision/recall and F$1$ score are loss functions often adopted in natural language processing \cite{tsochantaridis2005}. They are used to measure the similarity between two binary sequences. Given two binary sequences $z,y\in\{0,1\}^k$ of length $k$, we have $\loss(z,y) = \loss(z^\top y, \|z\|^2,\|y\|^2)$. In particular, the precision correponds to $\loss(z,y) = z^\top y/\|z\|^2$, the recall to $\loss(z,y) = z^\top y/\|y\|^2$ and the F$1$ score to $\loss(z,y) = z^\top y/(\|z\|^2 + \|y\|^2)$. These functions are in the form of \eqref{eq:loss-base-parts} if taking $|P| = k$ and $\mathfrak{i}_\Y(y,p) = (y_p,\|y\|)$, $\mathfrak{i}_\Y(z,p) = (z_p,\|z\|)$. Note that the number of elements in $y$ and $z$ can vary depending on the cardinality $|x|$ of each input $x$, (see e.g. \cite{tsochantaridis2005}). In this sense the $\loss(z,y|x)$ is necessarily parametrized by $x$ and in particular the set $P$ is a set $P(x) = \{1,\dots,|x|\}$.

\paragraph{Multitask Learning} Multitask learning settings  have a natural decomposition into parts:  the output and label spaces $\Z$ and $\Y$ are subset of $\R^{T}$, and $\loss(z,y) = \frac{1}{T}\sum_{t=1}^T \L(z_t,y_t)$, with $\L$ any loss function commonly used in standard supervised learning problems (e.g. least-squares for regression, hinge or logistic for classification). In settings where $\Z$ is not a linear space but a {\em constraint set}, our model recovers the non-linear multitask learning framework considered in \cite{ciliberto2017consistent}.

\paragraph{Learning sequences}.~\label{ex:learning-sequences}
  Let $\X = A^k$, $\Y = \Z =  B^k$ for two sets $A, B$ and $k \in \N$ a fixed length. We consider a set of structures $P \subseteq \N^2$ such that any pair $p = (s,l) \in P$ indicates the starting element and the length of a subsequence. In particular, we choose the set of parts ${\cal X} = \cup_{t=1}^k A^t$ and ${\cal Y} = {\cal Z} =  \cup_{t=1}^k B^t$ with
  \eqals{
    x_p = (x^{(s)},\dots, x^{(s+l-1)}) \in {\cal X} \qquad\qquad \forall~ x \in \X, ~~ \forall~ (s,l) \in P
  }
  where we have denoted $x^{(s)}$ the $s$-th entry of the sequence $x\in X$. Analogously $y_p = (y^{(s)},\dots, y^{(s+l-1)})$ for $y\in Y$. Finally, we choose the loss $L_0$ to be the (normalized) edit distance between two strings of same length
  \eqals{
    L_0(z,y;x,(s,l)) = \frac{1}{l} \sum_{i=1}^l \boldsymbol 1(z^{(i)} \neq y^{(i)})
  }
  where $\boldsymbol 1(z^{(i)} \neq y^{(i)}) = 0$ if $z^{(i)} = y^{(i)}$ and $1$ otherwise (clearly a generic loss function $h(z^{(i)} \neq y^{(i)})$ and weight $w_i$ can be used instead of $\boldsymbol 1$ and $1/l$). Finally, we can choose the uniform distribution $\pi(p|x) = 1/|P|$ (but clearly also less symmetric weighting strategy can be adopted). 

\paragraph{Pixelwise classification on images}~\label{ex:pixelwise-classification}
  Consider the problem of assigning each pixel of an image to one of $T$ separate classes. In this setting $X = \R^{d \times d}$ is the set of images (with fixed width and height equal to $d\in\N$) and $Y = Z = \R^{T \times d \times d}$ is the set of all possible ways to label an image. We choose the set of parts  $\mathcal X = \cup_{w,h = 1}^{d} \R^{w \times h}$ to be the set of all possible patches of $d\times d$ image and the set of structures to be a $P \subset \N^4$ such that for any image $x\in X$ and $p = (u,l,w,h)\in P$ the selectors $x_p\in\R^{w \times h}$ and $y_p,z_p \in \R^{T \times w \times h}$ correspond to the patch of the image $x$ or the labeling $y$ and $z$ with width $w$, height $h$ and upper-left corner at the pixel $(u,l)$.
  We choose the loss $L_0$ to be a function comparing the class ``statistics'' in a given patch: e.g.
  \eqals{
    L_0(z_p,y_p;x_p,p) = \|\sigma(z_p) - \sigma(y_p)\|^2 \qquad \sigma(\zeta) = \frac{\sum_{i=1}^{\textrm{width}(\zeta)}\sum_{j=1}^{\textrm{height}(\zeta)}\zeta_{:,i,j}}{\textrm{width}(\zeta)\textrm{height}(\zeta)}.
  }
  Since it is more likely to have larger values for $L_0$ at higher scales (the object patch overlaps other classes), we choose a weighting $\pi(p|x)$ that is decreasing with respect to the size of the patch $p = (u,l,w,h)$. For instance we can choose $\pi(p|x) = \frac{\exp(-\gamma wh)}{\sum_{p'=(u',l',w',h')\in P}\exp(-\gamma w'h')}$, for $\gamma > 0$.

\subsection{Example: Locality on sequences}
\label{sec:example-sequence-to-sequence}

We comment here on the example in \cref{ex:locality-on-sequences} proving the inequality \cref{eq:locality-on-sequences-example}. We assume \cref{asm:within-locality} to hold for $\P = \{1,\dots,|\P|\}$ with $d(p,q) = |p - q|$ and $\gamma>0$. We have
\begin{align}
    \mq & = \frac{\mathsf{r}^2}{|P|}\sum_{p,q=1}^{|P|}~ e^{-\gamma |p-q|} \\
    & \leq \frac{\mathsf{r}^2}{|P|}\sum_{p=q=1}^{|\P|}~ e^{-\gamma |p-q|} + 2\frac{\mathsf{r}^2}{|P|}\sum_{p=1}^{|P|-1}\sum_{q>p}^{|\P|}~ e^{-\gamma |p-q|}.
\end{align}
Now, we introduce the change of variable $t = q - p$ to obtain 
\begin{align}
    \mathsf{r}^2 + 2\frac{\mathsf{r}^2}{|P|}\sum_{p=1}^{|P|-1}\sum_{\substack{t = 1 \\ q = p + t}}^{|\P|-p}~ e^{-\gamma |p - q|} & = \mathsf{r}^2 + 2\frac{\mathsf{r}^2}{|P|}\sum_{p=1}^{|P|-1}\sum_{t = 1}^{|\P|-p}~ e^{-\gamma t} \\
    & \leq \mathsf{r}^2 + 2\frac{\mathsf{r}^2}{|P|}\sum_{p=1}^{|P|-1}\sum_{t = 1}^{|\P|}~ e^{-\gamma t}\\
    & \leq \mathsf{r}^2 + 2\mathsf{r}^2\sum_{t = 1}^{|\P|}~ e^{-\gamma t}\\
    & \leq 2\mathsf{r}^2 (\sum_{t = 0}^{|\P|}~ e^{-\gamma t}).
\end{align}
We can upper bound $\sum_{t = 1}^{|\P|}~ e^{-\gamma t} = \sum_{t = 1}^{|\P|}~ (e^{-\gamma})^t$ with the geometric series $\sum_{t = 1}^{+\infty}~ e^{-\gamma t}$. Since $\gamma>0$ we conclude that such series is upper bounded by $(1-e^{-\gamma})^{-1}$, concluding 
\eqals{
\mq ~~\leq~~ 2\mathsf{r}^2 (1-e^{-\gamma})^{-1},
}
as desired.

\end{document}